\documentclass[aos]{imsart}

\RequirePackage{amsthm,amsmath,amsfonts,amssymb}
\RequirePackage[numbers]{natbib}
\RequirePackage[colorlinks,citecolor=blue,urlcolor=blue]{hyperref}
\RequirePackage{graphicx}

\startlocaldefs

\newcommand{\ignore}[1]{}
\usepackage{graphicx,epsfig,rotating,algorithm,algpseudocode}
\usepackage{amssymb}
\usepackage{threeparttable}
\usepackage{hyperref}
\usepackage[dvipsnames,table,dvipsnames*, svgnames*, hyperref]{xcolor}
\usepackage{multirow}
\usepackage{bm}

\newtheorem{thm}{Theorem}
\newtheorem{assu}{Assumption}
\newtheorem{defn}{Definition}
\newtheorem{rem}{Remark}
\newtheorem{prop}{Proposition}
\newtheorem{lem}[thm]{Lemma}

\DeclareMathOperator*{\argmax}{argmax}
\newcommand{\R}{\mathbb{R}}
\newcommand{\xb}{\mathbf{x}}
\newcommand{\yb}{\mathbf{y}}
\newcommand{\vb}{\mathbf{v}}
\newcommand{\ub}{\mathbf{u}}
\newcommand{\zb}{\mathbf{z}}
\newcommand{\bK}{\mathbf{K}}
\newcommand{\bA}{\mathbf{A}}
\newcommand{\bB}{\mathbf{B}}
\newcommand{\bE}{\mathbf{E}}
\newcommand{\bL}{\mathbf{L}}
\newcommand{\bD}{\mathbf{D}}
\newcommand{\bO}{\mathbf{O}}
\newcommand{\bR}{\mathbf{R}}
\newcommand{\bV}{\mathbf{V}}
\newcommand{\bU}{\bold{U}}
\newcommand{\bM}{\bold{M}}
\newcommand{\bLam}{\boldsymbol{\Lambda}}

\newcommand{\bSig}{\boldsymbol{\Sigma}}

\newcommand{\bGam}{\boldsymbol{\Gamma}}
\newcommand{\E}{\mathbb{E}}
\newcommand{\OO}{\mathrm{O}}
\newcommand{\oo}{\mathrm{o}}
\newcommand{\ri}{\mathrm{i}}
\newcommand{\dd}{\mathrm{d}}
\newcommand{\sfp}{\mathsf{p}}
\newcommand{\sfP}{\mathsf{P}}
\newcommand{\sfL}{\mathsf{L}}
\newcommand{\sfr}{\mathsf{r}}
\newcommand{\sfh}{\mathsf{h}}
\newcommand{\HS}{\mathsf{HS}}

\newcommand{\beq}{\begin{equation}}
\newcommand{\eeq}{\end{equation}}
\newcommand{\beas}{\begin{eqnarray*}}
	\newcommand{\eeas}{\end{eqnarray*}}
\newcommand{\bea}{\begin{eqnarray}}
\newcommand{\eea}{\end{eqnarray}}
\newcommand{\bei}{\begin{itemize}}
	\newcommand{\eei}{\end{itemize}}
\newcommand{\ben}{\begin{enumerate}}
	\newcommand{\een}{\end{enumerate}}

\newcommand{\red}{\color{red}}

\endlocaldefs

\begin{document}

\begin{frontmatter}
\title{Learning Low-Dimensional Nonlinear Structures from High-Dimensional Noisy Data: An Integral Operator Approach}
\runtitle{Kernel Spectral Embedding of High-Dimensional Data}

\begin{aug}
\author[A]{\fnms{Xiucai} \snm{Ding}\ead[label=e1]{xcading@ucdavis.edu}}
\and
\author[B]{\fnms{Rong} \snm{Ma}\ead[label=e2]{rongm@stanford.edu}}
\address[A]{Department of Statistics,
University of California, Davis
\printead{e1}}

\address[B]{Department of Statistics,
Stanford University
\printead{e2}}
\end{aug}

\begin{abstract}
We propose a kernel-spectral embedding algorithm for learning low-dimensional nonlinear structures from noisy and high-dimensional observations, where the data sets are assumed to be sampled from {a nonlinear manifold model} and corrupted by high-dimensional noise. The algorithm employs an adaptive bandwidth selection procedure which does not rely on prior knowledge of the underlying manifold. The obtained low-dimensional embeddings can be further utilized for downstream purposes such as data visualization, clustering and prediction. Our method is theoretically justified and practically interpretable. Specifically, {for a general class of kernel functions}, we establish the convergence of the final embeddings to their noiseless counterparts when {the dimension grows polynomially with the size}, and characterize the effect of the signal-to-noise ratio on the rate of convergence and phase transition. We also prove the convergence of the embeddings to the eigenfunctions of an integral operator defined by the kernel map of some reproducing kernel Hilbert space capturing  the underlying nonlinear structures. {Our results hold even when the dimension of the manifold grows with the sample size.} Numerical simulations and analysis of real data sets show the superior empirical performance of the proposed method, compared to many existing methods, on learning various nonlinear manifolds in diverse applications.     	
\end{abstract}



\end{frontmatter}

\section{Introduction}

With rapid technological advancements in data collection and processing,  massive large-scale and high-dimensional data sets are widely available nowadays in diverse research fields such as astronomy, business analytics, human genetics and microbiology. A common feature of these data sets is that their statistical and geometric properties can be well understood via a meaningful low-rank representation of reduced dimensionality. Learning low-dimensional structures from these high-dimensional noisy data is one of the central topics in statistics and data science. 
Moreover, nonlinear structures have been found predominant and intrinsic in many real-world data sets, which may not be easily captured or preserved in commonly used linear or quasi-linear methods such as principal component analysis (PCA), singular value decomposition (SVD) \citep{jolliffe2002principal} and multidimensional scaling (MDS) \citep{borg2005modern}. As a longstanding and well-recognized technique for analyzing data sets with possibly nonlinear structures, kernel methods have been shown effective in various applications ranging from clustering, data visualization to classification and prediction \citep{shawe2004kernel,hofmann2008kernel,kung2014kernel}. On the other hand, spectral methods \citep{chen2021spectral}, as a fundamental tool for dimension reduction, are oftentimes applied in combination with kernel methods to better capture the underlying low-dimensional nonlinear structure in the data. These approaches are commonly referred to as nonlinear dimension reduction techniques; see Section \ref{related.works} below for a brief overview.    

{Despite} the effectiveness and success of the kernel-spectral methods in many applications, these methods are usually applied {heuristically}, especially in terms of tuning parameter selection and interpretation of the results. For example, as a key step in these methods,
the construction of the kernel (or affinity) matrices requires specifying a proper bandwidth parameter, which is usually determined by {certain} heuristics or conventional empiricism. This is mainly due to a lack of theoretical understanding of the methods, including their intrinsic objective, their range of applicability, and their susceptibility to noise or dimensionality of the data. Likewise, the absence of a theoretical foundation may significantly constrain the users' recognition or exploitation of the full potential of the method. Given the indispensable status and the practical power of these methods in many fields of applications, such as single-cell transcriptomics \citep{wang2017visualization,moon2019visualizing,kobak2019art} and medical informatics \citep{adeli2017kernel,shiokawa2018application}, there is a pressing need of theoretically justified kernel-spectral method that accounts for both the high-dimensionality and the noisiness of the data.

In this study, we focus on an integral operator approach to learning low-dimensional nonlinear structures from high-dimensional noisy data. We propose a kernel-spectral embedding algorithm with a data-adaptive bandwidth selection procedure ({cf.} Algorithm \ref{al0}), justified by rigorous theoretical analysis, and provide solid interpretations of the low-dimensional embeddings by establishing their {asymptotic behavior}.

Specifically, we consider a manifold "signal-plus-noise" data-generative model
\begin{equation}\label{eq_basicmodel}
\yb_i=\xb_i+\zb_i \in \mathbb{R}^p,\quad 1 \leq i \leq n, 
\end{equation}
where $\{\yb_i\}_{1\le i\le n}$ are observed samples, $\{\xb_i\}_{1\le i\le n}$ are the underlying noiseless samples drawn independently {from a nonlinear manifold model}, embedded into $\mathbb{R}^p$ {(see Assumption \ref{assum_generalmodelassumption} below for more detail)}, and $\{\zb_i\}_{1\le i\le n}$ are independent sub-Gaussian noise (see Section \ref{notation.sec} for the definition) with covariance $\sigma^2{\bf I}$. As will be seen in Section \ref{man.sec}, (\ref{eq_basicmodel}) is closely related to the spiked covariance model \citep{johnstone2001distribution} extensively studied in the past two decades. Throughout, we focus on the high-dimensional setting where the number of variables $p$ { grows polynomially with} the sample size $n$, that is, { for some constant $0<\eta <\infty$, we have
\begin{equation}\label{eq_aspect}
p\asymp n^\eta.
\end{equation}}
As a general and important problem in nonlinear dimension reduction and manifold learning, our  goal  is to find a well-justified low-dimensional embedding of $\{\yb_i\}_{1\le i\le n}$ that captures the nonlinear structure of the underlying manifold, as encapsulated in the noiseless random samples $\{\xb_i\}_{1\le i\le n}$. To facilitate our discussion, we construct the kernel matrix  $\bK_n \in \mathbb{R}^{n \times n}$ using the observations $\{\yb_i\}$ as follows {
\beq \label{Kmat}
\bK_n=\big( K(i,j) \big)_{1\le i,j\le n},\qquad K(i,j)=f\left(\frac{\| \yb_i-\yb_j \|_2}{h^{1/2}_n} \right),
\eeq
for some function $f: \R_{\geq 0}\to\R_{\geq 0}$,
and $h_n>0$ is some bandwidth. }

\subsection{Some Related Works} \label{related.works}

Unlike our general setup (\ref{eq_basicmodel}) under high-dimensionality (\ref{eq_aspect}), most of the existing works concern computational algorithms for learning nonlinear manifolds,  developed under either the noiseless (i.e., $\yb_i=\xb_i$) or low-dimensional settings  where $p$ is fixed. These methods can be roughly separated into two categories, depending on how the data is conceptually represented.  On the one hand, there are spectral graph based methods, including Laplacian eigenmap \citep{Belkin_Niyogi:2003}, diffusion map (DM) \citep{Coifman_Lafon:2006}, vector diffusion map (VDM) \citep{singer2012vector},  ISOMAP \citep{tenenbaum2000global}, maximal variance unfolding (MVU) \citep{weinberger2006introduction}, locally linear embedding (LLE) \citep{Roweis_Saul:2000}, Hessian LLE \citep{donoho2003hessian}, t-SNE \citep{van2008visualizing},  UMAP \citep{mcinnes2018umap} and local tangent space alignment (LTSA) \citep{zhang2004principal}, among many others. For more details of these methods, see \citep{lee2007nonlinear, van2009dimensionality}. Generally speaking, these methods start by constructing a possibly sparse graph using kernels in which the nodes represent the input objects and the edges represent the neighborhood relations. Consequently, the resulting graph can be viewed as a discretized approximation of the manifold sampled by the inputs. From these graphs, one may apply results from spectral graph theory and construct matrices whose spectral decomposition reveals the low-dimensional structure of the manifold. Incidentally, these methods have been used to facilitate various downstream statistical applications, such as spectral clustering  \citep{von2007tutorial}, regression analysis \citep{cheng2013local}, among others.
The main differences between these methods are the kernels used to construct the spectral graphs and the associated nonlinear eigenfunctions (or feature maps \citep{hofmann2008kernel}) employed for embedding (or data representation).  On the other hand, there are reproducing kernel Hilbert space (RKHS) based methods, that use
eigenfunctions of some kernel map of the RKHS to infer the nonlinear structures of the manifold based on noiseless observations \citep{shawe2004kernel,hofmann2008kernel,kung2014kernel}. Specifically, these methods start from a certain similarity matrix based on some positive kernel function, and infer the geometric properties from the integral operator associated with the kernel. Similarly, various downstream tasks such as clustering, regression and classification \citep{hofmann2008kernel,scholkopf2002learning,AOSko,10.2307/25464638}  have been treated in combination with the RKHS-based methods. In addition to these, kernel principal component analysis (kPCA) \citep{scholkopf1998nonlinear} has also been widely used.

From a mathematical viewpoint, in the low-dimensional and noiseless setting, theoretical properties of some of these methods have been explored. For the graph-based methods, existing theoretical results reveal that, under some regularity conditions,  the discretized approximation obtained from various kernel-spectral methods will converge to a continuum quantity, which is usually represented as a certain operator capturing the underlying structures of the manifold. For example, it has been shown that the discrete Graph Laplacian (GL) matrices ({cf.} (\ref{eq_GL}) and (\ref{eq_tran})) used by Laplacian eigenmap and DM would converge to the  Laplace-Beltrami operator of the manifold under various settings and assumptions \citep{belkin2007convergence,calder2020lipschitz,DUNSON2021282,MR4130541,gine2006empirical,SINGER2006128,singer2017spectral,wormell2021spectral}. For the LLE, as shown in \cite{wu2018think}, after being properly normalized, the similarity matrix would converge to the LLE kernel which is closely related to the Laplace-Beltrami operator. Convergence of t-SNE, MVU, VDM and LTSA have been studied in \cite{arora2018analysis,linderman2019clustering,cai2021theoretical,arias2013convergence,singer2012vector,NIPS2008_735143e9}. For the RKHS-based methods, it has been shown that the empirical eigenvalues and eigenfunctions would converge to those of an integral operator associated with some reproducing kernel \citep{JMLR:v7:braun06a, BJKO,JMLR:v11:rosasco10a,AOSko, Smale2007LearningTE,MR2558684, 10.2307/25464638}. Finally, kPCA has been studied under the noiseless setting  \citep{blanchard2007statistical}.       

Another important difference between the graph-based methods and the RKHS-based methods lies in the treatment of the kernel matrices ({cf.} (\ref{Kmat})). The graph-based methods usually involve some GL-type operations, where the kernel affinity matrix is normalized by the graph degrees \citep{DW2,van2009dimensionality}, whereas the RKHS-based methods commonly employ the kernel matrix directly without further normalization. Consequently, in terms of interpretations, existing theory under the low-dimensional noiseless setting indicates that the embeddings from the graph-based methods are associated with the eigenfunctions of Laplace-Beltrami operator of $\mathcal{M}$, whereas those from the RKHS-based methods are related to the eigenfunctions of some reproducing kernels of the manifold.

{Despite} these endeavors, much less is known when the data is high-dimensional and noisy as modeled by (\ref{eq_basicmodel}) and (\ref{eq_aspect}), nor does the behavior of the associated kernel random matrix $\bK_n$. In the null case (i.e., $\yb_i=\zb_i$), it has been shown in \cite{bordenave2013euclidean,cheng2013spectrum,9205615,do2013spectrum, el2010spectrum,Fan2018} that when $h_n=p$ and {\red $\eta=1$ in (\ref{eq_aspect})}, the random kernel matrix $\bK_n$  can be well approximated by a low-rank perturbed random Gram matrix. Consequently, studying $\bK_n$ under pure noise is closely related to PCA with some low-rank perturbations. Moreover, since in this case the degree matrix is close to a scalar matrix \citep{9205615}, the graph-based methods and the RKHS-based methods are asymptotically equivalent. As for the non-null cases, spectral convergence of $\bK_n$ has been studied in some special cases in \cite{el2010information} when $h_n=p$, and more recently under a general setting in \cite{DW2}. These results show how the eigenvalues of the kernel matrices of $\{\yb_i\}_{1\le i\le n}$ relate to those of $\{\xb_i\}_{1\le i\le n}.$ Moreover,  in the non-null cases, the degree matrix is usually non-trivial \cite{DW2}, so that the graph-based methods can be very different from the RKHS-based methods.  

In what follows, we focus on an RKHS-based integral operator approach to nonlinear dimension reduction and manifold learning  under the general setup (\ref{eq_basicmodel}) and (\ref{eq_aspect}) {which includes the settings $n \asymp p$, $n \gg p$ and $n \ll p$}. Even though in this setting $\bK_n$ has been studied to some extent in  \cite{el2010information} and \cite{DW2}, several pieces are still missing for rigorous and interpretable statistical applications. Firstly, it is unclear how to select the bandwidth adaptively for embedding, that is, free from prior knowledge of the manifold. Secondly, the limiting behavior  of the eigenvectors of $\bK_n$ have not been analyzed. Such results are particularly relevant for embedding purposes. Finally, even though the convergence of kernel matrices associated to $\{\yb_i\}_{1\le i\le n}$ ({cf.} (\ref{Kmat})) and $\{\xb_i\}_{1\le i\le n}$ ({cf.} (\ref{Kn*})) can be established, it is  unclear how they are related to the underlying manifold. Such results are important for appropriate interpretations of the embeddings. We will accomplish these goals in this paper.

\subsection{Overview of Main Results and Contributions}
Briefly speaking, our proposed kernel-spectral embedding algorithm starts by constructing  some kernel matrix $\bK_n \in \mathbb{R}^{n \times n}$ as in (\ref{Kmat}),
where $h_n$ is  carefully selected by a theory-guided data-adaptive procedure proposed in Section \ref{method.sec}. Then we obtain the kernel-spectral embedding of the original data $\{\yb_i\}_{1\le i\le n}$ by conducting an eigen-decomposition for the scaled kernel matrix $\frac{1}{n}\bK_n$, and define the embeddings as the leading eigenvectors of $\frac{1}{n}\bK_n$, weighted by their associated eigenvalues (Algorithm \ref{al0}). As we will show later, as {$n,p \to\infty$} with respect to (\ref{eq_aspect}), the thus constructed final embeddings  are essentially related to and determined by a population integral operator defined by the reproducing kernel of an RKHS associated to the underlying manifold $\mathcal{M}$. It is in this sense our proposed method is an integral operator approach.

Rigorous theoretical understanding is obtained for the proposed algorithm, including the theoretical justifications for the  bandwidth selection procedure, the limiting behavior of the low-dimensional embeddings, and their interpretations in relation to the underlying manifolds. Compared to the existing works on kernel-spectral embeddings, the current study has the following methodological advantages and theoretical significance:
\begin{itemize}
	\item We propose a kernel-spectral embedding algorithm for learning low-dimensional nonlinear structure from high-dimensional noisy data. Unlike most  existing methods, our proposal takes into account both the noisiness of the data, characterized by our  model (\ref{eq_basicmodel}), and its high-dimensionality as in (\ref{eq_aspect}). {Moreover, we do not impose assumptions on the dimension of the underlying manifold and allow it to diverge with the sample size $n.$} To the best of our knowledge, this is the first kernel-spectral embedding method with theoretically guaranteed performance for high-dimensional and noisy data under non-null settings.
	\item A key component in any kernel-based method is to select a proper bandwidth parameter. One major innovation of our proposed method lies in a careful analysis of the asymptotic behavior of  kernel random matrices constructed from high-dimensional noisy data, which in turn leads to a theory-informed bandwidth selection procedure guaranteeing the strong performance of the method in a data-adaptive manner. In particular, our bandwidth selection procedure does not rely on prior knowledge about the underlying manifold, and {can be efficiently applied to a large family of kernel functions. }
	\item We also provide an in-depth theoretical understanding of the proposed method. On the one hand, we study the convergence of the final low-dimensional embeddings to their oracle counterparts based on the noiseless samples $\{\xb_i\}_{1\le i\le n}$ {as $n,p\to\infty$}, and characterize explicitly the effect of the overall signal-to-noise ratio on the respective rates of convergence and phase transition. On the other hand, we establish the convergence of the kernel matrices to the population integral operator that captures the nonlinear structures. The second result is essential for our understanding of the final low-dimensional embeddings, interpreted as a finite approximation of the samples projected onto the leading eigenfunctions of the population integral operator. 
\end{itemize}
Our theoretical results explain the empirically observed data-adaptive feature of the proposed method (Section \ref{data.sec}), that captures the nonlinear structures regardless of the manifold {structures}. Moreover, our recognition of the limiting integral operator also suggests potential advantages of our method over alternative approaches targeting distinct geometric features of the manifold. For example, we may conclude from the theory that our method differs significantly from the graph-based methods, such as Laplacian eigenmap, DM, and LLE, as they all essentially aim at the Laplace-Beltrami operator rather than an integral operator (see Section \ref{comp.sec} and Section \ref{dis.sec} of our supplement \cite{suppl} for more discussions).

{
Our theoretical analysis, on the one hand, relies on a general model reduction scheme, detailed in Section \ref{man.sec}, that simplifies the analysis by connecting the noisy nonlinear manifold model (\ref{eq_basicmodel}) with a potentially divergent spiked covariance matrix model, in which we allow the dimension of the manifold to diverge with the sample size $n.$ } On the other hand,  we also leverage operator theory and random matrix theory, to prove our main results.

The proposed method admits strong empirical performance. In Section \ref{simu.sec}, we present simulation studies that show the superiority and flexibility of the proposed bandwidth selection method over some existing alternatives. 
In Section \ref{cell.order.sec} and Section \ref{suppl_additionalrealexamples} of \cite{suppl}, we analyze three real-world high-dimensional datasets with distinct  nonlinear structures, including a path manifold, a circle manifold, and a multiclass mixture manifold, to demonstrate the usefulness of the method. In particular, for each of the examples, our proposed method shows significant improvements over the existing state-of-the-art methods  in inferring the respective underlying nonlinear structures.

\subsection{Organization and Notation} \label{notation.sec}

The rest of the paper is organized as follows. In Section \ref{method.sec}, we introduce our proposed kernel-spectral embedding algorithm in detail, and we point out important differences between our proposal and some existing methods. In Section \ref{theory.sec}, we study the theoretical properties of the proposed method, including the convergence of the low-dimensional embeddings to their noiseless counterpart, and the spectral convergence of the kernel matrix to an integral operator. 
In Section \ref{data.sec}, we include our simulation studies on bandwidth selection and analyze a real-world dataset to show the numerical performance of the proposed method in various applications.
 Technical proofs, additional discussions and further numerical and real data results are provided in our supplementary file \citep{suppl}.  

We will finish this section by introducing some notations used in the paper. To streamline our statements, we use the notion of stochastic domination, which  is commonly adopted in random matrix theory to syntactically simplify precise statements of the form ``$\mathsf{X}^{(n)}$ is bounded with high probability by $\mathsf{Y}^{(n)}$ up to small powers of $n$."

\begin{defn} [Stochastic domination]\label{defn_stochasdomi} Let $	\mathsf{X}=\big\{\mathsf{X}^{(n)}(u):  n \in \mathbb{N}, \ u \in \mathsf{U}^{(n)}\big\}$ and $\mathsf{Y}=\big\{\mathsf{Y}^{(n)}(u):  n \in \mathbb{N}, \ u \in \mathsf{U}^{(n)}\big\}$
	be two families of nonnegative random variables, where $\mathsf{U}^{(n)}$ is a possibly $n$-dependent parameter set. We say that $\mathsf{X}$ is {\em stochastically dominated} by $\mathsf{Y}$, uniformly in the parameter $u$, if for all small $\upsilon>0$ and large $ D>0$, there exists $n_0(\upsilon, D)\in \mathbb{N}$ so that 
	\begin{equation*} \label{sd}
	\sup_{u \in \mathsf{U}^{(n)}} \mathbb{P} \Big( \mathsf{X}^{(n)}(u)>n^{\upsilon}\mathsf{Y}^{(n)}(u) \Big) \leq n^{- D},
	\end{equation*}   
	for all $n \geq  n_0(\upsilon, D)$. In addition, we say that an $n$-dependent event $\Omega \equiv \Omega(n)$ holds {\em with high probability} if for any large $D>1$, there exists $n_0=n_0(D)>0$ so that $\mathbb{P}(\Omega) \geq 1-n^{-D},$
	for all $n \geq n_0.$ 
\end{defn}

We interchangeably use the notation $\mathsf{X}=\OO_{\prec}(\mathsf{Y})$, $\mathsf{X} \prec \mathsf{Y}$ or $\mathsf{Y}\succ \mathsf{X}$  if $\mathsf{X}$ is stochastically dominated by $\mathsf{Y}$, uniformly in $u\in\mathsf{U}^{(n)}$, when there is no risk of confusion. For two sequences of deterministic positive values $\{a_n\}$ and $\{b_n\},$ we write $a_n=\OO(b_n)$ if $a_n \leq C b_n$ for some positive constant $C>0.$ In addition, if both $a_n=\OO(b_n)$ and $b_n=\OO(a_n),$ we write $a_n \asymp b_n.$ Moreover, we write $a_n=\oo(b_n)$ if $a_n \leq c_n b_n$ for some positive sequence $c_n \to 0.$
For any probability measure $\sfP$ over $\Omega$, we denote $\mathcal{L}_2(\Omega, \sfP)$ as the  collection of $L_2$-integrable functions with respect to $\sfP$, that is, for any $f\in \mathcal{L}_2(\Omega, \sfP)$, we have $\|f\|_{\sfP}=\sqrt{\int_{\Omega} |f(y)|^2\sfP(dy)}<\infty$. 
For a vector $\bold{a} = (a_1,...,a_n)^\top \in \mathbb{R}^{n}$, we define its $\ell_p$ norm as $\| \bold{a} \|_p = \big(\sum_{i=1}^n |a_i|^p\big)^{1/p}$.  
We denote $\text{diag}(a_1,...,a_n)\in\R^{n\times n}$ as the diagonal matrix whose $i$-th diagonal entry is $a_i$.
For a matrix $ \bold{A}=(a_{ij})\in \R^{n\times n}$,  we define its Frobenius norm as $\| \bold{A}\|_F = \sqrt{ \sum_{i=1}^{n}\sum_{j=1}^{n} a^2_{ij}}$, 
and its operator norm as $\| \bold{A} \| =\sup_{\|\bold{x}\|_2\le 1}\|\bold{A}\bold{x}\|_2 $.
For any integer $n>0$, we denote the set $[n]=\{1,2,...,n\}$. 
For a random vector $\mathbf{g},$ we say it is sub-Gaussian if $
\mathbb{E} \exp(\mathbf{a}^\top \mathbf{g}) \leq \exp\left( \| \mathbf{a} \|_2^2/2 \right)$ 
for any deterministic vector $\mathbf{a}.$
Throughout, $C,C_1,C_2,...$ are universal constants independent of $n$, and can vary from line to line.

\section{Kernel-Spectral Embedding for High-Dimensional Noisy Data} \label{method.sec}

In this section, we introduce our proposed bandwidth selection and kernel-spectral embedding algorithm. After that, we discuss its unique features  compared with some popular existing methods. 

\subsection{Data-Adaptive Embedding Algorithm} \label{alg.sec}

Our proposed embedding method  is summarized as Algorithm \ref{al0}. 

\begin{algorithm}[t]
	\caption{High-dimensional noisy kernel-spectral embedding} \label{al0}
	\begin{algorithmic}
		\State {\bf Input:} Observed samples $\{\yb_i\}_{i\in[n]}$, {kernel function $f$,} { percentile} $\omega\in(0,1)$ and {eigenvector index set} $\Omega\subseteq\{1,2,...,n\}$. 
		\State 1. {\bf Bandwidth selection:} 
		\State \hspace{4mm} (i) let $d_{ij}=\|\yb_i-\yb_j\|_2^2$  for all $1\le i<j\le n$, and define the empirical cumulative distribution function
		\beq
		\nu_n(t)=\frac{2}{n(n-1)}\sum_{1\le i<j\le n}1_{\{d_{ij}\le t\}},
		\eeq 
		\State \hspace{4mm}  (ii) define the bandwidth $h_n$ to be the 
		solution to the equation
		\begin{equation}\label{eq_bandwidthselection}
		\nu_n(h_n)=\omega. 
		\end{equation}
		\State 2. {\bf Kernel matrix construction:} define the kernel matrix $\bK_n =(K(i,j))_{1\le i,j\le n}$ using the above bandwidth $h_n$ by letting
		\begin{equation}\label{eq_kernelmatrix}
		K(i,j)=f\left(\frac{\| \yb_i-\yb_j \|_2}{h^{1/2}_n} \right). 
		\end{equation} 
		\State 3. {\bf Spectral embedding:} 
		\State \hspace{4mm} (i) obtain the eigendecomposition of the scaled kernel matrix $n^{-1}\bK_n$ as
		\begin{equation} \label{eigen_decomp}
		\frac{1}{n}\bK_n=\bU {\bLam} \bU^\top,
		\end{equation} 
		where $\bLam=\text{diag}(\lambda_1,\lambda_2,...,\lambda_n)\in\R^{n\times n}$ with $\lambda_1\ge\lambda_2\ge...\ge\lambda_n$ being the eigenvalues of $n^{-1}\bK_n$, and $\bU=[\ub_1\quad\ub_2\quad...\quad \ub_n]\in\R^{n\times n}$ with $\{\ub_i\}_{1\le i\le n}$ being the corresponding eigenvectors.
		\State \hspace{4mm} (ii) define the kernel-spectral embeddings as the rows of $\bU_{\Omega} \bLam_{\Omega}\in\R^{n\times |\Omega|}$, where $\bU_{\Omega}\in\R^{n\times |\Omega|}$ and $\bLam_{\Omega}\in\R^{|\Omega|\times |\Omega|}$ only contain the eigenvectors and eigenvalues indexed by the elements in $\Omega$.  
		\State {\bf Output:} the embedding matrix $\bU_{\Omega} \bLam_{\Omega}\in\R^{n\times |\Omega|}$.
	\end{algorithmic}
\end{algorithm}	

In Step 1 of the algorithm,  a data-adaptive bandwidth parameter $h_n$ is defined as the $\omega$-{percentile} of the empirical cumulative distribution function $\nu_n(t)$ of the pairwise squared-distances $\{d_{ij}\}_{1\le i<j\le n}$ among the observed data. Such a strategy is motivated by our theoretical analysis of the spectrum of kernel random matrices and its dependence on the associated bandwidth parameter. It ensures the thus determined bandwidth $h_n$ adapts well to the unknown nonlinear structure and the signal-to-noise ratio of the data, so that the associated kernel matrix captures the respective underlying low-dimensional structure via an integral operator; see Section \ref{theory.sec} for more detail. {The percentile} $\omega$ is a tunable parameter. {In Section \ref{theory.sec}, we show in theory  that under the assumption of (\ref{eq_aspect}) $\omega$ can be chosen as any constant between 0 and 1 to have the final embeddings  achieve the same asymptotic behavior; in Section \ref{data.sec}, we also demonstrate numerically that the final embeddings are insensitive to the choice of $\omega$. In practice, to optimize the empirical performance and improve automation of the method, we recommend using a resampling approach, described in Section \ref{choiceomegereampls} of our supplement \citep{suppl}, to determine {the percentile} $\omega$. }

{In Step 2, some function $f$ is adopted for the construction of the kernel matrix. The choice of $f$ is flexible and depends on  specific applications. In general, our theoretical results indicate that any kernel function that is bounded,  H\"older continuous, and positive semidefinite (cf. Assumption \ref{ker.assu}), can be used here with guaranteed performance. Specifically, the kernel matrices under these kernel functions are provably consistent in terms of their spectral convergence to the associated underlying population integral operators; see Theorem \ref{thm_mainthm2} for more detail.} 

In Step 3, the final embeddings are defined as the $\Omega$-indexed leading eigenvectors of the scaled kernel matrix, weighted by their eigenvalues. Similar forms of embeddings  have been considered in \cite{amini2021concentration,hofmeyr2019improving} for spectral clustering and in \cite{JMLR:v20:18-170} for network clustering. The {eigenvector index set} $\Omega$ of the embedding space is determined by the users, depending on the specific aims or downstream applications of the low-dimensional embeddings. For example, in Section \ref{cell.order.sec} and Section \ref{suppl_additionalrealexamples} of \cite{suppl}, we choose $\Omega=\{1,2\}$ for visualizing a path manifold and $\Omega=\{2\}$ for the downstream ranking task, choose $\Omega=\{2,3\}$ for learning circle manifolds, and set $\Omega=\{1,2,...,\mathsf{r}\}$ for a variety of integers $\mathsf{r}$ for downstream clustering purposes.

\subsection{Comparison with Existing Kernel-Spectral Embedding Methods} \label{comp.sec}

First, regardless of the kernel functions being used, Algorithm \ref{al0} has important differences from the existing kernel-spectral embedding methods. We now focus on kPCA, Lapalcian eigenmap, and DM due to their close relations to our proposal.

Compared to Step 3 of Algorithm \ref{al0}, these methods rely on  eigendecomposition of matrices distinct from $n^{-1}\bK_n$. Specifically, in the standard kPCA implementation (Section 12.3 of \cite{bishop2006pattern}), the eigendecomposition is applied to the centred kernel matrix
\beq \label{kpca}
\bar\bK_n=\bK_n-\frac{1}{n}{\bf 11^\top}\bK_n-\frac{1}{n}\bK_n{\bf 11^\top}+\frac{1}{n^2}{\bf 11^\top}\bK_n{\bf 11^\top},
\eeq
where ${\bf 1}\in\R^n$ is an all-one vector;
in the Laplacian eigenmap \citep{Belkin_Niyogi:2003}, the eigendecomposition is applied to the kernelized graph Lapalcian  matrix
\beq\label{eq_GL}
\bL_n={\bf I}-\bD_n^{-1}\bK_n,
\eeq
where $\bD_n=\text{diag}(\sum_{j=1}^nK(1,j), ..., \sum_{j=1}^nK(n,j))$; in the diffusion map \citep{Coifman_Lafon:2006}, the eigendecomposition is applied to the normalized kernel matrix (i.e., transition matrix)
\beq \label{eq_tran}
\bM_n={\bD'}_n^{-1}\bK'_n,\qquad \bK'_n=\bD_n^{-\zeta}\bK_n\bD_n^{-\zeta},
\eeq
where $\zeta>0$ is some tunable parameter and $\bD'_n=\text{diag}(\sum_{j=1}^nK'(1,j), ..., \sum_{j=1}^nK'(n,j))$. More importantly, our numerical results (Section \ref{data.sec} and Section \ref{suppl_additionalrealexamples} of \cite{suppl}) indicate that such  differences may lead to  low-dimensional embeddings distinct from our method. Theoretically (see Section \ref{op.sec}), we show that our proposed method may  capture geometric features that are not likely to be captured by these methods. 

As for other existing methods, it is largely unclear to what extent their performance is guaranteed, or if their interpretations continue to hold, for general high-dimensional and noisy datasets. 
Empirically, our real applications in Section \ref{cell.order.sec} and Section \ref{suppl_additionalrealexamples} of \cite{suppl} show that our proposed method outperforms various existing methods, especially when dimension increases. For example, see Figure \ref{cls.fig1} and Figure \ref{embed.fig22} of \cite{suppl} for illustrations.    

Compared to many existing methods, Algorithm \ref{al0} has a bandwidth selection procedure that is provably coherent with the subsequent  steps, in the sense that both the high-dimensionality and noisiness of the data are intrinsically accounted for in the constructed kernel matrix  and the final embeddings. In contrast, graph-based methods, usually require the bandwidth to decrease to zero, to ensure a meaningful convergence. As a result, these methods may suffer from sub-optimality for high-dimensional and noisy datasets (see discussions at the end of Section \ref{ker.sec}). In this respect, we are only aware of the recent work \cite{DW2}, in which a data-driven bandwidth selection procedure for GL was developed. In Section \ref{data.sec}, we compare various bandwidths and demonstrate the advantage of the proposed bandwidth.

{Second, in terms of downstream analysis like clustering, our method also differs from many existing methods. 
A graph-cut based spectral clustering algorithm was proposed in \citep{hofmeyr2019improving}. Instead of using the kernel matrix eigendecomposition as in (\ref{eigen_decomp}), \citep{hofmeyr2019improving} focuses on the graph Laplacian matrix of a zero-diagonal Gaussian kernel matrix. Moreover, it was noted in \cite{hofmeyr2019improving} that the bandwidth $h_n \asymp n^{-1/(r+4)}$ which relied on prior knowledge of the ambient dimension $r$ was chosen to ensure the consistency of the normalized cut rather than the kernel matrices. 
%
%
\citep{loffler2021optimality} studied the spectral embedding and clustering of high-dimensional data under the Gaussian mixture model, where the embedding was obtained based on spectral decomposition of the observed data matrix, rather than kernel matrices. In addition, \cite{abbe2020ell_p} developed a perturbation theory for a zero-diagonal version of PCA in general Hilbert spaces. Similar to our paper, \cite{abbe2020ell_p} considered spectral embedding of noisy high-dimensional data and studied the eigenvector convergence to their counterparts associated with the noiseless samples.  However, unlike the current work, \cite{abbe2020ell_p} focused on the eigendecomposition of some general Gram matrices, rather than distance kernel matrices as in (\ref{Kmat}) and (\ref{Kn*}). Moreover, \cite{abbe2020ell_p} considered a model in the RKHS rather than in the original sample space as in (\ref{eq_basicmodel}) and Assumption \ref{assum_generalmodelassumption}. Although the theoretical assumptions are comparable to ours in some special cases (see Section \ref{compare.sec} of our supplement \cite{suppl} for more detail), the assumptions of \cite{abbe2020ell_p} can be relatively less interpretable compared to ours, which are made directly on the underlying manifolds, the observed samples and the kernel functions.  Eigenvector perturbation bounds were also obtained in \cite{abbe2020ell_p} under various discrepancy measures, which were distinct from ours (\ref{eq_projectionbound}). Although the results of \cite{abbe2020ell_p} also highlighted the role of eigengap in the final rates of convergence, our results, as shown below, unveiled the impact of the underlying manifold structures and the choice of kernel functions on the final rate of convergence, as well as the proper interpretation of the eigenvectors with respect to the underlying population integral operators. For more details of the discussions, we refer the readers to Section \ref{compare.sec} of our supplement \cite{suppl}. Comparisons of the empirical performance with these methods can be found in Section \ref{cluster.sec}. } 



\section{Theoretical Properties: Justifications and Interpretations} \label{theory.sec}

We develop a theoretical framework to rigorously justify the embedding algorithm proposed in Section \ref{alg.sec}.  Our theory suggests an asymptotic geometric interpretation that helps to better understand  the  low-dimensional embeddings.

\subsection{Manifold Model and Reduction Scheme} \label{man.sec}


To facilitate our analysis of the proposed embedding algorithm under  model (\ref{eq_basicmodel}), we introduce the commonly used nonlinear manifold model and a useful model reduction scheme, that significantly simplifies the theoretical analysis without sacrificing the generality of our discussions. {For some necessary background on smooth manifold and Riemannian geometry, we refer the readers to Section \ref{supple_sec_Riemmanian} of our supplement \cite{suppl} for more detail. We first introduce the general manifold model as considered in \cite{barp2018riemann,cheng2013local,ding2021kernel,dunson2020graph,shen2022robust,shen2020scalability, wu2018think}.
\begin{assu}\label{assum_generalmodelassumption}
We assume that $\mathbf{x}_i, \ 1 \leq i \leq n,$ are independent and identically distributed (i.i.d.) samples of a random vector $X: \Omega \rightarrow \mathbb{R}^p$ with respect to some probability space $(\Omega, \mathcal{F}, \mathbb{P}).$ Furthermore, we assume that the range of $X$ is supported on an $m$-dimensional connected Riemannian manifold $\mathcal{M}$  isometrically embedded in $\mathbb{R}^p$ via $\iota: \mathcal{M} \rightarrow \mathbb{R}^p.$ Suppose that the dimension of the space $\iota(\mathcal{M})$ satisfies
\begin{equation*}
\operatorname{dim}(\iota(\mathcal{M}))=r. 
\end{equation*}
Note that $r \leq p.$ To properly define the probability density function (p.d.f) of $X,$ let $\widetilde{\mathcal{F}}$ be the Borel sigma algebra of $\iota(\mathcal{M})$ and denote $\widetilde{\mathbb{P}}$ as the probability measure of $X$ defined on  $\widetilde{\mathcal{F}}$ induced from $\mathbb{P}.$ We assume that $\widetilde{\mathbb{P}}$ is absolutely continuous with respect to the volume measure on $\iota(\mathcal{M}).$  
\end{assu}

Assumption \ref{assum_generalmodelassumption} is commonly used in machine learning and manifold learning literature to model the nonlinear structures of the observed data sets. First, the connectedness of $\mathcal{M}$ guarantees that $\widetilde{\mathbb{P}}$ corresponds to a continuous  random vector. Second, in our current paper, instead of using the restrictive setting of \cite{DW2} where $r=1$, we allow $r$ to be generic, which can diverge with $p.$ Consequently, our discussion does not depend on the specific mapping $\iota$. This allows us to model the nonlinear structure in a more flexible way. In fact, according to Nash's embedding theorem \citep{MR75639} (see Theorem \ref{thm_nashembedding} of our supplement \cite{suppl}), it is possible that $r \leq m(3m+11)/2$ for compact $\mathcal{M}$ and $r \leq m(m+1)(3m+11)/2$ for non-compact $\mathcal{M}.$ In this sense, we can allow either $m$ or $r$ to diverge with $p.$ Finally, in real applications, usually it is the embedded submanifold $\iota(\mathcal{M})$ that matters since the observations are sampled according to $X$ which is supported on $\iota(\mathcal{M}).$ Consequently, we focus on the understanding of the geometric structures of $\iota(\mathcal{M})$ rather $\mathcal{M}$ and $\iota$ separately. However, our discussion is naturally related to $\mathcal{M}$ in the following sense.  Let $\mathrm{g}$ be the metric associated with the Riemannian manifold $\mathcal{M}$, $\dd V$ be the volume form associated with $\mathrm{g}$ and $\iota_* \dd V$ be the induced measure on $\iota(\mathcal{M})$. Then by Radon-Nikodym theorem (e.g., \citep{billingsley2008probability}), for some differentiable function $\mathsf{f}$ defined on $\mathcal{M}$, under Assumption \ref{assum_generalmodelassumption}, we have that for $x \in \iota(\mathcal{M})$
\begin{equation}\label{eq_density}
\dd \widetilde{\mathbb{P}}(x)=\mathsf{f}(\iota^{-1}(x)) \iota_* \dd V(x). 
\end{equation}
$\mathsf{f}$ is commonly referred as the p.d.f of $X$ on $\mathcal{M}.$ For example, if $\mathsf{f}$ is constant, we call $X$ a uniform random sampling scheme. With the above setup, we can easily calculate the expectation with respect to $X$ on the embedded manifold $\iota(\mathcal{M}).$ More specifically, for an integrable function $\zeta: \iota(\mathcal{M}) \rightarrow \mathbb{R},$ we have that 
\begin{align*}
\mathbb{E} \zeta(X) & =\int_{\Omega} \zeta(X (\omega)) \dd \mathbb{P}(\omega)=\int_{\iota(\mathcal{M})} \zeta(x) \dd \widetilde{\mathbb{P}}(x) \nonumber \\
&=\int_{\mathcal{M}} \zeta(x) \mathsf{f}(\iota^{-1}(x)) \iota_* \dd V(x)=\int_{\mathcal{M}} \zeta(\iota(y)) \mathsf{f}(y) \dd V(y). 
\end{align*} 
}

{Under Assumption \ref{assum_generalmodelassumption},} there exists a  rotation matrix $\bR\in\R^{p\times p}$ only depending on {$\mathcal{M}$ and $\mathrm{g}$} that 
\begin{equation}
\bR \xb_i=(x_{i1}, x_{i2}, \cdots, x_{ir}, {0, \cdots, 0})^\top\in\R^p.
\end{equation}
Recall that $\{\xb_i\}_{1\le i\le n}$ are drawn independently {according to $X$}. Let $\bSig$ be the covariance matrix of $(x_{i1}, x_{i2}, \cdots, x_{ir})^\top$, and denote its eigendecomposition  as
\begin{equation}
\bSig=\bV\bGam\bV^\top,
\end{equation}
where $\bGam=\text{diag}(\theta_1,...,\theta_r)$ contains the eigenvalues of $\bSig$, arranged in nonincreasing order, and the columns of $\bV\in\R^{r\times r}$ are the corresponding eigenvectors. Given $\bV$, we could define an orthonormal matrix ${\bO}\in\R^{p\times p}$ as
\begin{equation}
{\bO}=
\begin{pmatrix}
\bV^\top & {\boldsymbol{0}}\\
{\boldsymbol{0}}& {\bf I}_{p-r}
\end{pmatrix}.
\end{equation}
Then it is clear that for each $i\in[n]$, the covariance matrix of $ \bO \bR\xb_i$ is diagonal, with
\beq\label{eq_OR}
\text{Cov}( \bO \bR\xb_i) = \begin{pmatrix}
	\bGam & {\boldsymbol{0}}\\
	{\boldsymbol{0}}& {\boldsymbol{0}}
\end{pmatrix}.
\eeq
Therefore, by rotating the original model (\ref{eq_basicmodel}) with the matrix $\bO \bR$, we obtain that 
\begin{equation} \label{model_red}
\yb_i^0= \xb_i^0+\zb_i^0,
\end{equation} 
where  $\yb_i^0= \bO \bR \yb_i$, $\xb_i^0=\bO \bR \xb_i$ and $\zb_i^0= \bO \bR\zb_i.$
In particular, after rotation  $\{\xb_i^0\}_{1\le i\le n}$ become independent block random vectors in the sense that 
\begin{equation}\label{eq_reducedmodel}
\xb^0_i=(\bm{x}_i^0, 0, \cdots, 0)\in\R^p, 
\end{equation}
where $\bm{x}^0_i=(x_{i1}^0, \cdots, x_{ir}^0)\in\R^r$ and 
\begin{equation}\label{eq_modelcov}
\operatorname{Cov}(\bm{x}_i^0)=\bGam=\operatorname{diag}(\theta_1, \cdots, \theta_r),
\end{equation}
whereas the random noises $\{\zb_i^0\}_{1\le i\le n}$ remains independent and centred.
Now a key observation is that the kernel matrix $\bK_n$ from (\ref{eq_kernelmatrix})  only depends on the pairwise Euclidean distances of the data $\{\yb_i\}_{1\le i\le n}$, which are invariant to any rotations. In other words, we have
\begin{equation*}
f\left(\frac{\| \yb_i-\yb_j \|_2}{h_n^{1/2}} \right)=f\left(\frac{\| \yb^0_i-\yb^0_j \|_2}{h_n^{1/2}} \right),\qquad 1\le i,j\le n.
\end{equation*}
Therefore, to theoretically analyze the kernel random matrix $\bK_n$ and the subsequent spectral embeddings, we can invariably start from the structurally reduced  model as characterized by (\ref{model_red}) to (\ref{eq_modelcov}). {Based on the above discussion, we arrive at the following result.}
{
\begin{lem}\label{lem_modelreduced}
Suppose Assumption \ref{assum_generalmodelassumption} holds. Then the kernel random matrices defined under model (\ref{eq_basicmodel}) are equivalent to those under the reduced model characterized by (\ref{model_red}) to (\ref{eq_modelcov}).  
\end{lem}
}

The main benefit of such a reduction scheme is two-fold. On the one hand, it helps to identify a natural coordinate system for the noiseless samples $\{\xb_i\}_{1\le i\le n}$ with fewer informative components, compared to the original coordinate system. On the other hand, the reparametrization leads to a parsimonious model in which the sampling distribution from the nonlinear manifold model as in Assumption \ref{assum_generalmodelassumption} has a simple diagonal covariance structure. Both aspects contribute to improving the theoretical accessibility of the problem without loss of generality.

\subsection{Spectral Convergence to Noiseless Kernel Matrix} \label{ker.sec}

We show the convergence of the low-dimensional embeddings to their noiseless counterparts, by establishing the spectral convergence of the scaled kernel random matrix $n^{-1}\bK_n$ to the noiseless kernel matrix $n^{-1}\bK_n^*$, where
\beq \label{Kn*}
\bK_n^*=(K^*(i,j))_{1\le i,j\le n},\qquad K^*(i,j)=f\left(\frac{\| \xb_i-\xb_j \|_2}{\sfh^{1/2}} \right),
\eeq
In (\ref{Kn*}), $\sfh \equiv \sfh_n$ is the noiseless counterpart of $h_n$, defined as 
\beq \label{sfh}
\nu^*_n(\sfh)=\omega,
\eeq
where $\nu^*_n(t)=\frac{2}{n(n-1)}\sum_{1\le i<j\le n}1_{\{d^*_{ij}\le t\}}$,
and $d^*_{ij}=\|\xb_i-\xb_j\|_2^2$  for all $1\le i<j\le n$.  
 {To simplify our presentation, in the following two subsections we will first consider the Gaussian kernel function (cf. $f(x)=\exp(-x^2)$) as a demonstrating example due to its wide applications. The theoretical results concerning other kernel functions and their properties are deferred to Section \ref{ker.alt.sec}.}

By the invariance $\| \xb_i-\xb_j \|_2=\| \bm{x}^0_i-\bm{x}^0_j \|_2$, the noiseless kernel matrix $\bK_n^*$ can also be treated as a characterization of the noiseless samples  in $\R^r$.
To this end, we first introduce and discuss the theoretical assumptions we made throughout our analysis.  According to Lemma \ref{lem_modelreduced}, it suffices to focus on the reduced model (\ref{model_red}) to (\ref{eq_modelcov}).

\begin{assu}\label{assum_mainassumption}
Suppose (\ref{eq_aspect}) holds. We assume that both $\{\xb_i^0 \}_{1\le i\le n}$ and $\{\zb^0_i\}_{1\le i\le n}$ in (\ref{model_red}) are independent sub-Gaussian random vectors with independent entries, and
	\beq \label{eq_meanandcov}
	\mathbb{E} \zb_i=0,\qquad \textup{Cov}(\zb_i)=\sigma^2{\bf I}_p,
	\eeq
	with $\sigma^2\asymp n^{\beta}$ for some constant $\beta\ge 0$.
	{Moreover, we assume that $\bm{x}_i^0$ has a continuous density function which is bounded below away from zero.  Finally, in (\ref{eq_modelcov}), we assume that for each $1 \leq i \leq r$, we have $\theta_i \asymp n^{\alpha_i},$ for some constants $\alpha_1 \geq \alpha_2 \geq \cdots \geq  \alpha_r \geq 0,$ and
	{
\begin{equation*}
p\sigma^2=\mathrm{o} \left( \sum_{i=1}^r \theta_i \right),
\end{equation*}	
	or equivalently, }
	\begin{equation}\label{eq_sigmaimagnititude}
{ n^{\beta+\eta}=\mathrm{o}\left(\sum_{i=1}^r n^{\alpha_i}\right).}
	\end{equation}	
	}
\end{assu}

A few remarks on Assumption \ref{assum_mainassumption} are in order. {First, in terms of the notations in Assumption \ref{assum_generalmodelassumption}, we require $\widetilde{\mathbb{P}}$ to be sub-Gaussian and the density function $\mathsf{f}$ defined in (\ref{eq_density}) to be bounded away from zero. Moreover, if the manifold $\mathcal{M}$ is compact, $\widetilde{\mathbb{P}}$ will be bounded and hence sub-Gaussian. In this sense, our results hold for both compact manifold and non-compact manifold with some mild decay assumptions on the induced measure $\widetilde{\mathbb{P}}.$ Second, unlike most of the existing literature on manifold learning where $r$ is usually assumed to be independent of $n$ and bounded \cite{cheng2013local,DW2,ding2021kernel,dunson2020graph,shen2022robust,shen2020scalability, wu2018think}, our assumption  allows $r$ to possibly diverge with $n$. 
The general signal-to-noise ratio assumption { (\ref{eq_sigmaimagnititude}) or its equivalence $\sigma^2 p =\mathrm{o} (\sum_i \theta_i)$ guarantees the ability to learn the structures of the embedded submanifold through the noisy kernel matrix $\bK_n$.} When $r$ is bounded, we only  need one signal that { $\alpha_i \gg \beta+\eta$} for some $1 \leq i \leq r.$} In fact, it has been shown in \cite{DW2} that when $r=1$ and $\eta=1$, such a minimal signal condition is also necessary--if this condition fails, we will not able to obtain useful information through $\mathbf{K}_n$ as the noise will dominate the signals.  
Third, our model does not require the signals $\bm{x}_i^0$ to be isotropic, that is, we allow their marginal variances to differ freely as in (\ref{eq_modelcov}). 
Finally, in the current paper, we assume that $\{\zb_i\}_{1\le i\le n}$ and therefore $\{\zb_i^0\}_{1\le i\le n}$ are white noise as in (\ref{eq_meanandcov}) for simplicity. Nevertheless, our discussion applies invariably to the colored noise setting where $\operatorname{Cov}(\zb_i)$ has the same eigenspace as $\operatorname{Cov}(\xb_i).$ For example, in light of (\ref{eq_modelcov}), we can allow a diagonal but non-scalar covariance matrix for $\operatorname{Cov}(\zb^0_i).$

Before stating the main result of this part, we introduce a few more notations.
We denote the eigenvalues of $n^{-1}\bK_n^*$ as $\{\mu_i\}_{1\le i\le n}$ where $\mu_1\ge \mu_2\ge...\ge\mu_n$, and denote the corresponding eigenvectors as $\{\vb_i\}_{1\le i\le n}$. Let $\sfP$ be the probability measure on $\mathbb{R}^r$ for the independent subvectors $\{\bm{x}^0_i\}_{1\le i\le n}$ of the reduced noiseless samples $\{\xb_i^0\}_{1\le i\le n}$. We define the associated integral operator $\mathcal{K}$ with the kernel function $f$ such that, for any $g \in \mathcal{L}_2(\mathbb{R}^r, \sfP),$   {
\begin{equation}\label{eq_intergraldefinition}
\mathcal{K}g(\bm{x})=\int f\left(\frac{\| \bm{x}-\bm{y} \|_2}{\mathsf{h}^{1/2}}\right)g(\bm{y}) \sfP(\dd \bm{y}) , \qquad \bm{x}, \bm{y} \in \mathbb{R}^r. 
\end{equation} }
By Mercer's theorem (e.g., \cite{jorgens1982linear}), there exist a sequence of nonnegative eigenvalues $\{\gamma_i\}_{i\ge 1}$ in the decreasing order and orthonormal basis of $\mathcal{L}_2(\mathbb{R}^r, \sfP),$ known as eigenfunctions so that 
\beq
\mathcal{K} \phi_i(\bm{x})=\gamma_i \phi_i(\bm{x}),\qquad \text{for $i\ge 1$.}
\eeq
Finally, we also recall that the eigenvalues and eigenvectors of $n^{-1}\bK_n$ are  $\{\lambda_i\}_{1\le i\le n}$ and $\{\ub_i\}_{1\le i\le n}$, respectively.


{
\begin{thm}\label{thm_mainthm}
	Suppose Assumption \ref{assum_mainassumption} holds and the kernel function $f(x)=\exp(-x^2)$. Define 
	\beq \label{psi}
{ 	\psi_n:= \frac{\sigma}{(\sum_{i=1}^r\theta_i)^{1/2}}+\frac{\sigma^2 p}{\sum_{i=1}^r\theta_i}\asymp\frac{1}{(\sum_{i=1}^rn^{{\alpha_i-\beta}})^{1/2}}+\frac{1}{\sum_{i=1}^rn^{\alpha_i-\beta-\eta}}.}
	\eeq
	Then the following holds.
	\begin{enumerate}
		\item(Eigenvalue convergence) We have
		\begin{equation}\label{eq_matrixclose1}
		\left\| n^{-1}\bK_n-n^{-1}\bK_n^* \right\|=\OO_{\prec}\left( \psi_n\right).
		\end{equation}
		Therefore, by Weyl’s inequality, we have 
		\begin{equation} \label{eq_eigs_v}
		\max_{i\in [n]} |\lambda_i-\mu_i|=\OO_{\prec}\left( \psi_n\right).
		\end{equation}
		\item(Eigenvector convergence) For any $i\in[n]$, if the $i$-th population eigen-gap 	$\mathsf{r}_i:=\min\{\gamma_{i-1}-\gamma_i, \gamma_i-\gamma_{i+1}\}$
		satisfies that {
		\begin{equation}\label{eq_assumption}
		 \frac{1}{\sqrt{n}}+\frac{1}{\sum_{i=1}^r n^{\alpha_i-\beta-\eta}}=\mathrm{o}(\sfr_i^2),
		\end{equation} }
		then we have
		\begin{equation}\label{eq_projectionbound}
		| \langle \ub_i, \vb_i \rangle^2-1 |=\OO_{\prec}\left( \frac{\psi_n}{\sfr_i^2}\right).
		\end{equation}
	\end{enumerate}
\end{thm}
}


Theorem \ref{thm_mainthm} establishes the asymptotic spectral equivalence between the noisy kernel matrix $n^{-1}\bK_n$ and the noiseless kernel matrix $n^{-1}\bK_n^*$. Note that both matrices are random with respect to the manifold model and noise. {We first provide some explanation of the convergence rates. For the eigenvalues, the rate $\psi_n$ consists of two parts. Under Assumption \ref{assum_mainassumption}, we note that the second term of the right-hand side of (\ref{psi}) satisfies {$\left( \sum_{i=1}^r n^{\alpha_i-\beta-\eta} \right)^{-1} \asymp p \sigma^2/(\sum_{i=1}^r \theta_i).$} Here, the factor $p \sigma^2$ can be understood as the overall noise level and $\sum_{i=1}^r \theta_i$ is the overall signal strength. Therefore, this term quantifies the impact of signal-to-noise ratio (SNR). Technically, the pairwise distance has three parts following $\| \yb_i-\yb_j \|_2^2=\|\xb_i-\xb_j \|_2^2+\|\zb_i-\zb_j \|_2^2+2(\xb_i-\xb_j)^\top (\zb_i-\zb_j).$  As proved in Proposition \ref{lem_bandwidthconcentration} below, with high probability, $h_n \asymp \sum_{i=1}^r \theta_i.$ Therefore, such an SNR term comes from the pure noise component $\|\zb_i-\zb_j \|_2^2/h_n.$ Similarly, for the first term of the right-hand side of (\ref{psi}), it comes from the cross-product $(\xb_i-\xb_j)^\top (\zb_i-\zb_j)/h_n.$ This explains the rate $\psi_n$ heuristically. For the eigenvectors, the rate in (\ref{eq_projectionbound}) contains two parts. The $\mathsf{r}_i$ part  concerns the eigengap which appears frequently in the analysis of eigenvectors of large random matrices \cite{bao2022statistical,bao2021singular,fan2018eigenvector}; 
the $\psi_n$ part  comes from the eigenvalue convergence. {Note that the left-hand side of (\ref{eq_assumption}) is much larger than $\psi_n$ so that (\ref{eq_projectionbound}) implies consistent estimation.} { The rate function $\psi_n$ reveals the impact of $p$ and $\sigma^2$ on the rates of convergence. Specifically, when $n\gg p,$ or equivalently $\eta\in(0,1)$ in (\ref{eq_aspect}), the second term $(\sum_{i=1}^rn^{\alpha_i-\beta-\eta})^{-1}$ of $\psi_n$ will be smaller which suggests a faster rate of convergence compared to some exiting results assuming $r=1$ and $\eta=1$ \cite{DW2}. On the other hand, when $p\gg n$, our theory suggests that spectral convergence is still possible, although in this case a stronger overall signal strength $\sum_i n^{\alpha_i}$ is required in light of (\ref{eq_sigmaimagnititude}) for larger $\eta$.}  { We also point out that comparing to the Davis-Kahan theorem \cite[Theorem 4.5.5]{vershynin2018high}, our result in (\ref{eq_projectionbound}) is general in the sense that $\mathsf{r}_i$ only depends on three population eigenvalues and  is allowed to decay to zero. Moreover, if we have stronger assumption that $\delta$ is bounded from below by a constant, one may directly apply Davis-Kahan theorem \cite[Theorem 4.5.5]{vershynin2018high}  so that the right-hand side of (\ref{eq_projectionbound}) can be improved to $\mathrm{O}_{\prec}(\psi_n^2).$ For generality, we choose to keep the current (\ref{eq_projectionbound}). For more details, we refer the readers to Section \ref{suppl_additionalremark} of our supplement \cite{suppl}.}

We emphasize that unlike \cite{DW2},  in the current paper, we do not impose assumptions on $r.$ In general, $r$ is some generic parameter and can grow with $n.$ The only condition we used is the overall SNR in (\ref{eq_sigmaimagnititude}). In fact, (\ref{eq_matrixclose1}) improves the results of \cite{DW2} (Theorem 3.1)  in the following sense. In \cite{DW2}, the authors mainly consider { $r=\eta=1$ } and $\beta=0$; under the assumption of (\ref{eq_sigmaimagnititude}), our rate $\psi_n \asymp n^{-\alpha_1/2}+n^{-\alpha_1+1}$ improves the rate in \cite{DW2} which reads as $n^{-1/2}+n^{-\alpha_1+1}.$ In addition, when $r$ diverges, unlike \cite{DW2}, which requires $\alpha_i>\beta+\eta$ for all $1 \leq i \leq r,$ we only need a much weaker condition (\ref{eq_sigmaimagnititude}). In fact, our Theorem \ref{thm_mainthm} reveals the role of $r$ in the convergence rate in a more general setting. For example, setting {$\eta=1$} and $\beta=0,$ when $r$ is divergent in the sense that $r \gg n^{1-\delta}$ for some small $\delta>0,$  we only need $\alpha_i \geq \delta, 1 \leq i \leq r.$ This significantly weakens the conditions for the signals used in \cite{DW2, el2010information,ElKaroui_Wu:2016b}.

}

\begin{rem}
We remark that as shown in Corollary 3.2 of \cite{DW2}, for $i \gg \log n$ and $r=\OO(\log n),$  we have $\mu_i=\OO_{\prec}(n^{-1}).$ Consequently, in practice, we only need to focus on recovering the first few (at most at an order of $\log n$) eigenvalues and eigenvectors of $n^{-1}\bK_n^*$ when $r$ diverges slowly. Moreover, the results concerning the eigengap $\sfr_i$ can be made more precise given additional information. For example, if $\sfP$ is multivariate Gaussian and $\eta=\sigma=1$, then from Section \ref{sec_someresultintegral} of our supplement \citep{suppl}, for any finite $i,$ we have $\sfr_i \asymp 1.$ If we further assume that $r=\OO(1)$, whenever $\alpha_k>3/2$ for some $k\in\{1,2,...,r\}$, we have $\sqrt{n}$-consistency for the leading $N=\OO(1)$ eigenvalues  $
\max_{i\in [N]} |\lambda_i-\mu_i|=\OO_{\prec}(n^{-1/2})$,
and eigenvectors $
1-\min_{i\in[n]}\langle \ub_i, \vb_i \rangle^2=\OO_{\prec}(n^{-1/2}).$

Moreover, we believe our assumption on the signal-to-noise ratio in (\ref{eq_sigmaimagnititude}) is actually necessary to guarantee convergence. Considering the setting when $\eta=1$ and $r=\OO(1),$ as proved in part (i) of Theorem 3.1 of \cite{DW2}, when $\alpha_i<\beta+1$ for all $i\le r$, the noise dominates the signal and the convergence between $n^{-1}\bK_n$ and $n^{-1}\bK^*_n$ fails to hold.  Indeed, if $\sigma=1$ and $\alpha_1<1$, we have $h_n \asymp p$, so that asymptotically $n^{-1}\bK_n$ can be approximated by the Gram matrix of the noise vectors $\{\frac{1}{\sqrt{n}}\zb_i\}$, plus a few  noninformative rank-one spikes. In particular, these rank-one spikes, as pointed out by Remark 2.4 of \cite{DW2}, only depend on the kernel function, and do not contain any information about  $\{\xb_i\}_{1\le i\le n}$; meanwhile, the empirical spectral distribution of $\bK_n$ will be governed by that of the noise Gram matrix and converges to the Marchenko-Pastur law \citep{marvcenko1967distribution}.   See Section \ref{dis.sec} of \cite{suppl} for more discussions.
\end{rem}


As has been seen from Theorem \ref{thm_mainthm}, there are fundamental advantages of the proposed bandwidth. Under the current setup (\ref{eq_aspect}), most of the existing literature regarding the kernel random matrix $\bK_n$ focus on a pre-specified fixed bandwidth \citep{cheng2013spectrum,9205615,do2013spectrum,el2010spectrum}. In particular, when { $\sigma=\eta=1,$} as they mainly deal with the null cases, $h_n=Cp$  is chosen for some constant $C>0$, since the total noise is $\operatorname{tr}(\operatorname{Cov}(\zb_i))= p$. In contrast, under the non-null regime (\ref{eq_sigmaimagnititude}) considered in the current study, if we still choose  $h_n=Cp,$ the bandwidth may be too small compared to the signal ($\asymp \sum_{i=1}^r n^{\alpha_i}$)  so that the decay of the kernel function forces each sample to be "blind" of the other nearby samples. Especially, following the proof of Theorem 2.7 of \cite{DW2}, when the signals are strong enough so that $\sum_{i=1}^r \theta_i >\beta+3,$ if we still choose $h_n=Cp,$ then we have $\lambda_i(\bK_n) \rightarrow 1$ for all $1 \leq i \leq n$ with high probability. In this case, the kernel matrix converges to the identity matrix and becomes useless. 
In comparison, under our non-null regime (\ref{eq_sigmaimagnititude}), the above discussion suggests  that one {should choose $h_n\asymp \sum_{i=1}^r \theta_i$, so that the kernel matrix does not degenerate. In this connection, the following result plays an important role in the proof of Theorem \ref{thm_mainthm}, and justifies the efficacy of the proposed bandwidth selection procedure, without requiring any prior knowledge about the underlying structures such as $m, \iota, \{\theta_i\}_{1\le i\le r}$, or the noise level $\sigma$. 

	
	\begin{prop}\label{lem_bandwidthconcentration}
		Suppose Assumption \ref{assum_mainassumption} holds and $h_n$ is selected according to Algorithm \ref{al0} and $\sfh$ is defined in (\ref{sfh}). Then we have  $\sum_{i=1}^r \theta_i\prec \sfh \prec \sum_{i=1}^r \theta_i$ and $|h_n/\sfh-1|=\OO_{\prec}(\psi_n)$.
	\end{prop}
}

{Proposition \ref{lem_bandwidthconcentration}, on the one hand, implies that with high probability, the population (clean) bandwidth satisfies that $\sfh \asymp \sum_{i=1}^r \theta_i.$ We claim that this is a useful choice of bandwidth in the following sense. First, if $\sfh \ll  \sum_{i=1}^r \theta_i,$ we have that with high probability for $i \neq j,$ $\|\xb_i-\xb_j \|_2^2/\sfh \rightarrow \infty$ and consequently for any decay kernel function $f(x)$ (e.g. $f(x)=\exp(-x^2)$), $f(\|\xb_i-\xb_j \|_2^2/\sfh)=\oo(1).$ As a result, when $\sfh$ is too small, following the proof of Theorem 2.9 of \cite{DW2}, $\bK_n^*$ becomes trivial in the sense that $\bK_n^* \approx \mathbf{I}$ with high probability, where $\mathbf{I}$ is an $n \times n$ identity matrix. Similarly, if $\sfh \gg \sum_{i=1}^r \theta_i,$ we have that with high probability, for $i \neq j,$ $\|\xb_i-\xb_j \|_2^2/\sfh \rightarrow 0$ and consequently  $f(\|\xb_i-\xb_j \|_2^2/\sfh)=f(0)+\oo(1).$ This shows that when $\sfh$ is too large, $\bK_n^*$ will become trivial again in the sense that $\bK_n^* \approx (1-f(0))\mathbf{I}+f(0) \mathbf{1} \mathbf{1}^\top$ with high probability, where $\mathbf{1} \in \mathbb{R}^n$ is a vector with all entries being unity. For illustrations on the usefulness of our population bandwidth $\sfh$, we refer the readers to Section \ref{sec_impactkernelpercentile} and Figure \ref{bw.fig1}  of our supplement \cite{suppl}. Together with the above arguments, we find that our proposed $\mathsf{h}$ is a natural choice to make $n^{-1}\bK_n^*$ nontrivial.  Second, with such a choice of $\sfh,$ it will be seen from Section \ref{op.sec} that together with the kernel function $f(x),$ $n^{-1} \bK_n^*$ will converge to some integral operator of some RKHS asymptotically.  
   
On the other hand, the above proposition also guarantees that the proposed bandwidth $h_n$ is sufficiently close to $\sfh$ to
distinguish signals from the noise. Moreover, by the definition of $h_n$, the validity of the above proposition does not rely on the specific form of the kernel function, making the proposed bandwidth selection method readily applicable for other types of kernel functions; see Section \ref{ker.alt.sec} for more detail. }

\subsection{Spectral Convergence to Population Integral Operator}\label{op.sec}

So far we have shown that our proposed low-dimensional spectral embeddings converge to those associated with the noiseless random samples $\{\xb_i\}_{1\le i\le n}$. However, it is still unclear how these embeddings relate to  $\iota(\mathcal{M})$. Answering this question is important as to understand the proper interpretation of these embeddings.

Our next result concerns the limiting behavior of the eigenvalues and eigenvectors of $\bK_n$, as characterized by a deterministic integral operator defined over the manifold $\iota(\mathcal{M})$. Before stating our main result, we first introduce a few notations from operator theory. Recall that $\sfP$ is the probability measure on $\R^r$ for $\{\bm{x}^0_i\}_{i\in[n]}$ in (\ref{eq_reducedmodel}) and $\bO\bR$ is the rotation matrix satisfying (\ref{eq_OR}). Let $\mathcal{S}=\{\bO\bR\xb: \xb\in  { \iota(\mathcal{M})}\}$ be the clean sample space after rotation. We define $\widetilde{\sfP}$ as the probability measure on $\mathcal{S}$ for $\{\xb^0_i\}_{1\le i\le n}$, which is also $\sfP$ embedded into $\R^p$ such that for $\xb=(\bm{x}, 0,...,0) \in \mathbb{R}^p$, we have $\widetilde{\sfP}(\xb)=\sfP(\bm{x}).$
Accordingly, we define the population integral operator 
$\widetilde{\mathcal{K}}$ such that for any $g\in \mathcal{L}_2(\mathbb{R}^p, \widetilde\sfP),$ we have 
\begin{equation}\label{K.tilde}
\widetilde{\mathcal{K}} g(\xb)=\int f\left(\frac{\| \xb-\yb \|_2}{\sfh^{1/2}}\right)g(\yb) \widetilde{\sfP}(\dd \yb) , \qquad \xb\in \mathcal{S}.
\end{equation}
Clearly, the eigenvalues and eigenfunctions of $\widetilde{\mathcal{K}},$ denoted as $\{\widetilde{\gamma}_i\}_{i\ge 1}$ and $\{\widetilde{\phi}_i(\xb)\}_{i\ge 1}$ satisfy 
\begin{equation} \label{eigs.tilde}
\widetilde{\gamma}_i=\gamma_i, \quad \widetilde{\phi}_i(\xb)=\phi_i(\bm{x}), \qquad \text{for $i\ge 1$.}
\end{equation}
Moreover, let $\widehat{\sfP}_n$ be the empirical distribution of the noisy samples $\{\yb^0_i\}_{1\le i\le n}$. Then for $\quad \xb \in \mathbb{R}^{p},$ the corresponding empirical operator $\widehat{\mathcal{K}}_n$ is defined by
\begin{equation}\label{eq_empericialintegraloperator}
\widehat{\mathcal{K}}_{n} g(\xb)=\int f\left(\frac{\| \xb-\yb \|_2}{h_n^{1/2}}\right) g( \yb) \widehat{\sfP}_n(\dd \yb)=\frac{1}{n}\sum_{i=1}^n f\left(\frac{\| \xb-\yb^0_i \|_2}{h^{1/2}_n}\right) g(\yb^0_i).
\end{equation}
It is easy to see that (for example, see the discussion in Section 2.2 of \cite{AOSko}), for $h_n\asymp \sfh$, the eigenvalues of $\widehat{\mathcal{K}}_n$ coincide with $n^{-1} \bK_n$, and for any eigenvalue $\lambda_i>0$ of $\widehat{\mathcal{K}}_n,$ the corresponding eigenfunction $\widehat{\phi}_i^{(n)}$ satisfies 
\begin{equation} \label{eigs.hat}
\widehat{\phi}^{(n)}_i(\xb)=\frac{1}{\lambda_i\sqrt{n}} \sum_{j=1}^n f\left(\frac{\| \xb-\yb^0_j \|_2}{h^{1/2}_n}\right) u_{i j}, \qquad \xb \in \mathbb{R}^{p},
\end{equation}
and $\|\widehat{\phi}^{(n)}_i\|_{\widehat{\sfP}_n}=1$,
where $u_{ij}$ is the $j$-th component of $\ub_i$, the $i$-th eigenvector of $\bK_n$. In particular, we have 
\beq
\widehat{\phi}^{(n)}_i(\yb^0_j)=\sqrt{n}u_{ij},\qquad \text{for all $j\in[n]$,}
\eeq
so that the eigenvector $\ub_i$ can be interpreted as the discretized $i$-th empirical eigenfunction $\widehat{\phi}_i^{(n)}$ evaluated at the noisy samples $\{\yb_i^0\}_{1\le i\le n}$.   Finally, for some constant $\delta>0,$ we define the integer 
\begin{equation}\label{eq_defni}
\mathsf{K} \equiv \mathsf{K}(\delta):=\arg \operatorname{max} \left\{ 1 \leq i \leq n: \gamma_i \geq 
\delta \right\}.
\end{equation}
To prove the convergence of eigenfunctions, we consider an RKHS $\mathcal{H}_K$ for functions defined on $\text{supp}(\widetilde\sfP)={ \mathcal{S}}$, associated with the kernel function $K(\cdot,\cdot):\mathcal{S} \times \mathcal{S} \to\R$ such that
\[
K(\xb,\yb)=f\bigg(\frac{\|\xb-\yb\|_2}{\sfh^{1/2}}\bigg).
\]
Specifically, we define $\mathcal{H}_K$ as the completion of the linear span of the set of functions $\{K_{\xb}=K(\xb,\cdot): \xb\in \mathcal{S} \}$, with the inner product denoted as $\langle \cdot,\cdot \rangle_K$ satisfying $\langle K(\xb, \cdot), K(\yb, \cdot)\rangle_K=K(\xb,\yb)$ and the reproducing property $\langle K(\xb, \cdot), g\rangle_K=g(\xb)$ for any $g\in \mathcal{H}_K$ (see \cite{berlinet2011reproducing,manton2015primer} for systematic treatments of RKHS).

\begin{thm}\label{cor_maincor}
	Under the assumptions of Theorem \ref{thm_mainthm}, the following results hold. \\
	\begin{enumerate} 
		\item(Eigenvalue convergence) We have
		\begin{equation} \label{v.conv}
		\max_{i\in[n]} |\lambda_i-\gamma_i|={ \OO_{\prec}\left( \frac{1}{\sqrt{n}}+\frac{1}{\sum_{i=1}^r n^{\alpha_i-\beta-\eta}}\right).}
		\end{equation}
		\item(Eigenfunction convergence) For $\mathsf{K}$ in (\ref{eq_defni}) and $ 1 \leq i \leq \mathsf{K}$, for any $\yb\in\mathcal{S}$, we have 
		\begin{equation} \label{f.conv}
		|\sqrt{\lambda_i}\widehat{\phi}_i^{(n)}(\yb)-\sqrt{\gamma_i}{\widetilde{\phi}_i}(\yb)|=\OO_{\prec}\left(  \psi_n+\frac{1}{\sfr_i}\bigg[ \psi_n^{1/2}+\frac{1}{\sqrt{n}}\bigg]\right),
		\end{equation} 
		where $\psi_n$ is defined in (\ref{psi}). 
	\end{enumerate}
\end{thm}

Theorem \ref{cor_maincor} establishes the spectral convergence of the kernel matrix $n^{-1}\bK_n$ to the population integral operator $\widetilde{\mathcal{K}}$ given by (\ref{K.tilde}) in $\mathcal{H}_K$. {In addition to the convergence rates obtained in Theorem \ref{thm_mainthm}, an additional $n^{-1/2}$ comes into play. Such a rate is standard in the RKHS theory when using $n^{-1} \bK_n^*$ to approximate the integral operator; see Section \ref{sec_sub_review} of our supplement \cite{suppl} for a summary.} Based on the above theorem, on the one hand, (\ref{v.conv}) demonstrates that the deterministic limits of the eigenvalues $\{\lambda_i\}_{i\in[n]}$ are those of the integral operator $\widetilde{\mathcal{K}}$. {The convergence rates consist of two parts. The $n^{-1/2}$ part is the rate when we use $n^{-1}\bK_n^*$ to approximate the integral operator and the { $(\sum_{i=1}^r n^{\alpha_i-\beta-\eta})^{-1}$} part is related to the overall signal-to-noise ratio and appears when we use $n^{-1} \bK_n$ to approximate $n^{-1} \bK_n^*.$} On the other hand, (\ref{f.conv}) indicates that the empirical eigenfunctions $\widehat{\phi}_i^{(n)}$, after proper rescaling, can also be regarded as an approximation of the population eigenfunctions $\widetilde\phi_i$ in a pointwise manner. {For the convergence rate in (\ref{f.conv}), the understanding is similar to that in (\ref{v.conv}) where the $\psi_n^{1/2}/\mathsf{r}_i$ is from (\ref{eq_projectionbound}).}
In particular, Theorem \ref{cor_maincor} implies that for any $1\le i\le \mathsf{K}$, 
\beq
{\lambda_i}\ub_i = {\frac{\lambda_i}{\sqrt{n}}}\cdot(\widetilde\phi_i^{(n)}(\yb^0_1), \widetilde\phi_i^{(n)}(\yb^0_2), ..., \widetilde\phi_i^{(n)}(\yb^0_n))^\top\approx \frac{\gamma_i}{\sqrt{n}}\cdot(\widetilde{\phi}_i(\yb^0_1), \widetilde{\phi}_i(\yb^0_2),...,\widetilde{\phi}_i(\yb^0_n))^\top, \nonumber
\eeq
so that as long as $\Omega\subseteq\{1,2,...,\mathsf{K}\}$, the final low-dimensional embedding $\bU_{\Omega}\bLam_{\Omega}$ is approximately the leading $\mathsf{K}$ eigenfunctions evaluated at the rotated data $\{\yb^0_i\}_{1\le i\le n}$, and weighted by their corresponding eigenvalues $\{\gamma_i\}$. In other words, each coordinate of the final embedding is essentially a nonlinear transform of the original data, by some functions uniquely determined by the population integral operator $\widetilde{\mathcal{K}}$. 
This explains why our proposed method may capture important geometric features of $\iota(\mathcal{M})$, and what kind of geometric features our method aims at.

As a comparison, existing methods such as Laplacian eigenmap, DM, or LLE, have their final embeddings corresponding asymptotically to the Laplace-Beltrami differential operator evaluated at low-dimensional noiseless data. The fundamental distinction between integral and differential operators indicates the potential advantage of the proposed method over these existing methods, especially for high-dimensional noisy datasets as indicated by our real data analysis (Section \ref{cell.order.sec} and Section \ref{suppl_additionalrealexamples} of \cite{suppl}).
In this respect, Theorem \ref{cor_maincor} provides a systematic way to understand the general geometric interpretation of the proposed low-dimensional embedding, although their concrete meanings may vary from case to case.  To better illustrate this, in Section \ref{sec_someresultintegral} of our supplement \citep{suppl} we consider a specific probability measure $\widetilde\sfP$, obtain the explicit forms of its corresponding eigenvalues and eigenfunctions, and demonstrate their geometric interpretations accordingly.

{
\subsection{Extension to General Kernel Functions} \label{ker.alt.sec}

In this subsection, we show that in addition to the Gaussian kernel function, 
our analysis can be extended to fit other kernel functions satisfying the following regularity conditions. 
 
	\begin{assu}\label{ker.assu}
	We assume the kernel function $f: \R_{\ge 0}\to \R_{\ge 0}$ satisfies
	\begin{itemize}
	\item[(i)] Positive-semidefiniteness: for any sequence of real values $\{c_i\}_{1\le i\le n}$, and real vectors $\{\xb_i\}_{1\le i\le n}\subset\R^p$, we have $$\sum_{i=1}^n\sum_{j=1}^nc_ic_jf(\|\xb_i-\xb_j\|_2)\ge 0.$$
	\item[(ii)] Boundedness: $\sup_{x\in \R_{\ge 0}}f(x)\le C<\infty$ for some absolute constant $C>0$.
	\item[(iii)] { H\"older-like continuity: there is a constant $L>0$ such that for all $x,y\in \R$, we have $|f(x)-f(y)|\le L\{|x|^{\nu(x)}+|y|^{\nu(y)}\}\cdot |x-y|^{\tau(x,y)}$ for some $0<\tau(x,y)\le 1$ and $\nu(x)\in\R$ such that $\nu(x)\ge 0$ for $0\le x\le 1$ and $\nu(x)\le 0$ for $x>1$.}
	\end{itemize}
\end{assu}

Conditions (i) and (ii) in Assumption \ref{ker.assu} are commonly used in  RKHS theory \citep{MR2558684,JMLR:v11:rosasco10a}  for the spectral convergence of the integral operator, whereas Condition (iii) concerns the local behavior of the kernel function, which is useful for characterizing the noise and bandwidth effects. In particular, in Condition (iii), { the parameters $\nu(x)$ for large $x$ and $\tau(x,y)$ characterize the decay rate of the kernel function, whereas $\nu(x)$ for small $x$ characterizes the flatness of the kernel function around 0.} Similar condition called Lipschitz-like continuity (or pseudo-Lipschitz continuity), that is $\tau(x,y)\equiv 1$, has been used in statistical learning theory, for example, see \cite{5695122}. 

 In addition to the Gaussian kernel function $f(x)=\exp(-x^2)$, {which satisfies Assumption \ref{ker.assu} with $\tau(x,y)=1$ and $\nu(x)=1\{0\le x\le 1\}-1\{x>1\}$}, there are many other kernel functions satisfying Assumption \ref{ker.assu}. For example, see Chapter 4 of \cite{williams2006gaussian}. For the readers' convenience, we provide a few examples as follows.  
\begin{itemize}
	\item Laplacian kernels: $f(x)=\exp(- x/\ell)$ for some constant $\ell>0$, { which satisfies Assumption \ref{ker.assu} with $\tau(x,y)=1$ and $\nu(x)=-1\{x>1\}$.} 
	\item Rational quadratic (polynomial) kernels: $f(x)=\big(1+\frac{x^2}{2\alpha\ell^2} \big)^{-\alpha}$ for some constants $\alpha,\ell>0$, { which satisfies Assumption \ref{ker.assu} with $\tau(x,y)=1$ and $\nu(x)=-(2\alpha+1)\cdot1\{x>1\}+1\{0\le x\le 1\}$.}
	\item Mat\'ern kernels: $f(x)=\frac{2^{1-\zeta}}{\Gamma(\zeta)}\big( \frac{\sqrt{2\zeta}x}{\ell}  \big)^\zeta B_{\zeta}\big(\frac{\sqrt{2\zeta}x}{\ell} \big)$ for some constants $\zeta,\ell>0$, where 
	$B_\zeta(x)$ is the modified Bessel function of the second kind of order $\zeta$. One interesting example is $
	f_{\zeta=3/2}(x) = \left( 1+\frac{\sqrt{3}x}{\ell}\right)\exp( -{\sqrt{3}x}/{\ell})$, { which satisfies Assumption \ref{ker.assu} with $\tau(x,y)=1$ and $\nu(x)=-1\{x>1\}$.}
		\item Truncated kernels: $f(x)=(1-x)_{+}^{\lfloor \alpha/2\rfloor+1}$, where $\alpha \geq 0$ is some constant and $(x)_{+}:=\max\{x,0\}$, { which satisfies Assumption \ref{ker.assu} with $\tau(x,y)=1$ and $\nu(x)=0$.}
\end{itemize}

We point out that in addition to the aforementioned kernel functions, according to the properties of positive-definite kernel functions \cite{williams2006gaussian,cuturi2009positive}, it is also straightforward to see that, for any finite numbers of kernel functions $f_1, ..., f_R$ satisfying Assumption \ref{ker.assu}, both their additive mixture $f(x)=\sum_{i=1}^R f_i(x)$ and its multiplicative mixture $f(x)=\prod_{i=1}^R f_i(x)$ also satisfy Assumption \ref{ker.assu}, suggesting the wide range of kernel functions covered by our theory. The following theorem shows the spectral convergence of $\bK_n$ under these general kernel functions. 
	
\begin{thm}\label{thm_mainthm2}
	Suppose Assumptions \ref{assum_mainassumption} and \ref{ker.assu} hold. { Denote $\nu_0=\inf_{0\le x\le 1}\nu(x)$, $\tau_0=\inf_{x,y\in\R}\tau(x,y)$, and define
	\begin{equation}\label{eq_xidefinition}
	\xi=\xi(\nu_0,\tau_0) = \left\{ \begin{array}{ll}
		\tau_0, & \textrm{ if $\nu_0\ge \tau_0$}\\
		{(\tau_0+\nu_0)}/{2}, & \textrm{ if $0\le \nu_0< \tau_0$}
	\end{array} \right. .
	\end{equation}}
	Then the following holds.
	\begin{enumerate}
		\item(Eigenvalue convergence) The eigenvalues $\{\lambda_i\}_{1\le i\le n}$ of $\bK_n$ satisfy 
			\begin{equation} \label{eq_eigs_v2}
		\max_{i\in [n]} |\lambda_i-\mu_i|=\OO_{\prec}\left( \psi_n^{\xi}\right),
		\end{equation}
        and
	\begin{equation} \label{v.conv2}
	\max_{i\in[n]} |\lambda_i-\gamma_i|=\OO_{\prec}\left(\psi_n^{\xi}+\frac{1}{\sqrt{n}}\right),
	\end{equation}
	where $\psi_n$ is defined in (\ref{psi}).
		\item(Eigenvector convergence) For any $i\in[n]$, if the $i$-th population eigengap $\mathsf{r}_i:=\min\{\gamma_{i-1}-\gamma_i, \gamma_i-\gamma_{i+1}\}$,
		satisfies 
		\begin{equation}\label{eq_generalribound}
		 \psi_n^{\xi}+\frac{1}{\sqrt{n}}=\oo(\sfr_i^2).
		\end{equation}
		Then we have that  
		\begin{equation}\label{eq_projectionbound2}
		| \langle \ub_i, \vb_i \rangle^2-1 |=\OO_{\prec}\left( \frac{\psi_n^{\xi}}{\sfr_i^2 }\right),
		\end{equation}
		where $\ub_i$ and $\vb_i$ are $i$-th eigenvector of $\bK_n$ and $\bK_n^*$, respectively. In addition, 
		 for $\mathsf{K}$ in (\ref{eq_defni}) and $ 1 \leq i \leq \mathsf{K}$, for any $\yb\in \mathcal{S}$, we have 
		\begin{equation} \label{f.conv2}
		|\sqrt{\lambda_i}\widehat{\phi}_i^{(n)}(\yb)-\sqrt{\gamma_i}{\widetilde{\phi}_i}(\yb)|=\OO_{\prec}\left(  \psi_n^{\xi}+\frac{1}{\sfr_i}\bigg[ \psi_n^{\xi/2}+\frac{1}{\sqrt{n}}\bigg]\right),
		\end{equation} 
		where $\widehat{\phi}_i^{(n)}$ and $\widetilde{\phi}_i$ are defined in (\ref{eigs.hat}) and (\ref{eigs.tilde}).
	\end{enumerate}
\end{thm}

Theorem \ref{thm_mainthm2} obtains similar rates of convergence as in our previous results under the Gaussian kernel, but also highlights  the  dependence of the convergence rates on the properties of specific kernel functions. For example, for the eigenvalue convergence in (\ref{eq_eigs_v2}), the rate depends on both $\psi_n$ and $\xi.$ $\psi_n$ appears because of the same reasoning as explained after Theorem \ref{thm_mainthm}, whereas the exponent $\xi$, which depends on specific kernel functions, appears when we approximate the entries of the kernel matrices. For illustrations on how the values of $\xi$ of different kernel functions affect the convergence rates, we refer the readers to Section \ref{sec_impactkernelpercentile} and Figures \ref{rate.fig} and \ref{ratio.fig} of our supplement \cite{suppl}.} {We point out that Gaussian kernel corresponds to {$\tau_0=\nu_0=1$ or $\xi=1$} in (\ref{eq_xidefinition}) so that Theorem \ref{thm_mainthm2} recovers the results of Theorems \ref{thm_mainthm} and \ref{cor_maincor}. In fact, technically, the generalization to the class of kernel functions considered in Assumption \ref{ker.assu} is highly motivated by the analysis of the Gaussian kernel function. We refer the readers to Section \ref{sec_novelty} below for more elaborations. }

{
\subsection{Proof Strategy and Novelty}\label{sec_novelty} In this subsection, we briefly describe our proof strategy and highlight the novelty. We start by reviewing the assumptions, results and strategies in a most related work \cite{DW2} which extensively generalizes the results of \cite{el2010spectrum,el2010information}, and then point out the significant differences from the current paper.     

In \cite{DW2}, the authors studied the eigenvalue convergence of $n^{-1} \bK_n$ to $n^{-1} \bK_n^*$, constructed using the Gaussian kernel function, under a mathematically simple setting $r=1.$ That is to say, the nonlinear manifold is 1-dimensional and the reduced model is a single-spiked covariance matrix model. Moreover, they fixed {$\beta=0$ and $\eta=1$.} In such a setting, our condition (\ref{eq_sigmaimagnititude}) reads as $\alpha \equiv \alpha_1>1.$ Under these conditions, they showed that 
\begin{equation}\label{eq_oldresult}
\left\| n^{-1} \bK_n-n^{-1} \bK_n^* \right\| \prec n^{-1/2}+n^{-\alpha+1}. 
\end{equation}
For the technical proof, similar to the ideas of \cite{el2010spectrum,el2010information}, to control the difference between $n^{-1} \bK_n$ and $n^{-1} \bK_n^*,$ \cite{DW2} directly applied entrywise Taylor expansion to $n^{-1} \bK_n$ and $n^{-1} \bK_n^*$ utilizing the smoothness and decay property of the Gaussian kernel function. To establish the convergence rates, \cite{DW2} proved some concentration inequalities for the pairwise distances of $\{\yb_i\}, \{\xb_i\}$ and $\{\zb_i\}$ and then studied the convergence rates and limits for the spectrum of kernel random matrices. We point out that technical proof of \cite{DW2} on the one hand depends crucially on the assumption that $r=1$ even though a generalization to some fixed integer is possible; on the other hand it relies on the properties of Gaussian kernel function.   

In our current paper, we consider both Gaussian and a general class of kernel functions. Our results apply to the general nonlinear manifold model as in Section \ref{man.sec}. Moreover, we no longer impose assumptions on $r$ (or the dimension of the submanifold $\mathcal{M}$ in Assumption \ref{assum_generalmodelassumption}). Instead, we set $r$ to be a generic parameter, allowed to diverge with $n.$  Unlike \cite{DW2}, our results in Theorems \ref{thm_mainthm} and \ref{thm_mainthm2} demonstrate that $r$ itself plays a role in the convergence rate especially when $r$ diverges with $n.$ Especially, unlike \cite{DW2}, which requires all signals to be strong, we only require their average to be strong. In particular, if $r$ is bounded, we only need one signal { $\alpha_i>\beta+\eta$} for some $1 \leq i \leq r.$ Furthermore, motivated by our applications, in addition to the eigenvalue convergence, we also studied the eigenvector convergence in the current paper and connected the noisy kernel matrices with an integral operator of some RKHS.

We now discuss the proof strategies of the main results. We start with the Gaussian kernel (Section \ref{ker.sec}). For the eigenvalue convergence of $n^{-1} \bK_n$ and $n^{-1} \bK_n^*$, our proof consists of two novel ideas and results. First, we used an approach different from the direct Taylor expansion approach as in \cite{DW2,el2010spectrum,el2010information}. More specifically, instead of comparing the matrices $\bK_n$ and $\bK_n^*$ directly by
expanding the entries of $\bK_n$ around $\bK_n^*$ using  Taylor expansion, we introduce an auxiliary matrix $\bK_s$ (cf. equation (\ref{eq_wsdefn}) of our supplement \cite{suppl}). Even though it is different from $\bK_n^*,$ it only depends on the kernel function, the bandwidth and $\bK_n^*.$ Consequently, we have $\bK_n-\bK_n^*=(\bK_n-\bK_s)+(\bK_s-\bK_n^*).$ The second term can be controlled by {$(\sum_{i=1}^r n^{\alpha_i-\beta-\eta})^{-1}$} (see equation (\ref{C19}) of our supplement \cite{suppl}) and the first term can be bounded by $\psi_n$ (see equations (\ref{conv.Ks}) and (\ref{partA}) of our supplement \cite{suppl}), leading to the result of (\ref{eq_matrixclose1}). We emphasize that by introducing $\bK_s,$ we are able to obtain the sharper rate $\psi_n.$ For example, as we mentioned earlier, in the setting of \cite{DW2}, $\psi_n=n^{-\alpha/2}+n^{-\alpha+1}$ improves the result of (\ref{eq_oldresult}). Second, on the technical level, the control of $\bK_n-\bK_s=\bK_n-\bK_y+\bK_y-\bK_s,$ where $\bK_y$ is the kernel matrix constructed using $\{\yb_i\}$ and the bandwidth $\sfh,$ relies crucially on Proposition \ref{lem_bandwidthconcentration}, which is a novel concentration inequality for our proposed bandwidth. Our proof of Proposition \ref{lem_bandwidthconcentration}, especially the lower bound, relies on two ingredients. On the one hand, we prove in Lemma \ref{event.lem} that for each fixed $i,$ with high probability,  $\|\xb_i-\xb_j \|_2^2 \ll \sum_{i=1}^r n^{\alpha_i}$ only happens for very few of $j's, 1 \leq j \leq n.$ On the other hand, we show that the order statistics of $\{\| \yb_i-\yb_j \|_2^2\}$ are close to those of  $\{\| \xb_i-\xb_j \|_2^2\}.$ The proof of Proposition \ref{lem_bandwidthconcentration} and Lemma \ref{event.lem} are  nontrivial and  of interests in themselves.    

With the eigenvalue convergence, we proceed to prove the eigenvector convergence. Instead of applying the Davis-Kahan theorem (which requires the eigen-gap to be bounded from below), we utilized the integral representation of the eigenvectors via the resolvent. That is 
\begin{equation*}
\langle \ub_i, \vb_i \rangle^2=\frac{1}{2 \pi \ri} \oint_{\Gamma_i} \vb_i^\top (z-n^{-1} \bK_n)^{-1} \vb_i \dd z,
\end{equation*}    
where $\Gamma_i$ is some properly chosen contour. To obtain the convergent limit, we decompose the above representation into three parts $\mathsf{L}_i, 1 \leq i \leq 3,$ as in (\ref{eq_uvdifference}) of our supplement \cite{suppl}.  The leading order part is $\mathsf{L}_1$ which can be calculated using Cauchy's residual theorem.  For the convergence rates, we apply resolvent expansions to  $\mathsf{L}_2$ and $\mathsf{L}_3.$ This proves (\ref{eq_projectionbound}). Once Theorem \ref{thm_mainthm} is proved, together with the RKHS theory as summarized in Section \ref{sec_sub_review} of our supplement \cite{suppl}, we can prove the convergence to the population integral operator as in Section \ref{op.sec}. We emphasize that the results in Theorem \ref{cor_maincor} {are} insightful. For example, the error rate in (\ref{v.conv}) contains two parts. The first part commonly appears in RKHS theory when using a kernel random matrix to approximate an integral operator. The second term can be understood as the signal-to-noise ratio. This also shows the power of our bandwidth selection scheme.

One advantage of our proof strategy is that it can be easily modified to fit other kernel functions satisfying Assumption \ref{ker.assu}. By modifying the proof of Gaussian kernel functions using (iii) of Assumption \ref{ker.assu}, we will be able to extend the results to a general class of kernel functions.  We emphasize that our theoretical results and numerical simulations demonstrate that even though our proposed bandwidth selection scheme and its property (cf. Proposition \ref{lem_bandwidthconcentration}) are independent of kernel functions, the convergence rates depend on the geometric properties of specific kernel functions crucially.     

}

\section{Numerical Studies and Real Data Analysis} \label{data.sec}
In this section, we provide numerical simulations and a real data example to illustrate the usefulness of our proposed method. More real data examples can be found in Section \ref{suppl_additionalrealexamples} of our supplement \cite{suppl}. 



\subsection{Simulation Studies} \label{simu.sec}
\subsubsection{Impact of  kernel functions and percentiles}\label{sec_impactkernelpercentile}
{ First, we conduct numerical simulations to show that our adaptive bandwidth selection scheme (\ref{eq_bandwidthselection}) is robust with respect to various underlying structures, kernel functions and the choices of percentiles.}  For given $n$, we set $p=\lfloor n/5\rfloor$, and generate $\zb_i\sim_{i.i.d.} \mathcal{N}({\bf 0}, {\bf I}_p)$. For the noiseless samples $\{\xb_i\}_{1\le i\le n}$, we set $\xb_i=(\bm{x}_i, 0, ..., 0)\in\R^p$ for some $\bm{x}_i\in\R^r$, and generate $\{\bm{x}_i\}_{1\le i\le n}$ from  one  of the following settings (see Figure \ref{man.fig} of our supplement \cite{suppl} for illustrations):
\begin{itemize}
	\item  { "Smiley face" with $r=2$: set $\bm{x}_i=n^{2/3}\bm{x}'_i$ and generate independent samples $\{\bm{x}'_i\}_{1\le i\le n}$ from a ``smiley face" $S=S_1\cup S_2\cup S_3\cup S_4\subset\R^2$, where $S_1=\{(x,y)\in\R^2: (x-0.5)^2+(y-0.5)^2\le 0.1\}$,
		$S_2=\{(x,y)\in\R^2: (x+0.5)^2+(y-0.5)^2\le 0.1\}$,
		$S_3=\{(x,y)\in\R^2: 1.8\le x^2+y^2\le 2\}$, and
		$S_4=\{(x,y)\in\R^2: 0.9\le x^2+y^2\le 1.1, y\le 0\}.$ }
	\item "Mammoth"  with $r=3$:  set $\bm{x}_i=n^{2/3}\bm{x}'_i$ and generate samples $\{\bm{x}'_i\}_{1\le i\le n}$ uniformly from a  "mammoth" manifold \citep{wang2021understanding,smithsonian} embedded in $\R^3$.
	\item { "Cassini oval" with $r=3$: set $\bm{x}_i=n^{2/3}\bm{x}'_i$ and generate samples $\{\bm{x}'_i\}_{1\le i\le n}$ uniformly from a Cassini oval in $\R^3$, defined by $x_1(t)=\sqrt{\cos(2t)+\sqrt{\cos(2t)^2+0.2}}\cos t$,		$x_2(t)=\sqrt{\cos(2t)+\sqrt{\cos(2t)^2+0.2}}\sin t$, and
    $x_3(t)=0.3\sin(t+\pi)$, where $t\in[0,2\pi).$ }
	\item { "Torus" with $r=3$: set $\bm{x}_i=n^{2/3}\bm{x}'_i$ and generate samples $\{\bm{x}'_i\}_{1\le i\le n}$ uniformly from a torus in $\R^3$, defined by $x_1(u,v)=(2+0.8\cos u)\cos v$, $x_2(u,v)=(2+0.8\cos u)\sin v$, and $x_3(u,v)=0.8 \sin u$, where $u,v\in [0,2\pi)$. }
\end{itemize}
 { Note that the global scaling factor $n^{2/3}$ quantifies the strength of the signals. In our simulation setting, for the rate $\psi_n$ in (\ref{psi}), we have that $\psi_n=n^{-1/3}.$ For the above four different simulation settings, we evaluate the proposed adaptive bandwidth selection scheme under various choices of kernel functions and percentiles $\omega$. Specifically, we consider a Gaussian kernel $f(x)=\exp(-x^2)$,  a Laplacian kernel $f(x)=\exp(-x)$, and a rational quadratic (polynomial) kernel $f(x)=(1+x^2/4)^{-2}$, and set $\omega\in\{0.05,0.25,0.5,0.75,0.95\}$. 
 
  In Figure \ref{rate.fig} of our supplement \cite{suppl}, we show that the spectral error $\|n^{-1}\bK_n-n^{-1}\bK_n^*\|$ between the sample kernel matrix $\bK_n$ and the underlying noiseless kernel matrix $\bK_n^*$ as $n$ increases from $500$ to $4000$, demonstrating the spectral convergence of the two matrices as $n\to\infty$ under different choices of kernels and percentile parameter $\omega$. Overall, for a given underlying structure the variation in the performance of the proposed method appeared to be small across different kernel functions and over a wide range of $\omega$, except that in some cases the rate of convergence was slower for very small $\omega$. In Section \ref{choiceomegereampls} of our supplement \citep{suppl}, we show that $\omega$ selected by a data-driven resampling approach can achieve similar convergence behavior in these examples.  Interestingly, we observed that the rates of convergence under the Gaussian kernel matrix were in general faster than that under the Laplacian kernel, and were similar to that under the polynomial kernel as reported in Figure \ref{ratio.fig} of \cite{suppl}; this phenomenon was captured by our theory in Theorem \ref{thm_mainthm2}, highlighting the role of the H\"older continuity property of the kernel function in the convergence rate.

Moreover, we use numerical simulations to demonstrate that the noiseless kernel matrix $\bK_n^*$ based on the bandwidth $\sfh$ defined in (\ref{sfh}) is potentially more desirable than other bandwidths that are either much larger or smaller than $\sfh$ as discussed below Proposition \ref{lem_bandwidthconcentration}. In particular, the fixed bandwidth $\sfh\equiv p$ considered in many existing works \cite{cheng2013spectrum,9205615,do2013spectrum,el2010information,el2010spectrum,ElKaroui_Wu:2016b} lies in the latter case whenever the signal is sufficiently strong, or $\sum_{i=1}^r\theta_i\gg p$. Thus, for each of the above manifold structures, we set $n=100$ and generate noiseless samples to construct noiseless Gaussian kernel matrices using (i) our proposed bandwidth $\sfh$ with $\omega=0.5$, (ii) a relatively large bandwidth $\sfh=n^5$, and (iii) a relatively small bandwidth $\sfh=p$. In Figure \ref{bw.fig1} of \cite{suppl}, we show the eigenvalues of these kernel matrices. We find that the kernel matrices based on our proposed bandwidth $\sfh$ are in general more informative compared to others based on $\sfh=p$ or $\sfh=n^5$: in the former case, the kernel matrices essentially degenerated into an identity matrix, whereas in the latter case they degenerated into a rank-one constant matrix since $f(0)=1$ for Gaussian kernel. 

}

\subsubsection{Comparison with other spectral clustering methods} \label{cluster.sec}

{
Another important application of spectral embedding is in the clustering of high-dimensional data. To demonstrate the advantage of the proposed method, we generate both linear and nonlinear clustered data and compare our method with some existing ones.} We set $p=n=300$ and generate $\zb_i\sim_{i.i.d.} \mathcal{N}(0,{\bf I}_p)$. For the signal samples $\{\xb_i\}_{1\le i\le n}$, we set $\xb_i=(\bm{x}_i,0,...,0)\in\R^p$ for some $\bm{x}_i\in\R^r$ generated from one of the following settings:
\begin{itemize}
	\item { Gaussian mixture model with $r=6$: for some global scaling parameter $\theta>0$, we set $\bm{x}_i=\theta\bm{x}'_i$ and generate $i.i.d.$ samples $\{\bm{x}'_i\}_{1\le i\le n}$ from a Gaussian mixture model in $\R^6$ with six clusters of equal class proportion. The six clusters have their mean vector coincide with the Euclidean basis in $\R^6$.} 
	\item Nested spheres with $r=6$: for some global scaling parameter $\theta>0$, we set $\bm{x}_i=\theta\bm{x}'_i$ and generate $i.i.d.$ samples $\{\bm{x}'_i\}_{1\le i\le n}$ from a mixture model of nested spheres $\sum_{i=1}^6w_iP_i$, where the cluster proportion $w_i=1/6$, and $\{P_i\}_{1\le i\le 6}$ are uniform distributions on nested spheres in $\R^6$ of radii $\{1,2,3,4,5,6\}$ respectively.
\end{itemize}
{ Under each setting, we evaluate the clustering performance of k-means applied to the low-dimensional embeddings based on: (1). the proposed algorithm ("prop"); (2). the proposed algorithm with fixed bandwidth $h_n=p$ ("h=p"); (3).  the proposed algorithm with the bandwidth selected according to \cite{AOSko} ("SBY"); (4). the linear spectral embedding of \citep{loffler2021optimality} ("LZZ"); (5). the "hollowed" linear spectral embedding of \cite{abbe2020ell_p} ("AFW");\footnote{{ Here we only compared with the linear (Euclidean) spectral embedding method of \cite{abbe2020ell_p} because (i) for given kernel matrices, the proposal of \cite{abbe2020ell_p} is equivalent to Step 3 of our algorithm, and (ii) the main advantage of our method lies in the construction of kernel matrices and bandwidth selection.}} (6). a graph-cut-based spectral clustering method "spuds" \citep{hofmeyr2019improving}.

\begin{figure}[ht]
	\centering
	\includegraphics[angle=0,width=14cm,height=7cm]{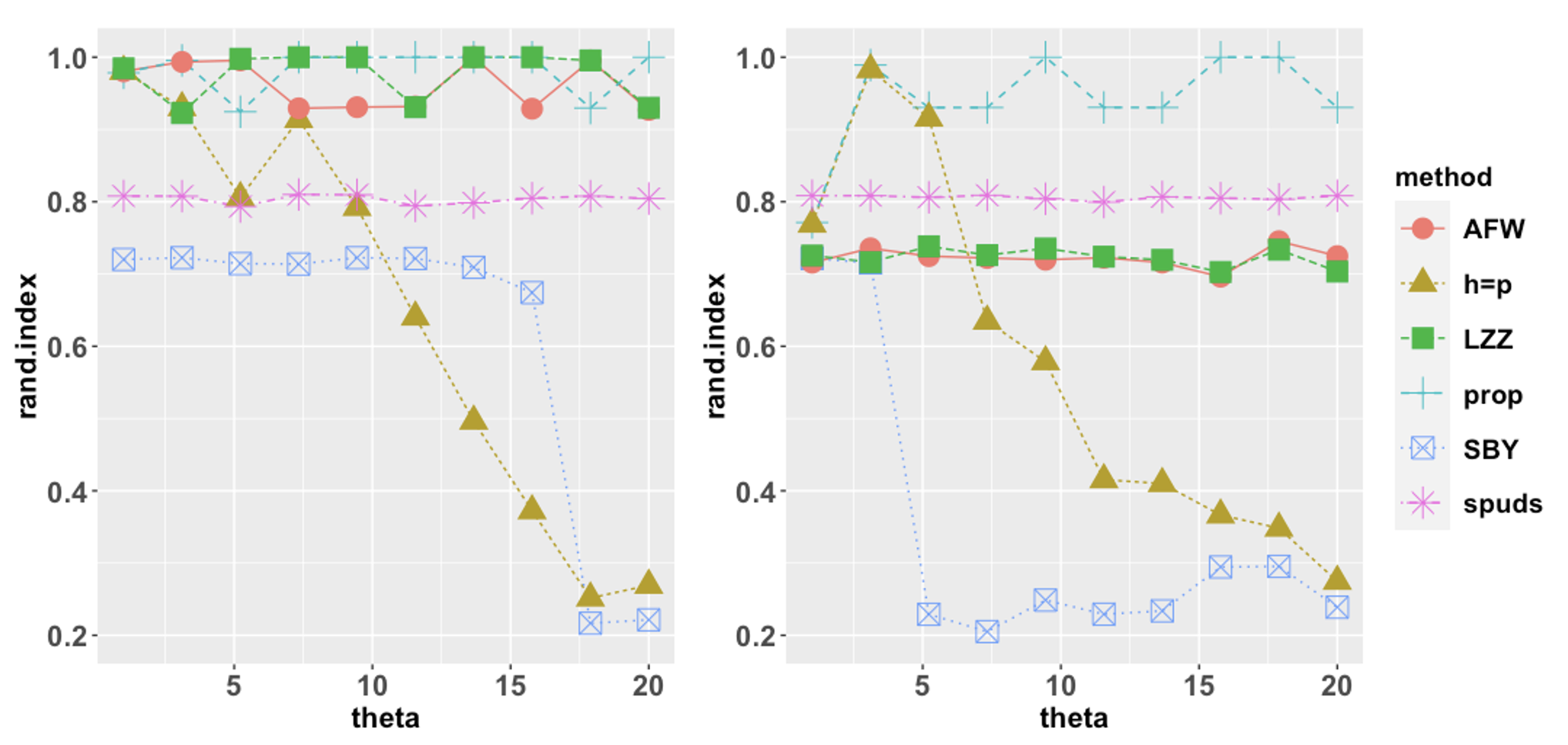}
	    \vspace*{-4mm}
	\caption{Comparison of spectral clustering methods. Left: Gaussian mixture model; Right: nested sphere model. "prop": the proposed algorithm; "h=p": proposed algorithm with fixed bandwidth $h_n=p$; "SBY": proposed algorithm with bandwidth selected according to \cite{AOSko}; "LZZ": linear spectral embedding method \citep{loffler2021optimality}; "AFW": "hollowed" linear spectral embedding of \citep{abbe2020ell_p}; "spuds": graph-cut based spectral clustering method  \citep{hofmeyr2019improving}.} 
	\label{cls.fig1}
\end{figure}

Figure \ref{cls.fig1} contains the Rand index for each of the above methods under various settings. A higher Rand index indicates better clustering performance. In the simulations, we used the Gaussian kernel function. In general, under the Gaussian mixture model, "AFW," "LZZ" and our method "prop" has a better performance compared to the other methods, whereas under the nested sphere model, our method "prop" is superior to all the other methods. Our results demonstrate the advantage of our proposed method, including our bandwidth $h_n$, in dealing with noisy nonlinear structures. 

 }

\subsection{Applications to Real-World Datasets} \label{cell.order.sec}


Our  example concerns pseudotemporal ordering of single cells. As an important problem in single-cell biology, pseudotemporal cell ordering aims to determine the pattern of a dynamic process experienced by cells and then arrange cells according to their progression through the process, based on single-cell RNA-Seq data collected at multiple time points. In the following, we show that our proposed embedding method is able to capture such an underlying progression path as a one-dimensional  manifold and then help determine a more precise ordering of cells, compared to the existing state-of-art methods. The dataset consists of single-cell RNA sequencing  reads for 149 human primordial germ cells ranging from 4 weeks  to 19 weeks old \citep{guo2015transcriptome}. Specifically, there are 6 cells of 4 weeks old, 37 cells of 7 weeks old, 20 cells of 10 weeks old, 27 cells of 11 weeks old, and 57 cells of 19 weeks old; these are treated as the ground truth for cell ordering. The raw count data were preprocessed, normalized, and scaled following the standard procedure (R functions \texttt{CreateSeuratObject}, \texttt{NormalizeData} and \texttt{ScaleData} under default settings) as incorporated in the R package \texttt{Seurat}.
We also applied the R function \texttt{FindVariableFeatures} in \texttt{Seurat} to identify  $p\in\{2500, 3000, 3500, 4000\}$ most variable genes for subsequent analysis. The final dataset consists of standardized expression levels of $p$ genes for the 149 cells. We apply our proposed method with  a variety of $\omega\in\{0.25,0.5,0.75\}$ and  consider a two-dimensional embedding $\Omega=\{1,2\}$ using Gaussian kernel function, where the one-dimensional manifold is already clearly captured; see left of Figure \ref{embed.fig3} of \cite{suppl} for an illustration.

Observing that the path manifold is aligned with the second eigenvector, we order the cells according to their values in the second eigenvector. Inferred cell orders are also obtained using state-of-art cell ordering algorithms such as TSCAN \citep{ji2016tscan}, SCORPIUS \citep{cannoodt2016scorpius}, SerialRank \citep{fogel2014serialrank}, and the  classical PCA. Note that SerialRank essentially has Laplacian eigenmap as its core. For TSCAN and SCORPIUS, we applied functions from their R packages\footnote{https://github.com/zji90/TSCAN, and https://cran.rstudio.com/web/packages/SCORPIUS/index.html} under the default settings. To evaluate the quality of the inferred orders, we compare them with the ground truth by evaluating their respective Kendall's tau statistic.
On the left of Figure \ref{embed.fig22} of \cite{suppl}, the temporal ordering based on the proposed method ($\omega=0.5$) aligns better with the actual cell order than the other methods, with almost three-fold improvement over the second best method TSCAN in Kendall's tau on average. Moreover, the superiority of the proposed method is not sensitive to the dimension $p$, or the percentile parameter $\omega$. 

Finally, in Figure \ref{embed.fig3} of \cite{suppl} we show the scatter plots of the 2-dimensional embeddings based on PCA and the proposed method, denoted as KEF standing for kernel eigenfunctions, when $p=3000$. The advantage of the proposed embedding over its linear counterpart, such as informativeness and robustness to outliers, is visible and significant.

{
\section*{Acknowledgment}
The authors would like to thank the editor, the associate editor and two anonymous reviewers for their suggestions and comments, which have resulted in a significant improvement of the manuscript. }

\vspace*{4pt}

{\center{\bf  \uppercase{Supplement to "Learning low-dimensional nonlinear structures from high-dimensional noisy data: an integral operator approach"}}}

\appendix

\renewcommand{\thesection}{\Alph{section}}
\counterwithin{table}{section}
\counterwithin{figure}{section}

\section{Preliminary Results}\label{sec_pre}

In this section, we summarize and prove some useful technical results, which play an important role in our proof of the main results.

\subsection{Spectral Convergence of Integral Operators in RKHS}\label{sec_sub_review}

We summarize and prove the spectral convergence of the empirical integral operator to the population integral operator. Most of the results can be found in \cite{JMLR:v7:braun06a, BJKO, JMLR:v11:rosasco10a, ICMLko, AOSko,Smale2007LearningTE, MR2558684}.
Consider that we observe $n$ i.i.d. samples $\{\xb_i\}_{1\le i\le n}$ drawn from some probability distribution $\sfP$ in $\mathbb{R}^{p}$. Then the population integral operator $\mathcal{K}$ with respect to  $\sfP$ and the reproducing kernel
\beq \label{eq_kernel}
K(\xb,\yb):=f\left(\frac{\|\xb-\yb \|_2}{h^{1/2}}\right),\qquad \xb,\yb \in \text{supp}(\sfP),
\eeq
for some $f:\R_{\ge 0}\to \R_{\ge 0}$ satisfying {Assumption \ref{ker.assu}}
is defined by
\begin{equation}  \label{K}
{\mathcal{K}} g(\xb)=\int f\left(\frac{\|\xb-\yb \|_2}{h^{1/2}}\right)g(\yb) {\sfP}(\dd \yb) , \qquad \xb, \yb\in \text{supp}(\sfP),
\end{equation}
and its empirical counterpart $\mathcal{K}_n$ is defined by
\begin{equation} \label{K_n}
{\mathcal{K}_n} g(\xb)=\int f\left(\frac{\|\xb-\yb \|_2}{h^{1/2}}\right) g( \yb) {\sfP}_n(\dd \yb)=\frac{1}{n}\sum_{i=1}^n f\left(\frac{\|\xb-\xb_i\|_2}{h^{1/2}}\right) g(\xb_i), \qquad \xb \in \text{supp}(\sfP).
\end{equation}
The reproducing kernel Hilbert space (RKHS) $\mathcal{H}_K$ associated with the kernel function $K(\xb,\yb):=f\left(\frac{\|\xb-\yb \|_2}{h^{1/2}}\right)$	
in (\ref{K}) and (\ref{K_n}) is defined as the completion of the linear span of the set of functions $\{K_{\xb}=K(\xb,\cdot): \xb \in \text{supp}(\sfP)\}$ with the inner product denoted as $\langle \cdot,\cdot \rangle_K$ satisfying $\langle K(\xb, \cdot), K(\yb, \cdot)\rangle_K=K(\xb,\yb)$ and the reproducing property $\langle K(\xb, \cdot), g\rangle_K=g(\xb)$ for any $g\in \mathcal{H}_K$.
Note that $\mathcal{K}$ and $\mathcal{K}_n$ may be considered as self-adjoint operators on $\mathcal{H}_K$, or on their respective $L_2$ spaces (that is, $\mathcal{L}_2(\Omega,\sfP)$ and $\mathcal{L}_2(\Omega,\sfP_n)$, $\Omega=\text{supp}(\sfP)$). 

In addition, it is easy to see that (for example, Section 2.2 of \cite{AOSko}) the eigenvalues of $\mathcal{K}_n$ coincide with $n^{-1} \bK^*_n$, where $\bK_n^*=(K(\xb_i,\xb_j))_{1\le i,j\le n}$ with $K(\cdot,\cdot)$ given by (\ref{eq_kernel}). The eigenfunctions $\{\phi_i^{(n)}\}$ associated with nonzero eigenvalues $\{\mu_i\}$ of $\mathcal{K}_n$ or $n^{-1}\bK_n^*$ satisfy
\beq \label{eigenvec}
{\phi}^{(n)}_i(\xb)=\frac{1}{\mu_i\sqrt{n}} \sum_{j=1}^n f\left(\frac{\| \xb-\xb_j \|_2}{h^{1/2}}\right) v_{i j}, \qquad \xb \in \text{supp}(\sfP),
\eeq
where $\vb_i=(v_{i1},v_{i2},...,v_{in})^\top$ is the $i$-th eigenvector of $n^{-1} \bK_n^*$, and that  $\|\phi_i^{(n)}\|_{\sfP_n}=1$.

In the following we will mainly consider them to be operators in $\mathcal{H}_K$ to facilitate quantitative comparison.
The following lemma concerns the convergence of the eigenvalues and eigenfunctions of $\mathcal{K}_n$ to those of $\mathcal{K}$.

\begin{lem}\label{lem_mutiplepertubation} For self-adjoint operators $\mathcal{K}$ and $\mathcal{K}_n$ on $\mathcal{H}_K$, defined by (\ref{K}) and (\ref{K_n}), under Assumption \ref{ker.assu}, we have 
	\begin{equation}\label{eq_evlimitclose}
	\| \mathcal{K}-\mathcal{K}_n \| \prec \frac{1}{\sqrt{n}}.
	\end{equation}
	Let $\{\gamma_i\}$ and $\{\phi_i\}$ be the eigenvalues and eigenfunctions of $\mathcal{K}$. We define the $i$-th population eigen-gap as $\mathsf{r}_i=\min\{\gamma_{i-1}-\gamma_i, \gamma_i-\gamma_{i+1}\}$. If $n^{-1/2}=o(\sfr_i)$,  then the eigenfunctions  $\{\phi_i^{(n)}\}$ of $\mathcal{K}_n$ associated with nonzero eigenvalues satisfy
	\begin{equation} \label{eq_eigenfunction_conv}
	\left\| \sqrt{\gamma_i}{\phi_i}-\sqrt{\mu_i}\phi_i^{(n)} \right\|_{K} \prec \frac{1}{ \mathsf{r}_i \sqrt{n}},
	\end{equation}
	where $\| \sqrt{\gamma_i}{\phi_i}\|_K=\|\sqrt{\mu_i}\phi_i^{(n)} \|_K=1$, and $\|\phi_i^{(n)} \|_{\sfP_n}=1$.
\end{lem}
\begin{proof}
	To obtain (\ref{eq_evlimitclose}), according to \cite[Proposition 1]{MR2558684} or \cite[Theorem 7]{JMLR:v11:rosasco10a}, in view of our notion of stochastic domination,  we have that 
	\begin{equation*}
	\| \mathcal{K}-\mathcal{K}_n \|_{\HS} \prec \frac{1}{\sqrt{n}},
	\end{equation*} 
	where we used Assumption \ref{ker.assu}(ii) that the kernel function $K$
	in (\ref{K}) and (\ref{K_n}) is bounded. Then the result follows from the elementary relation $\| \mathcal{K}-\mathcal{K}_n \| \leq \| \mathcal{K}-\mathcal{K}_n \|_{\HS}.$
	
	For the eigenfunctions, the proof is similar to \cite[Corollary 1]{MR2558684} or \cite[Theorem 12]{JMLR:v11:rosasco10a} using the perturbation argument in Lemma \ref{lem_pertubationargument}. Specifically, from \cite[Corollary 1]{MR2558684}, we have 
	\beq \label{smale-zhou}
	\|\sqrt{\gamma_i}{\phi_i}-\Phi_i\|_{K}\prec \frac{1}{ \mathsf{r}_i \sqrt{n}},
	\eeq
	where $$\Phi_i(\xb)=\frac{1}{\sqrt{n\mu_i}}\sum_{j=1}^nf\left(\frac{\| \xb-\xb_j \|_2}{h^{1/2}}\right) v_{i j},$$
	$\vb_i=(v_{i1},...,v_{in})^\top$ and $\mu_i$ are $i$-th eigenvector and eigenvalue of $\frac{1}{n}\bK^*_n=n^{-1}(K(\xb_i,\xb_j))_{1\le i,j\le n}$, respectively. By definition, we have
	\[
	\phi_i^{(n)}=\frac{1}{\sqrt{n}\mu_i}\sum_{j=1}^nK(\xb_i,\cdot)v_{ij}=\frac{1}{\sqrt{\mu_i}}\Phi_i,
	\]
	this implies (\ref{eq_eigenfunction_conv}). Finally, to verify the normalizing factors in (\ref{eq_eigenfunction_conv}),
	direct calculation yields
	\begin{align*}
	\|\sqrt{\gamma_i}\phi_i\|^2_K & =\gamma_i\langle \phi_i,\phi_i\rangle_K = \langle \mathcal{K}\phi_i,\phi_i\rangle_K = \bigg\langle \int K(\cdot,\yb)\phi_i(\yb) {\sfP}(\dd \yb) ,\phi_i\bigg\rangle_K \\
	&=\int \langle K(\cdot,\yb) ,\phi_i\rangle_K \phi_i(\yb) {\sfP}(\dd \yb) \\
	& =\int[\phi_i(\yb)]^2{\sfP}(\dd \yb)=\|\phi_i\|_{\sfP}^2=1,
	\end{align*}
	and
	\begin{align*}
	&\|\Phi_i\|^2_K=\bigg\langle \frac{1}{\sqrt{n\mu_i}}\sum_{j=1}^nK(\xb_i, \cdot)v_{i j},\frac{1}{\sqrt{n\mu_i}}\sum_{j=1}^nK(\xb_i, \cdot)v_{i j}\bigg\rangle_K\\
	& =\frac{1}{{n\mu_i}}\sum_{j=1}^n\sum_{k=1}^nv_{i j}v_{ik}\langle K(\xb_j, \cdot),K(\xb_k, \cdot)\rangle_K =\frac{1}{{n\mu_i}}\sum_{j=1}^n\sum_{k=1}^nv_{i j}v_{ik}K(\xb_j,\xb_k) \\
	&=\frac{1}{\mu_i} \vb_i^\top (n^{-1}\bK_n^*)\vb_i=1,
	\end{align*}
	and 
	\begin{align*}
	&\|\phi_i^{(n)}\|_{\sfP_n}^2=\bigg\langle \frac{1}{{\sqrt{n}\mu_i}}\sum_{j=1}^nK(\xb_j, \cdot)v_{i j},\frac{1}{{\sqrt{n}\mu_i}}\sum_{j=1}^nK(\xb_j, \cdot)v_{i j}\bigg\rangle_{P_n}\\
	&=\frac{1}{n\mu_i^2}\sum_{j=1}^n\sum_{k=1}^n v_{ij}v_{ik}\int K(\xb_j, \yb)K(\xb_k, \yb)\sfP_n(\dd\yb)\\
	&=\frac{1}{n^2\mu_i^2}\sum_{j=1}^n\sum_{l=1}^n\sum_{k=1}^n v_{ij}v_{ik}K(\xb_j, \xb_l)K(\xb_l, \xb_k)\\
	&=\frac{1}{\mu_i^2}\vb_i^\top(n^{-1}\bK_n^*)^2\vb_i=1.
	\end{align*}
	
\end{proof}

{
\subsection{Brief Summary of Riemannian Manifold and Embedding Theorems}\label{supple_sec_Riemmanian} We provide a succinct overview of definitions, important results and the embedding theorems in the theory of smooth manifolds and Riemannian geometry. For a complete introduction, we refer the readers to the monographs \cite{boothby2003introduction, lee2013smooth}.

We start by introducing some important definitions. A \emph{topological $m$-manifold} is a second countable Hausdorff  topological space and locally Euclidean of dimension $m$. The topological manifold can be characterized by a collection of \emph{charts}, commonly called \emph{atlas}. In order to do calculus on the manifold, we need to provide a smooth structure which is usually known as a \emph{complete atlas}. A \emph{smooth $m$-manifold} is a topological $m$-manifold with a smooth structure. The calculus of a smooth manifold is conducted on its \emph{tangent space} $T_{\mathfrak{p}} \mathcal{M}$, $\mathfrak{p} \in \mathcal{M}.$ A \emph{Riemannian metric} $\mathrm{g}$ on the smooth $m$-manifold is a symmetric and positive definite smooth tensor field. Roughly speaking, $\mathrm{g}=(\mathrm{g}|_{\mathfrak{p}}, \ \mathfrak{p} \in \mathcal{M})$ is a collection of inner products (bilinear forms) defined on $ T_{\mathfrak{p}} \mathcal{M} \times T_{\mathfrak{p}} \mathcal{M}$ for all $\mathfrak{p} \in \mathcal{M}$ and all $\mathrm{g}|_{\mathfrak{p}}$ vary smoothly. With $\mathrm{g},$ we can define other important geometric quantities like the norm, angle and curvature.

A \emph{Riemannian manifold} is a smooth manifold $\mathcal{M}$ equipped with a Riemannian metric $\mathrm{g},$ denoted as the pair $(\mathcal{M}, \mathrm{g}).$ When there is no confusion, we usually omit the metric $\mathrm{g}$ and simply write $\mathcal{M}.$ For a Riemannian manifold, it is useful to use the \emph{Riemannian density or volume form} to do the integration on the manifold, denoted as $\dd V,$ where $V \equiv V_g.$ Especially, if $f: \mathcal{M} \rightarrow \mathbb{R}$ is a compactly supported continuous function, we denote the integral of $f$ over $\mathcal{M}$ as $\int_{\mathcal{M}} f \dd V.$

In statistics and data science, observations are oftentimes collected in Euclidean space. Moreover, it is much easier to do calculations in Euclidean space. Motivated by these aspects, it is useful to link an arbitrary Riemannian manifold $(\mathcal{M}, \mathrm{g})$ isometrically to some subspace of the Euclidean space with a specific metric. The feasibility of the above statement is guaranteed by the embedding theory. An \emph{embedding} $\iota: \mathcal{M} \rightarrow \mathbb{R}^r,$ for some $r \geq m,$ is a smooth map and a homeomorphism onto its image. When this happens, $\mathcal{M}$ is called a $m$-dimensional \emph{submanifold} of the Euclidean space and $\mathbb{R}^r$ is said to be its \emph{ambient space}. Moreover, we call $\iota(\mathcal{M})$ the \emph{embedded submanifold}. From the computational perspective, we are interested in the \emph{isometric embedding} so that the calculations of the distances, angles and curvatures reduce to those in the Euclidean space. 

The following theorem of John Nash \cite{MR75639} indicates that every  Riemannian manifold can be considered as a submanifold of some ambient space $\mathbb{R}^r.$   

\begin{thm}\label{thm_nashembedding}
Let $(\mathcal{M}, \mathrm{g})$ be a $m$-dimensional Riemannian manifold. Then there exists an isometric embedding $\iota: \mathcal{M} \rightarrow \mathbb{R}^r$ from $\mathcal{M}$ to $\mathbb{R}^r$ for some $r$. Moreover, when $\mathcal{M}$ is compact, it is possible that 
\begin{equation*}
r \leq \frac{m(3m+11)}{2};
\end{equation*}
when $\mathcal{M}$ is non-compact, it is possible that 
\begin{equation*}
r \leq \frac{m(m+1)(3m+11)}{2}. 
\end{equation*}
\end{thm}

 With the above embedding $\iota$ and the \emph{pushforward}, we consider the induced metric $\iota_* \mathrm{g}$ on the embedded submanifold $\iota(\mathcal{M})$ which is clearly another Riemannian manifold. For integration, we consider the induced volume form $\iota_* \dd V$ which is the Riemannian density on $\iota(\mathcal{M}).$ Consequently, for any integrable function $f: \iota(\mathcal{M}) \rightarrow \mathbb{R},$ we can define the associated integral as $\int_{\iota(\mathcal{M})} f \iota_* \dd V.$ 

Finally, as in Section \ref{man.sec}, we explain how the above manifold model is connected to our statistical applications. Suppose we observe i.i.d. samples $\mathbf{x}_i, 1 \leq i \leq n,$ according to the random vector $X$ as in Assumption \ref{assum_generalmodelassumption}. Moreover, we assume that the support of $X$ is $\iota(\mathcal{M}).$ Therefore, for any integrable function $\zeta: \iota(\mathcal{M}) \rightarrow \mathbb{R},$ since $X$ is supported on $\iota(\mathcal{M}),$ the calculations of $\mathbb{E}(\zeta(X))$ can be efficiently reduced to the integration on $\mathcal{M}$ using the above induced measure and volume form (Section \ref{man.sec}). We emphasize that in real applications, it is the embedded submanifold that matters since the observations are sampled according to $X$ which is supported on $\iota(\mathcal{M}).$ Consequently, we focus on the understanding of the geometric structure of $\iota(\mathcal{M})$ rather $\mathcal{M}$ and $\iota$ separately. For more discussions on this perspective, we refer the readers to \cite{cheng2013local,DUNSON2021282, shen2022robust,wu2018think}.


}

\subsection{Concentration Inequalities}


The following lemma concerns tails of sub-Gaussian random vectors. {
\begin{lem} \label{sg.bnd.lem}
	Let $B$ be an $m\times n$ matrix, and let $\mathbf{x}$ be a mean zero, sub-Gaussian random vector in $\R^n$ with parameter bounded by $K$. Then for any $t\ge 0$, we have
	\[
	\mathbb{P}(\|B \mathbf{x} \|_2\ge CK\|B\|_F+t)\le \exp\bigg(-\frac{ct^2}{K^2\|B\|^2}\bigg).
	\]
\end{lem} 
\begin{proof}
See page 144 of \cite{vershynin2018high}.
\end{proof}

The next lemma is the classical Chernoff bound for Binomial  random variables.

\begin{lem} \label{cher.lem}
	Let $x_1, ..., x_n$ be independent random variables with $\mathbb{P}(x_k=1)=\mathsf{p}$ and $\mathbb{P}(x_k=0)=1-\mathsf{p}$ for  each $k$. Then for any $t>n\mathsf{p}$, we have
	\[
	\mathbb{P}\bigg(\sum_{k=1}^n x_k>t\bigg)\le e^{-n\mathsf{p}}\bigg( \frac{en \mathsf{p}}{t}\bigg)^t.
	\]
\end{lem}
\begin{proof}
See Section 2.3 of  \cite{vershynin2018high}.
\end{proof}

The next lemma provides a Bernstein type inequality of sub-exponential random variables.
\begin{lem}\label{lem_subexponentialconcentration} Let $x_i, 1 \leq i \leq n,$ be independent mean zero sub-exponential random variables. Then for every $t \geq 0$
\begin{equation*}
\mathbb{P}\left( \left| \sum_{i=1}^n x_i \right|\geq t \right) \leq 2 \exp\left[-c \min \left( \frac{t^2}{\sum_{i=1}^n \| x_i\|_{\psi_1}}, \frac{t}{\max_i \| x_i\|_{\psi_1}} \right) \right],
\end{equation*} 
where $c>0$ is some universal constant and $\|x_i \|_{\psi_1}$ is the sub-exponential norm of $x_i,$ i.e., $\| x_i\|_{\psi_1}=\inf\{t>0: \mathbb{E} \exp(|x_i|/t) \leq 2\}.$ 
\end{lem}
\begin{proof}
See Theorem 2.8.1 of \cite{vershynin2018high}.
\end{proof}
}

In the following lemmas, we will use stochastic domination to characterize the high-dimensional concentration. The next lemma collects some useful concentration inequalities for the noise vector.

\begin{lem}\label{lem_concentrationinequality} Suppose Assumption \ref{assum_mainassumption} holds. Then we have that 
$$
	\max_{i \neq j}	\frac{1}{\sigma^2p}\left| \zb_i^\top \zb_j \right| \prec  {p^{-1/2}},
$$
	and 
$$ 
	\max_i	\left| \frac{1}{ \sigma^2p} \| \zb_i \|_2^2-1 \right| \prec  { p^{-1/2}}. 
$$
\end{lem}
\begin{proof}
	See Lemma A.4 of \cite{DW2} or Lemmas A.1 and A.2 of \cite{9205615}. 
\end{proof}

The following lemma provides a sharp upper bound for the norm of $\bK^*_n.$ 
\begin{lem}\label{lem_matrixnorm} Suppose Assumptions \ref{assum_mainassumption} and \ref{ker.assu} hold. For any $h>0$, let $\bK^*_n=(K^*(i,j))_{1\le i,j\le n}$ and $K^*(i,j)=f\left(\frac{\| \xb_i-\xb_j \|_2}{\sfh^{1/2}} \right).$
	Then we have $\|\bK^*_n\| \prec n$.
\end{lem} 
\begin{proof}
	Note that $\max_{i,j}|\bK^*_n(i,j)| \prec 1.$ The proof follows from Lemma \ref{lem_circle}. 
\end{proof}

\subsection{Useful Tools from Linear Algebra}

\begin{lem}\label{lem_hadmaradproduct} Suppose $M$ is a real symmetric matrix with nonnegative entries and $E$ is a real symmetric matrix. Then we have that 
	\begin{equation*}
	s_1(M \circ E) \leq s_1(M) \max_{i,j}|E(i,j)|,
	\end{equation*}
	where $s_1(\cdot)$ is the largest singular value of the given matrix, and $A\circ B$ is the Hadamard product of two matrices.. 
\end{lem}
\begin{proof}
	See Lemma A.5 of \cite{el2010spectrum}. 
\end{proof}

\begin{lem}\label{lem_circle} Le $A=(a_{ij})$ be a real $n \times n$ matrix. For $1 \leq i \leq n,$ let $R_i=\sum_{j \neq i} |a_{ij}|$ be the sum of the absolute values of the non-diagonal entries in the $i$th row. Let $D(a_{ii}, R_i) \subset \mathbb{R}$ be a closed disc with center $a_{ii}$ and radius $R_i$ referred to as the \emph{Gershgorin disc}. Then every eigenvalue of $A=(a_{ij})$ lies within at least one of the Gershgorin discs $D(a_{ii}, R_i).$ 
\end{lem}
\begin{proof}
	See Section 6.1 of \cite{2012matrix}. 
\end{proof}

\begin{lem}\label{lem_pertubationargument}
	Let $A$ and $\widehat{A}$ be two compact positive self-adjoint operators on a Hilbert space $\mathcal{H},$ with nondecreasing eigenvalues $\{\lambda_j(A)\}$ and $\{\lambda_j(\widehat{A})\}.$ Let $w_k$ be a normalized eigenvctor of $A$ associated with the eigenvalue $\lambda_k.$ If $l_k>0$ satisfies 
	\begin{equation*}
	\lambda_{k-1}-\lambda_k \geq l_k, \ \lambda_k-\lambda_{k+1} \geq l_k, \ \|A -\widehat{A} \| \leq \frac{l_k}{2},
	\end{equation*} 
	then we have that 
	\begin{equation*}
	\|w_k-\widehat{w}_k \| \leq \frac{4}{l_k}\|A-\widehat{A} \|. 
	\end{equation*}
\end{lem}
\begin{proof}
	See Proposition 2 of \cite{MR2558684}. 
\end{proof}

\section{Proof of Main Theoretical Results}

In this section, we prove the main results, i.e., Theorems \ref{thm_mainthm}, \ref{cor_maincor}, and \ref{thm_mainthm2}, and Proposition \ref{lem_bandwidthconcentration}. {Since $\xb_i^0, 1 \leq i \leq n,$ are identically distributed and the kernel random matrices only reply on the pairwise distance, in what follows, without loss of generality, we assume that $\mathbb{E} \xb_i^0=0.$ }

\subsection{Convergence of Noiseless Kernel Matrix: Proof of Theorem \ref{thm_mainthm}} 

In the following, according to Lemma \ref{lem_modelreduced} of Section \ref{man.sec}, without loss of generality, we will consider the reduced data $\{\yb_i^0\}_{1\le i\le n}$ and their related quantities $\{\xb_i^0\}_{1\le i\le n}$ and $\{\zb_i^0\}_{1\le i\le n}$. For simplicity, with a slight abuse of notation,  we omit their superscripts and denote them as $\{\yb_i\}$, $\{\xb_i\}$ and $\{\zb_i\}$. Recall that in this subsection, we work with the Gaussian kernel function $f(x)=\exp(-x^2).$


\begin{proof}[\bf Eigenvalue convergence.] Recall that $\sfh$ is defined according to 
	\begin{equation*}
	\nu^*_n(\sfh)=\omega,
	\end{equation*}
	where
	\beq
	\nu^*_n(t)=\frac{2}{n(n-1)}\sum_{1\le i<j\le n}1_{\{d^*_{ij}\le t\}}, \nonumber
	\eeq 
	and $d^*_{ij}=\|\xb_i-\xb_j\|_2^2$  for all $1\le i<j\le n$. By Proposition \ref{lem_bandwidthconcentration}, we have $\frac{\sfh}{h_n}=1+\mathrm{o}_{\prec}(1)$, and
	\begin{equation}\label{eq_choicestepone}
	\sum_{i=1}^r \theta_i\prec \sfh\prec \sum_{i=1}^r\theta_i.
	\end{equation}
	With a slight abuse of notations,  we can define the kernel matrices for ${\yb_i}, {\xb_i},$ and ${\zb_i}$ using the bandwidth $\mathsf{h}$ as $\bK_y, \bK_x$ and $\bK_z,$ respectively. For example, $\bK_x=(K(\xb_i,\xb_j))_{1\le i,j\le n}$, where $K(\xb_i,\xb_j)=\exp\left(-\|\xb_i-\xb_j\|_2^2/\sfh\right)$. 
	
	First and foremost, we show that $\bK_y$ converges to $\bK_x \equiv \bK_n^*$.
	We denote an auxiliary matrix $\bK_c$ whose $(i,j)$-th entry is given by
	\begin{equation}\label{eq_wcform}
	\bK_c(i,j)=\exp\left(-\frac{2(\xb_i-\xb_j)^\top (\zb_i-\zb_j) }{\sfh} \right),
	\end{equation}
	and define
	\begin{equation*}
	\bE_0=\exp\left(- \frac{2\sigma^2p}{\sfh} \right)\mathbf{1} \mathbf{1}^\top+\left(1-\exp\left(- \frac{2\sigma^2p}{\sfh} \right)\right){\bf I}_n,
	\end{equation*}
	where $\mathbf{1} \in \mathbb{R}^n$ is a vector with all entries being unity. Let $\bA \circ \bB$ be the Hadamard product of two matrices $\bA$ and $\bB$. We define another auxiliary matrix 
	\begin{equation}\label{eq_wsdefn}
	\bK_{s}=\bE_0 \circ \bK_x. 
	\end{equation}  
	Using the above notations and the trivial relation that
	\[
	\bK_s=(\mathbf{1} \mathbf{1}^\top) \circ \bK_s,  \qquad \bK_y=\bK_x \circ \bK_c \circ \bK_z,
	\]
	we readily see that 
	\begin{align}\label{eq_errordecomposition}
	\bK_y-\bK_{s} & =\left[ \left( \bK_c-\mathbf{1} \mathbf{1}^\top \right) \circ \bK_z \circ \bK_x \right]+\left[ (\bK_z-\bE_0) \circ \bK_x \right] \nonumber \\
	& := \mathcal{E}_1+\mathcal{E}_2. 
	\end{align}
	We first control $\mathcal{E}_2.$ Note that $\bK_z(i,i)-\bE_0(i,i)=0$ for all $1\le i\le n$. 
	By the mean value theorem, we obtain that, for $i \neq j,$
	\begin{equation*}
	\bK_z(i,j)=\exp\left(-\frac{2\sigma^2p}{\sfh} \right)-\frac{\sigma^2p}{\sfh}\exp\left(-\frac{\zeta(i,j) \sigma^2p}{\sfh} \right) \Delta(i,j),
	\end{equation*}
	for some $\zeta(i,j)$ between $\| \zb_i-\zb_j \|_2^2/(\sigma^2p)$ and $2$, and $\Delta(i,j)$ defined as
	\begin{equation} \label{Deltaij}
		\Delta(i,j)=\frac{\| \zb_i-\zb_j \|_2^2}{\sigma^2 p}-2,
	\end{equation}
can be controlled as follows.
	Since
	\beq
	|\Delta(i,j)|=\bigg|\frac{\|\zb_i\|_2^2+\|\zb_j\|_2^2-2\zb_i^\top\zb_j}{\sigma^2 p}-2\bigg|\le  \bigg|\frac{\|\zb_i\|_2^2}{\sigma^2 p}-1\bigg|+\bigg|\frac{\|\zb_i\|_2^2}{\sigma^2 p}-1\bigg|+\frac{2|\zb_i^\top\zb_j|}{\sigma^2p}, \nonumber
	\eeq
	by Lemma \ref{lem_concentrationinequality} and Assumption \ref{eq_aspect}, we conclude that 
	\beq \label{Delta.bnd}
	\max_{i \neq j} |\Delta(i,j)| \prec {\frac{1}{\sqrt{p}}}.
	\eeq
	Now since by definition $|\zeta(i,j)-2|\le \big|\| \zb_i-\zb_j \|_2^2/(\sigma^2 p)-2\big|=|\Delta(i,j)|$, we  have
	\begin{equation*}
	\max_{i \neq j} \zeta(i,j)=2+{\OO_{\prec}(\frac{1}{\sqrt{p}})}. 
	\end{equation*} 
	Thus, it follows that, 
	\begin{align}
	\max_{1\le i,j\le n}|\bK_z(i,j)-\bE_0(i,j)| &=\max_{1\le i, j\le n, i\ne j}\bigg|\frac{\sigma^2p}{\sfh}\exp\left(-\frac{\zeta(i,j) \sigma^2p}{\sfh} \right) \Delta(i,j)\bigg| \nonumber \\
	&\prec{ \frac{n^{\beta+\eta}}{\sum_{i=1}^r \theta_i}\cdot\frac{1}{\sqrt{p}}\nonumber} \\
	&\prec{ \frac{1}{\sum_{i=1}^r n^{\alpha_i-\beta-\frac{\eta}{2}}}.} \label{eq_bounddifference}
	\end{align}
	Since $\bK_x$ is symmetric with nonnegative entries, by (\ref{eq_bounddifference}) and Lemma \ref{lem_hadmaradproduct}, we obtain that 
	\begin{equation} \label{eq_e2control}
	\frac{1}{n}\| \mathcal{E}_2 \| \prec \frac{1}{n}\cdot\max_{1\le i,j\le n}|\bK_z(i,j)-\bE_0(i,j)|\cdot  \| \bK_x \|\prec { \frac{1}{\sum_{i=1}^r n^{\alpha_i-\beta-\frac{\eta}{2}}},}
	\end{equation}
	where in the last inequality we used Lemma \ref{lem_matrixnorm} which implies
	\begin{equation}\label{eq_wxinusenorm}
	\frac{1}{n}\|\bK_x \| \prec 1. 
	\end{equation} 
	Analogously, for $\mathcal{E}_1,$ applying Lemma \ref{lem_hadmaradproduct}, we see that 
	\begin{equation}\label{eq_e1111control}
	\| \mathcal{E}_1 \| \leq \left(\max_{i,j}|\bK_c(i,j)-1| \right) \cdot\max_{i,j} \bK_z(i,j)\cdot \| \bK_x \|.
	\end{equation}
	On one hand, by (\ref{eq_bounddifference}) and the fact $\bK_z(i,i)=1$, we have that 
	\begin{equation*}
	\max_{i,j} \bK_z(i,j) \prec 1. 
	\end{equation*}
	Moreover, under Assumption \ref{assum_mainassumption}, since $\zb_i-\zb_j$ is a centred sub-Gaussian random vector with covariance matrix $2\sigma^2{\bf I}_p$, then conditional on $\{\xb_i\}_{1\le i\le n}$, it holds that
	\[
	\frac{|(\xb_i-\xb_j)^\top (\zb_i-\zb_j)|}{\sigma \sqrt{2}}\prec \|\xb_i-\xb_j\|_2.
	\]
	By the  fact that for all $1\le i \le n$, we have
	\beq
	\E\|\xb_i\|_2^2 = \sum_{k=1}^r \E (x_{ik})^2 = \sum_{i=1}^r \theta_i, \nonumber
	\eeq 
	it follows that {
		\beq\label{x.bnd}
		\|\xb_i-\xb_j\|_2^2  \leq 2\|\xb_i\|_2^2+2\|\xb_j\|_2^2\prec \sum_{i=1}^r \theta_i+\theta_1\log n,
		\eeq 
		where the last inequality follows from Lemma \ref{sg.bnd.lem} by setting $t \asymp \sqrt{\log n} \| B \|$ with $B=\operatorname{diag}\{\theta_1^{1/2}, \cdots, \theta_r^{1/2}\}.$ Moreover, by a union bound over $\{(i,j): 1\le i,j\le n,i\ne j\}$ with Definition \ref{defn_stochasdomi}, we have
		\beq\label{eq_xijbound}
		\max_{i\ne j}\|\xb_i-\xb_j\|_2^2\prec \sum_{i=1}^r n^{\alpha_i}.
		\eeq }
	Thus, we have
	\beq \label{cross.pd}
	\max_{i,j}\frac{|(\xb_i-\xb_j)^\top (\zb_i-\zb_j)|}{\sigma \sqrt{2}}\prec \bigg( \sum_{i=1}^r \theta_i\bigg)^{1/2}.
	\eeq
	Recall the elementary inequality that 
	\beq \label{basic.ineq}
	|e^x-1|\le 2|x|, \qquad \text{for all $|x|< 1$.}
	\eeq
	Then, as
	\begin{equation}\label{eq_controlcontrol}
	{\frac{\sigma\sqrt{2}}{\sfh} \asymp \frac{n^{\beta/2}}{ \sum_{i=1}^r \theta_i},}
	\end{equation}
	we have that under the assumption of (\ref{eq_sigmaimagnititude})
	\beq \label{cross.ineq}
\max_{i,j}\frac{2|(\xb_i-\xb_j)^\top (\zb_i-\zb_j) |}{\sfh} \prec \frac{n^{\beta/2}\big( \sum_{i=1}^r\theta_i\big)^{1/2}}{ \sum_{i=1}^r\theta_i}\prec \frac{1}{(\sum_{i=1}^r n^{\alpha_i-\beta})^{1/2}}<1,
	\eeq
	for sufficiently large $n$. Therefore, 
	\begin{equation*}
	\max_{i,j}\left| \bK_c(i,j)-1 \right| \prec {\max_{i,j}\frac{2|(\xb_i-\xb_j)^\top (\zb_i-\zb_j) |}{\sfh} \prec \frac{1}{(\sum_{i=1}^r n^{\alpha_i-\beta})^{1/2}}. }
	\end{equation*}  
	Together with (\ref{eq_wxinusenorm}), we immediately see from (\ref{eq_e1111control}) that 
	\begin{equation} \label{eq_e1control}
	\frac{1}{n} \|\mathcal{E}_1 \| \prec { \frac{1}{(\sum_{i=1}^r n^{\alpha_i-\beta})^{1/2}}.}
	\end{equation}
	Combining  (\ref{eq_e2control}) and  (\ref{eq_e1control}), we have that 
	\begin{align} 
	\frac{1}{n} \|\bK_y-\bK_s \| &\prec {\frac{1}{\sum_{i=1}^r n^{\alpha_i-\beta-\frac{\eta}{2}}}+\frac{1}{(\sum_{i=1}^r n^{\alpha_i-\beta})^{1/2}} }\nonumber\\
	 &\prec \frac{1}{(\sum_{i=1}^r n^{\alpha_i-\beta})^{1/2}},\label{conv.Ks}
	\end{align}
	as we assume that {$\sum_{i=1}^r n^{\alpha_i}\gg n^{\beta+\eta}$.}
	
	Additionally, since $\bK_x=(\mathbf{1} \mathbf{1}^\top) \circ \bK_x,$ using (\ref{eq_wsdefn}) and Lemma \ref{lem_hadmaradproduct},   we have
	\begin{align}
	\frac{1}{n} \| \bK_x-\bK_s\| & \le \frac{1}{n}\|\bK_x\| \cdot \max_{i,j}|\bE_0(i,j)-1| \nonumber \\
&	\prec \frac{\sigma^2 p}{\sfh} \asymp {\frac{1}{\sum_{i=1}^r n^{\alpha_i-\beta-\eta}}, } \label{C19}
	\end{align}
	where in the second step we used (\ref{eq_wxinusenorm}), (\ref{basic.ineq}) and the fact that { $\sigma^2p/\sfh \asymp \frac{1}{\sum_{i=1}^r n^{\alpha_i-\beta-\eta}}<1$} for sufficiently large $n$. {Combining with  (\ref{conv.Ks}), it yields that
	\begin{equation} \label{eq_matrixclose}
	\frac{1}{n} \left\| \bK_x-\bK_y \right\| \prec { \frac{1}{\sum_{i=1}^r n^{\alpha_i-\beta-\eta}}+\frac{1}{(\sum_{i=1}^r n^{\alpha_i-\beta})^{1/2}}}.
	\end{equation} }  

	Next, we control the error $\|\bK_y-\bK_n\|$. Note that for some $h^*$ between $h_n$ and $\sfh$, it holds that
\begin{align}\label{eq_ccone}
	|\bK_y(i,j)-\bK_n(i,j)|&\le \bigg|\frac{\|\yb_i-\yb_j\|_2^2}{\sfh}- \frac{\|\yb_i-\yb_j\|_2^2}{h_n}\bigg|\cdot\exp\bigg(- \frac{\|\yb_i-\yb_j\|_2^2}{h^*}\bigg) \nonumber \\
	&\le \bigg|\frac{\sfh}{h_n}-1\bigg|\cdot \frac{\|\yb_i-\yb_j\|_2^2}{\sfh}.
\end{align}
On the one hand, we have
\begin{equation}\label{eq_elementary}
	\|\yb_i-\yb_j\|_2^2\le 2\|\xb_i-\xb_j\|_2^2+2\|\zb_i-\zb_j\|_2^2.
\end{equation}
By (\ref{Delta.bnd}), we have
\[
\max_{i\ne j}\|\zb_i-\zb_j\|_2^2\prec { n^{\beta+\eta}}. \nonumber
\]
Together with (\ref{eq_xijbound}) and (\ref{eq_choicestepone}), under the assumption of (\ref{eq_sigmaimagnititude}), we have
\beq \label{y/h}
\max_{i\neq j} \frac{\|\yb_i-\yb_j\|_2^2}{\sfh} \prec { \frac{\sum_{i=1}^r n^{\alpha_i}+n^{\beta+\eta}}{\sum_{i=1}^r n^{\alpha_i}}\asymp 1}.
\eeq
On the other hand, by Proposition \ref{lem_bandwidthconcentration}, it holds that
\beq 
\bigg|\frac{\sfh}{h_n}-1\bigg|={ \OO_{\prec}\bigg(\frac{1}{\sum_{i=1}^r n^{\alpha_i-\beta-\eta}}+\frac{1}{(\sum_{i=1}^r n^{\alpha_i-\beta})^{1/2}}\bigg)}. \nonumber
\eeq
Combining the above bound with (\ref{y/h}), by Lemma \ref{lem_circle} and (\ref{eq_ccone}), 	we have
\begin{align} \label{partA}
	\frac{1}{n}\|\bK_y-\bK_n\| &\le \max_{i\ne j} |\bK_y(i,j)-\bK_n(i,j)| \nonumber \\
	&	={ \OO_{\prec}\left(\frac{1}{\sum_{i=1}^r n^{\alpha_i-\beta-\eta}}+\frac{1}{(\sum_{i=1}^r n^{\alpha_i-\beta})^{1/2}}\right).}
\end{align}

Combining (\ref{eq_matrixclose})  and (\ref{partA}), we conclude the proof of (\ref{eq_matrixclose1}). (\ref{eq_eigs_v}) simply comes from Weyl's inequality. 
\end{proof}

\begin{proof}[\bf Eigenvector convergence.]
{ We prove the eigenvector convergence by leveraging the eigenvalue convergence results. In particular, note that the obtained rate of convergence for $\left\| \bK_s-\bK_y \right\|$ as in (\ref{conv.Ks}) is in general faster than the obtained rate of convergence $\left\| \bK_x-\bK_y \right\|$ as in (\ref{eq_matrixclose}). Moreover, 
	by definition we have
	\begin{align*}
	\bK_s=\bK_x\circ \bE_0 &= \bK_x\circ \exp\left(- \frac{2\sigma^2p}{\sfh} \right)\mathbf{1} \mathbf{1}^\top+\bK_x\circ\left(1-\exp\left(- \frac{2\sigma^2p}{\sfh} \right)\right){\bf I}_n\\
	&=\exp\left(- \frac{2\sigma^2p}{\sfh} \right)\bK_x+\left(1-\exp\left(- \frac{2\sigma^2p}{\sfh} \right)\right){\bf I}_n,
	\end{align*}
which indicates  $\bK_x$ and $\bK_s$	share the same set of eigenvectors $\{\vb_i\}_{1\le i\le n}$, and their eigenvalues differ only up to a small isotropic shift, i.e., 
	\begin{align} \label{eig.relation}
	\lambda_i(\bK_s)&=\exp\bigg(-\frac{2\sigma^2p}{\sfh}\bigg)\mu_i+\left(1-\exp\left(- \frac{2\sigma^2p}{\sfh} \right)\right) \nonumber \\
&	:=A_n\mu_i+B_n,\qquad 1\le i\le n.
	\end{align}
	Therefore, instead of directly using Davis-Kahan theorem \cite{yu2015useful}, we can take advantage of these facts and obtain potentially faster rate of convergence for the eigenvectors. }
	
	Our discussion relies on the following identity, which follows from the residual theorem and the spectral decomposition. Specifically, for any positive definite matrix $\bA$ admitting the spectral decomposition $\bA=\sum_{i=1}^n \lambda_i(\bA) \bm{\zeta}_i \bm{\zeta}_i^\top,$ we have 
	\begin{equation}\label{eq_integralrepresentation}
	\bm{\zeta}_i \bm{\zeta}_i^\top=\frac{1}{2 \pi \ri} \oint_\Gamma (z{\bf I}-\bA)^{-1} \dd z,  
	\end{equation} 
	where $\Gamma \subset \mathbb{C}$ is some simply connected contour only containing $\lambda_i(\bA)$ and no other eigenvalues.  Based on the above representation,  for each $i$ satisfying (\ref{eq_assumption}), we have {
		\[
		\ub_i \ub_i^\top=\frac{1}{2 \pi \ri} \oint_{\Gamma_i} (z{\bf I}-\bK_n)^{-1} \dd z,  
		\]
		and 
	\begin{align}\label{eq_uvdifference}
	\langle \ub_i, \vb_i  \rangle^2 & =\frac{1}{2 \pi \ri} \oint_{\Gamma_i} \mathbf{v}_i^\top  (z{\bf I}-n^{-1}\bK_s)^{-1}\mathbf{v}_i \dd z \nonumber \\
	&\quad +\frac{1}{2 \pi \ri} \oint_{\Gamma_i} \mathbf{v}_i^\top \left[ (z{\bf I}-n^{-1}\bK_y)^{-1}-(z{\bf I}-n^{-1}\bK_s)^{-1} \right] \mathbf{v}_i \dd z  \nonumber \\
	&\quad +\frac{1}{2 \pi \ri} \oint_{\Gamma_i} \mathbf{v}_i^\top \left[ (z{\bf I}-n^{-1}\bK_n)^{-1}-(z{\bf I}-n^{-1}\bK_y)^{-1} \right] \mathbf{v}_i \dd z  \nonumber \\
	&:= \mathsf{L}_1+\mathsf{L}_2+\mathsf{L}_3, 
	\end{align}}
	for some contour $\Gamma_i:=\mathbb{B}(\gamma_i,\frac{\mathsf{r}_i}{C}),$ for some large constant $C>0,$ where $\mathbb{B}(\gamma_i,\frac{\mathsf{r}_i}{C})$ is the disk centred at $\gamma_i$ with radius $\mathsf{r}_i/C$ and we recall $ \mathsf{r}_i:=\min\{\gamma_{i-1}-\gamma_i, \gamma_i-\gamma_{i+1}\}$. In particular, $\Gamma_i$ only contains $\lambda_i$ and no other eigenvalues, since
		\begin{align*}
		&|\lambda_i-\gamma_i|\le |\lambda_i-\mu_i|+|\gamma_i-\mu_i|\\
		&\prec { \frac{1}{(\sum_{i=1}^r n^{{\alpha_i-\beta}})^{1/2}}+ \frac{1}{\sqrt{n}}+\frac{1}{\sum_{i=1}^r n^{\alpha_i-\beta-\eta}}=\mathrm{o}(\sfr_i),}
	\end{align*}
	where the last inequality follows from the assumption of (\ref{eq_assumption}) and that 
	\begin{equation}\label{eq_largesmallbound}
		\frac{1}{(\sum_{i=1}^r n^{{\alpha_i-\beta}})^{1/2}}=\OO\left( \frac{1}{\sqrt{n}} \right),
	\end{equation} 
and the second last inequality follows from Part I of Theorem \ref{thm_mainthm}, the bound 
	 \beq \label{mu-gamma}
	 |\mu_i-\gamma_i|\le \|\mathcal{K}-\mathcal{K}_n\|\prec \frac{1}{\sqrt{n}},
	 \eeq
	 based on Lemma \ref{lem_mutiplepertubation} and Weyl's inequality.
	To control (\ref{eq_uvdifference}), on the one hand, it follows that
	\begin{align} \label{eigen.Ks}
	|\lambda_i(\bK_s)-\gamma_i|=|A_n\mu_i+B_n-\gamma_i| & \le |A_n-1|\cdot \mu_i+|\mu_i-\gamma_i| +B_n\notag \\
	& \prec { \frac{1}{\sum_{i=1}^r n^{\alpha_i-\beta-\eta}}+\frac{1}{\sqrt{n}}},
	\end{align}
	where in the last inequality we used the fact that $\max_{1\le i\le n}\mu_i\le \|\bK_x\|\prec 1$ and { $|A_n-1|\asymp 1/(\sum_{i=1}^r n^{\alpha_i-\beta-\eta})$} (e.g. see (\ref{eq_controlcontrol})) and the definition of $B_n$. Now under the assumption of (\ref{eq_assumption}), we conclude that when $n$ is sufficiently large, $\lambda_i(\bK_s)$ is the only simple pole of the resolvent $(z{\bf I}-n^{-1}{\bK}_s)^{-1}$ inside $\Gamma_i$. Using the spectral decomposition of $(z{\bf I}-n^{-1}\bK_s)^{-1}$ and residual theorem, we have that 
	\begin{equation}\label{eq_L1control}
	\sfL_1=1. 
	\end{equation}
	On the other hand, the control of $\sfL_2$ relies on the   
	resolvent identity (see equation (5.5) of \cite{MR678094}) that 
	\begin{align}\label{eq_resolventidentity}
	& (z{\bf I}-n^{-1}\bK_y)^{-1}-(z{\bf I}-n^{-1}\bK_s)^{-1} \notag \\
	& =(z{\bf I}-n^{-1}\bK_y)^{-1}\left[n^{-1} \bK_y-n^{-1} \bK_s\right] (z{\bf  I}-n^{-1}\bK_s)^{-1}. 
	\end{align}
	By the definition of $\Gamma_i,$  we have
	\begin{align*}
	&|\lambda_i(\bK_y)-\gamma_i|\le |\lambda_i(\bK_y)-\gamma_i| \le |\lambda_i(\bK_y)-\mu_i|+|\gamma_i-\mu_i| \\
	&\prec { \frac{1}{(\sum_{i=1}^r n^{{\alpha_i-\beta}})^{1/2}}+ \frac{1}{\sqrt{n}}+\frac{1}{\sum_{i=1}^r n^{\alpha_i-\beta-\eta}},}
	\end{align*}
	where the last inequality follows from 
	using (\ref{eq_matrixclose}) and (\ref{mu-gamma}) as well as the definition of $B_n$. Thus by (\ref{eq_largesmallbound}), under the same event  we have that 
	\begin{equation*}
	\inf_{z \in \Gamma_i} \min \{ |\lambda_i(\bK_y)-z|, |\lambda_{i-1}(\bK_y)-z|, |\lambda_{i+1}(\bK_y)-z| \} \asymp \mathsf{r}_i.
	\end{equation*}
	Here we used the assumption of (\ref{eq_assumption}), which implies that for { $\sfr_i \gg  (n^{-1/2}+1/(\sum_{i=1}^r n^{\alpha_i-\beta-\eta}))$} regardless whether $\sfr_i$ is less than 1 or not. Moreover, by the definition of the resolvent, we have that
	\begin{align*}
	\sup_{z \in \Gamma_i} \| (z{\bf I}-n^{-1}\bK_y)^{-1} \| & \leq \sup_{z \in \Gamma_i} \left( \frac{1}{|\lambda_i(\bK_y)-z|}+\frac{1}{|\lambda_{i-1}(\bK_y)-z|}+\frac{1}{|\lambda_{i+1}(\bK_y)-z|} \right)  \\
&	\prec \frac{1}{\mathsf{r}_i}.
	\end{align*}
	Similarly, by (\ref{eigen.Ks}), we also have
	\[
	|\lambda_i(\bK_s)-\gamma_i|\le 	|\lambda_i(\bK_s)-\mu_i|+	|\mu_i-\gamma_i|\prec \sfr_i,
	\]
	which implies
	\begin{align*}
	\sup_{z \in \Gamma_i} \| (z{\bf I}-n^{-1}\bK_s)^{-1} \| \prec  \frac{1}{\mathsf{r}_i}.
	\end{align*}
	Together with (\ref{eq_resolventidentity}) and (\ref{conv.Ks}), we have that {
	\begin{equation*}
	\sfL_2 \prec  \frac{1}{\mathsf{r}_i^2}\left[\frac{1}{(\sum_{i=1}^r n^{{\alpha_i-\beta}})^{1/2}}\right]. 
	\end{equation*}}
{ Similarly, by (\ref{partA}), we  have
		\begin{equation*}
		\inf_{z \in \Gamma_i} \min \{ |\lambda_i-z|, |\lambda_{i-1}-z|, |\lambda_{i+1}-z| \} \asymp \mathsf{r}_i,
	\end{equation*}
which implies
		\begin{align*}
		\sup_{z \in \Gamma_i} \| (z{\bf I}-n^{-1}\bK_n)^{-1} \| \prec  \frac{1}{\mathsf{r}_i},
	\end{align*}
and {
\[
\sfL_3 \prec { \frac{1}{\mathsf{r}_i^2}\left(\frac{1}{\sum_{i=1}^r n^{\alpha_i-\beta-\eta}}+\frac{1}{(\sum_{i=1}^rn^{\alpha_i-\beta})^{1/2}}\right).}
\] }
	Insert the above bounds and (\ref{eq_L1control}) back into (\ref{eq_uvdifference}), we immediately obtain that 
	\begin{align*}
	\left| \langle \ub_i, \vb_i \rangle^2-1 \right| \prec {\frac{1}{\mathsf{r}_i^2}\left(\frac{1}{\sum_{i=1}^r n^{\alpha_i-\beta-\eta}}+\frac{1}{(\sum_{i=1}^r n^{\alpha_i-\beta})^{1/2}}\right).}
	\end{align*}}
	This completes our proof.
\end{proof}

\subsection{Convergence to Population Integral Operator: Proof of Theorem \ref{cor_maincor}}\label{bbb_reduced}

\begin{proof}[\bf Eigenvalue convergence.]  By the triangle inequality, we have
	\begin{equation*}
	|\lambda_i-\gamma_i| \leq |\lambda_i-\mu_i| + |\mu_i-\gamma_i|.
	\end{equation*}
	The first term of the right-hand side of the above equation can be controlled by Theorem \ref{thm_mainthm}. For the second term, since $\bK_n^*$ and $\widetilde{\mathcal{K}}$ have the same bandwidth chosen as (\ref{sfh}) in the main paper, we can use a sample splitting argument on $\sfh$
	and apply Lemma  \ref{lem_mutiplepertubation} to obtain the final upper bound. The proof then follows from (\ref{eq_largesmallbound}).   
\end{proof}

\begin{proof}[\bf Eigenfunction convergence.]
	By Proposition \ref{lem_bandwidthconcentration}, one can apply the same argument leading to (\ref{partA}) above, to replace the bandwidth $h_n$ in $\bK_n$ by $\sfh$ defined in (\ref{sfh}) of the main paper, with an negligible error compared with the final result. In the following, we focus on the setting where the bandwidth is $\sfh$.  As before,  without loss of generality, we will consider the reduced data $\{\yb_i^0\}_{1\le i\le n}$ and their related quantities $\{\xb_i^0\}_{1\le i\le n}$ and $\{\zb_i^0\}_{1\le i\le n}$. For simplicity, we omit their superscripts and denote them as $\{\yb_i\}$, $\{\xb_i\}$ and $\{\zb_i\}$.

	We denote $\yb \in \mathbb{R}^p$ as 
	\begin{equation*}
	\yb=(\bm{y}, \bm{y}^\perp), \ \bm{y} \in \mathbb{R}^r \ \text{and} \ \bm{y}^\perp \in \mathbb{R}^{p-r},
	\end{equation*}
	and set $\zb_i=(\bm{z}_i, \bm{z}_i^\perp)^\top$ in a similar fashion. Consequently, by the fact that $\xb_i=(\bm{x}_i, 0, ..., 0)$, we have that
	\begin{equation}\label{eq_elementraydecomposition}
	\|\yb-\yb_i\|_2^2=\|\bm{y}-\bm{x}_i\|_2^2+\| \zb_i \|_2^2+\| \bm{y}^\perp \|_2^2-2(\bm{y}^\perp)^\top \bm{z}_i^\perp-2(\bm{y}-\bm{x}_i)^\top \bm{z}_i.
	\end{equation}
	This leads to 
	\begin{align*}
	\exp\left(-\frac{\| \yb-\yb_i \|_2^2}{\sfh} \right)=\exp\left(-\frac{\| \bm{y}-\bm{x}_i \|_2^2}{\sfh} \right)\cdot\ell,
	\end{align*}
	where $\ell$ is denoted as 
	\begin{equation*}
	\ell:=\exp\left(-\frac{\| \zb_i \|_2^2+\| \bm{y}^\perp \|_2^2-2(\bm{y}^\perp)^\top \bm{z}_i^\perp-2(\bm{y}-\bm{x}_i)^\top \bm{z}_i}{\sfh}\right). 
	\end{equation*}
	As we will examine the RKHS norm in $\mathcal{H}_K$, containing functions defined only on the support of $\tilde{\sfP}$, without loss of generality, we can ignore $\bm{y}^\perp$ (that is, set $\bm{y}^\perp=0$) and only focus on $\bm{y}.$  Since $\bm{y}$ is sub-Gaussian,  using the elementary estimate $\exp(-x) \sim 1-x$ as $x \to 0, $  we obtain
	\begin{equation*}
	\ell=1+{\OO_{\prec}\bigg(\frac{1}{(\sum_{i=1}^rn^{{\alpha_i-\beta}})^{1/2}}+\frac{1}{\sum_{i=1}^r n^{\alpha_i-\beta-\eta}}\bigg),}
	\end{equation*}
	uniformly in $\yb$ on the support of $\widetilde\sfP$.
	This immediately shows that 
	\begin{align}\label{eq_reduceddimensionbound}
	\exp\left(-\frac{\| \yb-\yb_i \|_2^2}{\sfh} \right) & =\exp\left(-\frac{\| \bm{y}-\bm{x}_i \|_2^2}{\sfh} \right) \nonumber \\
	&+{ \OO_{\prec}\bigg(\frac{1}{(\sum_{i=1}^r n^{{\alpha_i-\beta}})^{1/2}}+\frac{1}{\sum_{i=1}^r n^{\alpha_i-\beta-\eta}}\bigg). }
	\end{align}
	With the above preparation, we start analyzing the eigenfunctions. Recall the definition in (\ref{eigenvec}). In the first step, we study the following  term
	\begin{equation}\label{eq_expansion}
	\left(\widehat{\phi}_i^{(n)}(\yb)-\phi_i^{(n)}(\bm{y}) \right)^2=\left(\mathrm{E}_1+\mathrm{E}_2+\mathrm{E}_3 \right)^2, 
	\end{equation}
	where $\mathrm{E}_k ,k=1,2,3,$ are denoted as follows
	\begin{align*}
	& \mathrm{E}_1:=\left( \frac{1}{\lambda_i}-\frac{1}{\mu_i} \right) \frac{1}{\sqrt{n}} \sum_{j=1}^n \exp\left(-\frac{\| \yb-\yb_j \|_2^2}{\sfh} \right) v_{ij}, \\
	&  \mathrm{E}_2:= \frac{1}{\sqrt{n} \mu_i} \sum_{j=1}^n \left[ \exp\left(-\frac{\| \yb-\yb_j \|_2^2}{\sfh} \right)-\exp\left(-\frac{\| \bm{y}-\bm{x}_j \|_2^2}{\sfh} \right)  \right] v_{ij}, \\
	& \mathrm{E}_3:=\frac{1}{\sqrt{n}\mu_i} \sum_{j=1}^n \exp\left(-\frac{\| \bm{y}-\bm{x}_j \|_2^2}{\sfh} \right)  (v_{ij}-u_{ij}).
	\end{align*}
	First,  by Theorem \ref{thm_mainthm}, when $1 \leq i \leq \mathsf{K},$ we have that $\lambda_i\ge \gamma_i-|\lambda_i-\gamma_i|\ge c$ for sufficiently large $n$ with high probability, and thus
	\begin{equation*}
	\left| \frac{1}{\lambda_i}-\frac{1}{\mu_i} \right| \prec \frac{|\lambda_i-\mu_i|}{\lambda_i\mu_i} \prec { \frac{1}{\mu_i}\bigg(\frac{1}{(\sum_{i=1}^r n^{{\alpha_i-\beta}})^{1/2}}+\frac{1}{\sum_{i=1}^r n^{\alpha_i-\beta-\eta}}\bigg).}
	\end{equation*}
	Consequently, using the boundedness of the kernels and the fact that $\|\vb\|_1\le\sqrt{n}\|\vb\|_2= \sqrt{n}$ for any unit vector $\vb\in\R^n$, we conclude that for $i \leq \mathsf{K}$ {
	\begin{equation*}
	\mathrm{E}_1\prec { \frac{1}{\mu_i}\bigg(\frac{1}{(\sum_{i=1}^r n^{{\alpha_i-\beta}})^{1/2}}+\frac{1}{\sum_{i=1}^r n^{\alpha_i-\beta-\eta}}\bigg).}
	\end{equation*} }
	Second, by (\ref{eq_reduceddimensionbound}), we readily have that for $i \leq \mathsf{K}$ {
	\begin{equation*}
	\mathrm{E}_2\prec { \frac{1}{\mu_i}\bigg(\frac{1}{(\sum_{i=1}^rn^{{\alpha_i-\beta}})^{1/2}}+\frac{1}{\sum_{i=1}^rn^{\alpha_i-\beta-\eta}}\bigg).}
	\end{equation*} }
	Third, we control $\mathrm{E}_3.$ On one hand, for $i \leq \mathsf{K},$ by Lemma \ref{lem_mutiplepertubation},
	we have that $\mathrm{E}_3=\OO(1).$ One the other hand, by setting 
	\begin{equation*}
	\mathbf{a}=(a_1, \cdots, a_n)^\top, \ \mathbf{b}=(b_1, \cdots, b_n)^\top, 
	\end{equation*}
	where $a_j=\frac{1}{\sqrt{n}\mu_i} \exp\left(-\frac{\| \bm{y}-\bm{x}_j \|_2^2}{\sfh} \right) $ and $b_j=v_{ij}-u_{ij},$ we can rewrite 
	\begin{equation*}
	\mathrm{E}_3^2=\mathbf{b}^\top \mathbf{a} \mathbf{a}^\top \mathbf{b}.
	\end{equation*}
	Note that $\| \mathbf{b} \|_2^2=\| \vb_i-\ub_i \|_2^2.$ Then we have that 
	\begin{equation}\label{eq_eqeqe3control}
	\mathrm{E}_3^2 \leq  \| \mathbf{a} \|_2^2 \| \mathbf{b} \|_2^2 \prec { \frac{1}{\mu_i^2}\cdot \frac{1}{\mathsf{r}_i^2}\left(\frac{1}{\sum_{i=1}^rn^{\alpha_i-\beta-\eta}}+\frac{1}{(\sum_{i=1}^rn^{\alpha_i-\beta})^{1/2}}\right),}
	\end{equation} 
	where in the second step we used Part 2 of Theorem \ref{thm_mainthm}. Combining all the above bounds, using (\ref{eq_expansion}), when $i \leq \mathsf{K},$ we arrive at 
	\begin{align}\label{eq_boundone}
	& \mu_i\left(\widehat{\phi}_i^{(n)}(\yb)-\phi_i^{(n)}(\bm{y}) \right)^2  \notag \\
	& \prec {\bigg[\frac{1}{(\sum_{i=1}^rn^{{\alpha_i-\beta}})^{1/2}}+\frac{1}{\sum_{i=1}^rn^{\alpha_i-\beta-\eta}}\bigg]^2+\frac{1}{\mathsf{r}_i^2}\left(\frac{1}{\sum_{i=1}^rn^{\alpha_i-\beta-\eta}}+\frac{1}{(\sum_{i=1}^rn^{\alpha_i-\beta})^{1/2}}\right),}
	\end{align}
	for any $\yb$ on the support of $\widetilde\sfP$. 
	Finally, since by Part 1 of Theorem \ref{thm_mainthm},  we have for any $\yb\in\mathcal{S},$
	\begin{align}
	\sqrt{\lambda_i} \phi_i^{(n)}(\yb)-\sqrt{\mu_i} \phi_i^{(n)}(\yb)\le (\sqrt{\lambda_i}-\sqrt{\mu_i})\phi_i^{(n)}(\yb)\prec { \frac{1}{(\sum_{i=1}^rn^{{\alpha_i-\beta}})^{1/2}}+\frac{1}{\sum_{i=1}^rn^{\alpha_i-\beta-\eta}},}
	\end{align}
	and similarly
	\beq
	\sqrt{\lambda_i} \widehat\phi_i^{(n)}(\yb)-\sqrt{\mu_i} \widehat\phi_i^{(n)}(\yb)\prec { \frac{1}{(\sum_{i=1}^rn^{{\alpha_i-\beta}})^{1/2}}+\frac{1}{\sum_{i=1}^rn^{\alpha_i-\beta-\eta}}.}
	\eeq
 Combining the above bounds, we have that 
	\begin{align}
	& \sqrt{\lambda_i}\widehat{\phi}_i^{(n)}(\yb)  -\sqrt{\gamma_i}{\widetilde{\phi}_i}(\yb)\nonumber \\
	&\leq \sqrt{\lambda_i}\widehat{\phi}_i^{(n)}(\yb) -\sqrt{\lambda_i}\phi_i^{(n)} (\yb) +\sqrt{\lambda_i} \phi_i^{(n)}(\yb) -\sqrt{\mu_i} \phi_i^{(n)}(\yb) +\sqrt{\mu_i} \phi_i^{(n)}-\sqrt{\gamma}_i\widetilde{\phi}_i(\yb) \nonumber \\
	&\le |\sqrt{\mu_i}\widehat{\phi}_i^{(n)}(\yb) -\sqrt{\mu_i}\phi_i^{(n)}(\yb)| +|\sqrt{\lambda_i} \widehat\phi_i^{(n)}(\yb) -\sqrt{\mu_i} \widehat\phi_i^{(n)}(\yb)| +	2|\sqrt{\lambda_i} \phi_i^{(n)}-\sqrt{\mu_i} \phi_i^{(n)}(\yb) |\nonumber\\
	&\quad+|\sqrt{\mu_i} \phi_i^{(n)}(\yb) -\sqrt{\gamma}_i\widetilde{\phi}_i(\yb) |\nonumber \\
	&\prec { |\sqrt{\mu_i}\widehat{\phi}_i^{(n)}(\yb) -\sqrt{\mu_i}\phi_i^{(n)} (\yb) |+\frac{1}{(\sum_{i=1}^rn^{{\alpha_i-\beta}})^{1/2}}+\frac{1}{\sum_{i=1}^rn^{\alpha_i-\beta-\eta}} +\frac{1}{\sqrt{n}\sfr_i}, }\label{tri.ineq}
	\end{align} 
	where in the last inequality we used 
	\beq
|f(\yb) |=|\langle f, K(\yb,\cdot)\rangle_K|\le \|f\|_K \sqrt{K(\yb,\yb)}\le C\|f\|_K,	
	\eeq
	and Lemma \ref{lem_mutiplepertubation}.  Finally, plugging in (\ref{eq_boundone}), we obtain the final result as $i \leq \mathsf{K}$.
\end{proof}

{
\subsection{Guarantee of Bandwith Selection: Proof of Proposition \ref{lem_bandwidthconcentration}}

We first obtain upper and lower bounds for $\sfh$ defined in (\ref{sfh}). 
Recall from (\ref{eq_xijbound}) that 
\beq
\max_{1\le i, j\le n}\|\xb_i-\xb_j\|_2^2\prec \sum_{i=1}^r\theta_i.
\eeq
This implies the upper bound that 
\beq \label{b1}
\sfh \prec \sum_{i=1}^r\theta_i.
\eeq

For the lower bound, we show that, with high probability, for any sufficiently small constant $\epsilon>0$, there exist at most $\frac{n^2}{\log^\epsilon n}$ pairs of $(\xb_i,\xb_j)$ so that
\beq\label{eq_eijbound}
\|\xb_i-\xb_j\|_2^2 \le \frac{C}{\log^{4\epsilon}n}\sum_{i=1}^r\theta_i, 
\eeq
where $C>0$ is some universal constant. To see this, we define the events
\beq
E(i,j)=\bigg\{\|\xb_i-\xb_j\|_2^2\le \frac{C}{ \log^{4\epsilon}n}\sum_{i=1}^r\theta_i \bigg\},\ \ 1\le i \neq j\le n. \nonumber
\eeq
We need the following lemma, whose proof will be given at the end of this section.

\begin{lem} \label{event.lem}
	For any given $i\in\{1,2,...,n\}$,  with high probability, at most $\frac{n}{\log^{\epsilon} n}$ events in $\{ E(i,j): 1\le j\le n, j\ne i\}$ are true.
\end{lem}

By Lemma \ref{event.lem}, we can apply the union bound over all $1\le i\le n$, and conclude that, with high probability, at most $\frac{n^2}{\log^\epsilon n}$ pairs of $(\mathbf{x}_i, \mathbf{x}_j)$ satisfy (\ref{eq_eijbound}).  Now for any percentile parameter $\omega\in(0,1)$ that is independent of $n$, when $n$ is sufficiently large, we always have $n(n-1)\omega>\frac{n^2}{\log^{\epsilon} n}$, which means the signal bandwidth $\sfh$ selected according to (\ref{sfh}) satisfies
\[
\sfh\ge \frac{C}{\log^{4\epsilon}n}\sum_{i=1}^r\theta_i ,
\]
with high probability. This leads to the lower bound 
\beq \label{b2}
\sfh\succ \sum_{i=1}^r\theta_i.
\eeq

Next we control $|\sfh/h_n-1|$. 
Note that 
\begin{equation}\label{eq_originaldecomposition}
\|\yb_i-\yb_j\|_2^2=\|\xb_i-\xb_j\|_2^2+\|\zb_i-\zb_j\|_2^2-2(\xb_i-\xb_j)^\top(\zb_i-\zb_j).
\end{equation}
By Lemma \ref{lem_concentrationinequality} and Assumption \ref{assum_mainassumption}, we have that
\beq
\max_{1\le i,j\le n}\|\zb_i-\zb_j\|_2^2\prec{ \sigma^2 p\asymp n^{\beta+\eta}.} \nonumber
\eeq
By an argument similar to that leading to (\ref{cross.pd}) and a union bound, we have
\beq \label{cross.pd0}
\max_{1\le i,j\le n}|2(\xb_i-\xb_j)^\top(\zb_i-\zb_j)|\prec \sigma \bigg(\sum_{i=1}^r\theta_i\bigg)^{1/2}\asymp \bigg(\sum_{i=1}^r n^{\alpha_i+\beta}\bigg)^{1/2}.
\eeq
Combining the above arguments, we find that  for any $(i,j), i\ne j$ 
\begin{align}
\big|\|\yb_i-\yb_j\|_2^2-\|\xb_i-\xb_j\|_2^2\big|&=  \big|\|\zb_i-\zb_j\|_2^2-2(\xb_i-\xb_j)^\top(\zb_i-\zb_j)\big|\nonumber \\
&\prec { n^{\beta+\eta}+n^{\beta/2}\big(\sum_{i=1}^r\theta_i\big)^{1/2}. }\label{y.bnd}
\end{align}

Moreover, for any $k$ such that $k=\omega n$ for some $\omega \in(0,1).$ Without loss of generality, we assume that $k$ is an integer. Let $X_{(k)}$ be the $k$-th largest element in $\{\|\xb_i-\xb_j\|^2_2:i\ne j\}$, then by Lemma \ref{event.lem}, we have that 
\beq \label{Xk.bnd}
X_{(k)}\ge \frac{C}{\log^{4\epsilon}n}\sum_{i=1}^r\theta_i.
\eeq
Together with (\ref{y.bnd}) and (\ref{eq_originaldecomposition}), we find that there are at least $k$ elements in $\{\|\yb_i-\yb_j\|^2_2:i\ne j\}$  satisfying
\beq
\|\yb_i-\yb_j\|_2^2 \ge X_{(k)}-{ \OO_{\prec}\bigg(n^{\beta+\eta}+n^{\beta/2}\big(\sum_{i=1}^r\theta_i\big)^{1/2}\bigg).} \nonumber
\eeq
which follows from the fact that  the first $k$  elements $\{X_{(i)}:1\le i\le k\}$ are no smaller than $X_{(k)}$ and hence they do not satisfy (\ref{eq_eijbound}). As a result, the $k$-th largest element $ Y_{(k)}$ in $\{\|\yb_i-\yb_j\|^2_2:i\ne j\}$ satisfies
\beq \label{Yklbnd}
Y_{(k)}\ge X_{(k)}-{ \OO_{\prec}\bigg(n^{\beta+\eta}+n^{\beta/2}\big(\sum_{i=1}^r \theta_i\big)^{1/2}\bigg).}
\eeq
On the other hand, we claim that there are at least $n-k+1$ elements in $\{\|\yb_i-\yb_j\|_2:i\ne j\}$ satisfying
\[
\|\yb_i-\yb_j\|_2^2 \le X_{(k)}+{\OO_{\prec}\bigg(n^{\beta+\eta}+n^{\beta/2}\big(\sum_{i=1}^r\theta_i\big)^{1/2}\bigg).}
\]
This is because Lemma \ref{event.lem} and (\ref{y.bnd}) imply that the last $n-k+1$ elements $\{X_{(i)}:k\le i\le n\}$ are no greater than $X_{(k)}$. As a result, we also have
\beq \label{Ykubnd}
Y_{(k)}\le X_{(k)}+{ \OO_{\prec}\bigg(n^{\beta+\eta}+n^{\beta/2}\big(\sum_{i=1}^r\theta_i\big)^{1/2}\bigg).}
\eeq

Hence, using the definitions of $\sfh$ and $h_n,$ by (\ref{Yklbnd}), (\ref{Ykubnd}) and (\ref{Xk.bnd}), we have that 
\beq \label{h/hn}
\frac{\sfh}{h_n}=1+{ \OO_{\prec}\bigg(\frac{1}{\sum_{i=1}^rn^{\alpha_i-\beta-\eta}}+\frac{1}{(\sum_{i=1}^rn^{\alpha_i-\beta})^{1/2}} \bigg) \equiv 1+\OO_{\prec}(\psi_n).}
\eeq

Finally, we prove Lemma \ref{event.lem}.


\begin{proof}[\bf Proof of Lemma \ref{event.lem}.]
Define the indicator function $1_{E(i,j)}$ so that $1_{E(i,j)}=1$ if $E(i,j)$ is true and $1_{E(i,j)}=0$ otherwise. We shall prove that, for any fixed $i\in\{1,2,...,n\}$, we have
\beq \label{mar.p}
\mathbb{P}\bigg(\sum_{1\le j\le n, j\ne i}^{} 1_{E(i,j)}\le \frac{n}{\log^{\epsilon} n}\bigg)\ge 1-n^{-C},
\eeq
for any large constant $C>0$ and sufficiently large $n$. Note that
\begin{align*}
\mathbb{P}\bigg(\sum_{1\le j\le n, j\ne i}^{} 1_{E(i,j)}\le \frac{n}{\log^{\epsilon} n}\bigg)&=\int \mathbb{P}\bigg(\sum_{1\le j\le n, j\ne i}^{} 1_{E(i,j)}\le \frac{n}{\log^{\epsilon} n}\bigg| \xb_i\bigg)\mathrm{d} \mathbb{P}_{\xb_i}.
\end{align*}
It suffices to show that, for any $\xb_i\in \text{supp}(\widetilde{\sfP})$, we have
\beq \label{cond.p}
\mathbb{P}\bigg(\sum_{1\le j\le n, j\ne i}^{} 1_{E(i,j)}\le \frac{n}{\log^{\epsilon} n}\bigg|\xb_i\bigg)\ge 1-n^{-C}.
\eeq
Now since
conditional on $\xb_i$, the events $\{ E(i,j): 1\le j\le n, j\ne i\}$ are mutually independent, so that $\{1_{E(i,j)}: 1\le j\le n, j\ne i\}$ are $i.i.d.$ Bernoulli random variables, with mean
\beq
\mu_{\xb_i}=\mathbb{P}(E(i,j)|\xb_i)=\E [1_{E(i,j)}|\xb_i]. \nonumber
\eeq
Note that when 
\beq \label{pn}
\frac{C\log n}{n}\le  \mu_{\xb_i}\le \frac{1}{\log^{\epsilon} n},
\eeq
using Lemma \ref{cher.lem} with $t=en\mu_{\xb_i}$ leads to
\beq
\mathbb{P}\bigg( \sum_{1\le j\le n, j\ne i}1_{E(i,j)}\ge en\mu_{\xb_i}\bigg|\xb_i\bigg)\le  \exp(-n\mu_{\xb_i}),  \nonumber
\eeq
which implies (\ref{cond.p}) and therefore (\ref{mar.p}). Moreover, if
\beq 
0\le \mu_{\xb_i}\le \frac{C\log n}{n}, \nonumber
\eeq
using Lemma \ref{cher.lem} with $t=\frac{n}{\log^{\epsilon}n}$ leads to
\beq
\mathbb{P}\bigg( \sum_{1\le j\le n, j\ne i} 1_{E(i,j)}\ge \frac{n}{\log^\epsilon n}\bigg|\xb_i\bigg)\le  \exp(-n\mu_{\xb_i})( e\mu_{\xb_i} \log^\epsilon n)^{n/\log^\epsilon n}\le n^{-C}, \nonumber
\eeq
for any large constant $C>0$ and sufficiently large $n$. As a result, we only need to show that 
\beq \label{pn1}
\mu_{\xb_i}\le \frac{1}{\log^{\epsilon} n}.
\eeq
The proof of (\ref{pn1}) is separated into two cases as follows.

\vspace{3pt}
\noindent Case I: $r=\mathrm{o}(\log^{\epsilon}n).$ 
For each $j=\{1,2,...,n\}$, we have
\begin{align*}
\mathbb{P}\bigg(    \|\xb_i-\xb_j\|_2^2  \ge \frac{C}{ \log^{4\epsilon}n}\sum_{k=1}^r\theta_i\bigg|\xb_i\bigg)&\ge \mathbb{P}\bigg( \bigcap_{k=1}^r \bigg\{ (x_{ik}-x_{jk})^2\ge \frac{C\theta_k}{\log^{4\epsilon} n} \bigg|\xb_i \bigg\}\bigg)\\
&\ge 1-\mathbb{P}\bigg( \bigcup_{k=1}^r\bigg\{ (x_{ik}-x_{jk})^2< \frac{C\theta_k}{\log^{4\epsilon}n}  \bigg\}\bigg| \xb_i\bigg)\\
&\ge 1-\sum_{i=1}^r \mathbb{P}\bigg( (x_{ik}-x_{jk})^2< \frac{C\theta_k}{\log^{4\epsilon}n}  \bigg| \xb_i\bigg).
\end{align*}
Now under the assumption that, for $1\le k\le r$, the marginal probability density of $x_{jk}/\sqrt{\theta_k}$ is continuous, bounded, and strictly bounded away from zero on $\text{supp}(\widetilde{\sfP}_k)/\sqrt{\theta_k}$, where $\widetilde{\sfP}_k$ is the marginal distribution of $x_{jk}, j=1,2,...,n$, by the mean value theorem, we have
\beq \label{mvt}
\mathbb{P}\bigg(|x_{ik}-x_{jk}|\le \frac{\sqrt{\theta_k}}{\log^{2\epsilon}n}\bigg| \xb_i\bigg)\asymp \frac{1}{\log^{2\epsilon}n},\qquad k=1,2,...,r.
\eeq
Thus, we have
\beq
\mathbb{P} \bigg(    \|\xb_i-\xb_j\|_2^2  \ge \frac{C}{ \log^{4\epsilon}n}\sum_{i=1}^r\theta_i\bigg|\xb_i\bigg)\ge  1-\frac{Cr}{\log^{2\epsilon}n}. \nonumber
\eeq
So that (\ref{pn1}) holds whenever $r=\mathrm{o}(\log^{\epsilon}n)$.

\vspace{3pt}
\noindent Case II: $r\gtrsim \log^\epsilon n$. The argument depends  on the magnitude of the ratio $\sum_{i=1}^r\theta_i/\theta_1$. On the one hand, when $\sum_{i=1}^r\theta_i/\theta_1\ge c_1\log \log^\epsilon n$ for some large constant $c_1>0$, since $(x_{ik}-x_{jk})^2, 1 \leq k \leq n$ are i.i.d. sub-exponential random variables,  we can apply Lemma \ref{lem_subexponentialconcentration} to obtain
\begin{align*}
&\mathbb{P}\bigg(    \|\xb_i-\xb_j\|_2^2  \le \frac{C}{ \log^{4\epsilon}n}\sum_{k=1}^r \theta_i\bigg|\xb_i\bigg)\\
&\le 2\exp\bigg(-c\min\bigg\{ \frac{(\sum_{i=1}^r\theta_i)^2+\|\xb_i\|_2^4}{\sum_{i=1}^r \theta_i^2} ,\frac{\sum_{i=1}^r\theta_i+\|\xb_i\|_2^2}{\theta_1}  \bigg\}  \bigg)\\
&\le 2\exp\bigg(-c\min\bigg\{ \frac{(\sum_{i=1}^r\theta_i)^2}{\sum_{i=1}^r\theta_i^2} ,\frac{\sum_{i=1}^r\theta_i}{\theta_1}  \bigg\}  \bigg) \\
& \le 2\exp\bigg(-c\frac{\sum_{i=1}^r\theta_i}{\theta_1} \bigg)\le \log^{-\epsilon c_2}n,
\end{align*}
for some $c_2>1$ and sufficiently large $r$, where the third step follows from the H\"older's inequality $\sum_{i=1}^r\theta_i^2\le \theta_1\cdot\sum_{i=1}^r\theta_i$. This implies (\ref{pn1}) for sufficiently large $n$. On the other hand when $\sum_{i=1}^r\theta_i/\theta_1<c_1\log\log^\epsilon n$, we have
\begin{align*}
\mathbb{P}\bigg(    \|\xb_i-\xb_j\|_2^2 & \le \frac{C}{ \log^{4\epsilon}n}\sum_{k=1}^r \theta_i\bigg|\xb_i\bigg) \\
& \le \mathbb{P}\bigg(    \|\xb_i-\xb_j\|_2^2  \le \frac{C_2\log \log^\epsilon n}{ \log^{4\epsilon}n}\theta_1\bigg|\xb_i\bigg)\\
&\le \mathbb{P}\bigg( (x_{i1}-x_{j1})^2\le \frac{C_2\theta_1}{\log^{3\epsilon} n} \bigg|\xb_i \bigg)\lesssim \frac{1}{\log^{3\epsilon/2}n},
\end{align*}
where the last inequality follows from (\ref{mvt}). This again implies   (\ref{pn1}) for sufficiently large $n$.  This proves (\ref{pn1}) and therefore completes the proof of Lemma \ref{event.lem}.
\end{proof}

\subsection{Extension to General Kernel Functions: Proof of Theorem \ref{thm_mainthm2}}

Similar as in the proof of Theorems \ref{thm_mainthm} and \ref{cor_maincor}, without loss of generality, we consider the reduced random vectors $\{\yb_i^0\}_{1\le i\le n}$, $\{\xb_i^0\}_{1\le i\le n}$ and $\{\zb_i^0\}_{1\le i\le n}$, and denote them for simplicity as $\{\yb_i\}$, $\{\xb_i\}$ and $\{\zb_i\}$. Without loss of generality, we assume they are centered. Again, we  denote the kernel matrices for ${\yb_i}, {\xb_i},$ and ${\zb_i},$  with  bandwidth $\sfh$ as $\bK_y, \bK_x$ and $\bK_z,$ respectively. 

\subsubsection{Convergence to Noiseless Kernel Matrix} \label{b.4.1}

\begin{proof}[\bf Eigenvalue convergence.] 
	Firstly, we study $n^{-1}\|\bK_y-\bK_x\|$. { Our discussion can be separated into three cases as follows. 
	
\noindent {Case I:}	{ $\inf_{0<x<1}\nu(x)\ge \tau_0$. By (iii) of  Assumption \ref{ker.assu}, for any $1\le i,j\le n$ and $i\ne j$, if we denote $x= \frac{\|\xb_i-\xb_j\|_2}{\sfh^{1/2}}$ and $y=\frac{\|\yb_i-\yb_j\|_2}{\sfh^{1/2}}$,
	we have
	\begin{align}
	\bK_y(i,j) &=f(y)\le f(x)+C\{ x^{\nu(x)}+ y^{\nu(y)} \}\cdot|x-y|^{\tau_0} \label{eq_startingexample}\\
	&\le f(x)+C\{ x^{\nu(x)}+ y^{\nu(y)} \}\cdot\bigg| \frac{x^2-y^2}{x+y}\bigg|^{\tau_0} \nonumber \\
	&\le f(x)+C\{ x^{\nu(x)-\tau_0}+ y^{\nu(y)-\tau_0} \}\cdot|{x^2-y^2}|^{\tau_0} \nonumber \\
	&\le f(x)+\OO_{\prec}\bigg(\bigg| \frac{\|\yb_i-\yb_j\|_2^2-\|\xb_i-\xb_j\|_2^2}{\sfh}\bigg|^{\tau_0}\bigg), \nonumber
	\end{align}
where in the last step we used the assumption $\nu(x) \geq \tau_0$ for all $0<x<1$, and $\nu(x)<0$ for all $x>1$,} and the fact
\[
\|\yb_i-\yb_j\|_2\prec \sfh^{1/2}, \quad \|\xb_i-\xb_j\|_2\prec \sfh^{1/2}.
\]
This shows that 
\beq
|\bK_y(i,j) -\bK_x(i,j) |\prec  \bigg(\frac{\|\zb_i-\zb_j\|_2^2+2|(\xb_i-\xb_j)^\top(\zb_i-\zb_j)|}{\sfh}\bigg)^{\tau}. \nonumber
\eeq
By (\ref{Delta.bnd}), we have
\[
\frac{\|\zb_i-\zb_j\|_2^2}{\sfh}\prec {\frac{\sigma^2 p}{\sfh}\prec \frac{1}{\sum_{i=1}^rn^{\alpha_i-\beta-\eta}}.}
\]
This along with (\ref{cross.ineq}) yields
\[
\max_{1\le i,j\le n}|\bK_y(i,j) -\bK_x(i,j) |\prec {\bigg(\frac{1}{\sum_{i=1}^rn^{\alpha_i-\beta-\eta}}+\frac{1}{(\sum_{i=1}^rn^{{\alpha_i-\beta}})^{1/2}}\bigg)^{\tau_0}.}
\]
Using a discussion similar to (\ref{partA}) with Lemma \ref{lem_circle}, we find that 
\begin{align} 
	\|n^{-1}\bK_y -n^{-1}\bK_x \| & \le \max_{1\le i,j\le n}|\bK_y(i,j) -\bK_x(i,j) |\nonumber\\
	&\prec { \bigg(\frac{1}{\sum_{i=1}^rn^{\alpha_i-\beta-\eta}}+\frac{1}{(\sum_{i=1}^rn^{{\alpha_i-\beta}})^{1/2}}\bigg)^{\tau_0}}.\label{value.bnd}
\end{align}

\noindent Case II: { $0\le\inf_{0<x<1}\nu(x)<\tau_0$. In this case, we denote $\nu_{0}=\inf_{0<x<1}\nu(x)$. Similar to the discussion of (\ref{eq_startingexample}), we have
\begin{align*}
	\bK_y(i,j) &=f(y)\le f(x)+C\{x^{\nu(x)}+y^{\nu(y)} \}\cdot \frac{|x^2-y^2|^{\nu_{0}}}{|x+y|^{\nu_{0}}}|x-y|^{\tau_0-{\nu_{0}}}\\
	&\le f(x)+C\{x^{\nu(x)-\nu_{0}}+y^{\nu(y)-\nu_{0}} \}\cdot {|x^2-y^2|^{\nu_{0}}}|x-y|^{\tau_0-{\nu_{0}}}\\
	&\le f(x)+C\{x^{\nu(x)-\nu_{0}}+y^{\nu(y)-\nu_{0}} \}\cdot {|x^2-y^2|^{\frac{\tau_0+\nu_{0}}{2}}}\\
	&\le f\bigg(\frac{\|\xb_i-\xb_j\|_2}{\sfh^{1/2}} \bigg)+\OO_{\prec}\bigg(\bigg| \frac{\|\yb_i-\yb_j\|_2^2-\|\xb_i-\xb_j\|_2^2}{\sfh}\bigg|^{\frac{\tau_0+\nu_{0}}{2}}\bigg),
\end{align*}
where in the third step we used the inequality
\begin{equation*}
x-y \leq \sqrt{x^2-y^2},  \ \text{for all} \ x \geq y \geq 0.
\end{equation*}
and in the last step we used the assumption that $\nu(x)>\nu_{0}$ for $0<x<1$ and $\nu(x)<0$ for $x\ge 1$.} This leads to 
\begin{align} 
	\|n^{-1}\bK_y -n^{-1}\bK_x \| & \le \max_{1\le i,j\le n}|\bK_y(i,j) -\bK_x(i,j) |\nonumber\\
	&\prec{ \bigg(\frac{1}{\sum_{i=1}^rn^{\alpha_i-\beta-\eta}}+\frac{1}{(\sum_{i=1}^rn^{{\alpha_i-\beta}})^{1/2}}\bigg)^{\frac{\tau_0+\nu_{0}}{2}}}.\label{value.bnd2}
\end{align}
}

Next, we study $\|\bK_y-\bK_n\|$.
Note that for some $h^*$ between $h_n$ and $\sfh$, by (iii) of Assumption \ref{ker.assu}, it holds that
\begin{align*}
	|\bK_y(i,j)-\bK_n(i,j)| &\le L \bigg|\frac{\|\yb_i-\yb_j\|_2}{\sfh^{1/2}}- \frac{\|\yb_i-\yb_j\|_2}{h_n^{1/2}}\bigg|^{\tau_0} \\
	&\le \bigg|\sqrt{\frac{\sfh}{h_n}}-1\bigg|\cdot \bigg( \frac{\|\yb_i-\yb_j\|_2^2}{\sfh}\bigg)^{\tau_0/2} \\
	&	\prec { \frac{1}{\sum_{i=1}^rn^{\alpha_i-\beta-\eta}}+\frac{1}{(\sum_{i=1}^rn^{\alpha_i-\beta})^{1/2}},}
\end{align*}
where the last step follows from (\ref{y/h}) and Proposition \ref{lem_bandwidthconcentration}. Similar to the discussion of (\ref{partA}), we have
\begin{align}
	\frac{1}{n}\|\bK_y-\bK_n\| & \le \max_{i\ne j} |\bK_y(i,j)-\bK_n(i,j)| \nonumber \\
	&	\prec { \left(\frac{1}{\sum_{i=1}^rn^{\alpha_i-\beta-\eta}}+\frac{1}{(\sum_{i=1}^rn^{\alpha_i-\beta})^{1/2}}\right).} \label{Ky2}
\end{align}  
Combining the above results, we have proven (\ref{eq_eigs_v2}). 
\end{proof}

\begin{proof}[\bf Eigenvector convergence.]
		{
 Based on the integral representation  (\ref{eq_integralrepresentation}), for each $i$ satisfying (\ref{eq_generalribound}), we have 
	\begin{align}\label{eq_uvdifference2}
	\langle \ub_i, \vb_i  \rangle^2 & =\frac{1}{2 \pi \ri} \oint_{\Gamma_i} \mathbf{v}_i^\top  (z{\bf I}-n^{-1}\bK_x)^{-1}\mathbf{v}_i \dd z \nonumber \\
	&\quad +\frac{1}{2 \pi \ri} \oint_{\Gamma_i} \mathbf{v}_i^\top \left[ (z{\bf I}-n^{-1}\bK_y)^{-1}-(z{\bf I}-n^{-1}\bK_x)^{-1} \right] \mathbf{v}_i \dd z  \nonumber \\
		&\quad +\frac{1}{2 \pi \ri} \oint_{\Gamma_i} \mathbf{v}_i^\top \left[ (z{\bf I}-n^{-1}\bK_n)^{-1}-(z{\bf I}-n^{-1}\bK_y)^{-1} \right] \mathbf{v}_i \dd z  \nonumber \\
	&:= \mathsf{L}_1+\mathsf{L}_2+\mathsf{L}_3, 
	\end{align}}
	for some contour $\Gamma_i:=\mathbb{B}(\gamma_i,\frac{\mathsf{r}_i}{C}),$ for some large constant $C>0,$ where $\mathbb{B}(\gamma_i,\frac{\mathsf{r}_i}{C})$ is the disk centered at $\gamma_i$ with radius $\mathsf{r}_i/C$ and we recall $ \mathsf{r}_i:=\min\{\gamma_{i-1}-\gamma_i, \gamma_i-\gamma_{i+1}\}$.
	
	On the one hand, by
	\beq \label{eigen.Ks2}
	|\mu_i-\gamma_i|\prec \frac{1}{\sqrt{n}},
	\eeq
	and the assumption of (\ref{eq_generalribound}) that ${ \sfr_i\gg \big[\big(\frac{1}{\sum_{i=1}^rn^{\alpha_i-\beta-\eta}}+\frac{1}{(\sum_{i=1}^rn^{{\alpha_i-\beta}})^{1/2}}\big)^{\xi}+\frac{1}{\sqrt{n}}\big]}$,  when $n$ is sufficiently large, $\lambda_i(\bK_x)$ is the only simple pole of the resolvent $(z{\bf I}-n^{-1}{\bK}_x)^{-1}$ inside $\Gamma_i$. Using the spectral decomposition of $(z{\bf I}-n^{-1}\bK_x)^{-1}$ and residual theorem, we have that 
	\begin{equation}\label{eq_L1control2}
	\sfL_1=1. 
	\end{equation}
	On the other hand, 
	by the definition of $\Gamma_i,$  we have
	\begin{align*}
	&|\lambda_i(\bK_y)-\gamma_i|  \le |\lambda_i(\bK_y)-\mu_i|+|\gamma_i-\mu_i| \\
&	\prec { \bigg(\frac{1}{\sum_{i=1}^rn^{\alpha_i-\beta-\eta}}+\frac{1}{(\sum_{i=1}^rn^{{\alpha_i-\beta}})^{1/2}}\bigg)^{\xi}+ \frac{1}{\sqrt{n}},}
	\end{align*}
	so that under the same event, by (\ref{eq_assumption}), we have
	\begin{equation*}
	\inf_{z \in \Gamma_i} \min \{ |\lambda_i(\bK_y)-z|, |\lambda_{i-1}(\bK_y)-z|, |\lambda_{i+1}(\bK_y)-z| \} \asymp \mathsf{r}_i.
	\end{equation*}
	Moreover, by the definition of the resolvent, we have that
	\begin{align*}
	\sup_{z \in \Gamma_i} \| (z{\bf I}-n^{-1}\bK_y)^{-1} \| & \leq \sup_{z \in \Gamma_i} \left( \frac{1}{|\lambda_i(\bK_y)-z|}+\frac{1}{|\lambda_{i-1}(\bK_y)-z|}+\frac{1}{|\lambda_{i+1}(\bK_y)-z|} \right) \\
	&\prec \frac{1}{\mathsf{r}_i}.
	\end{align*}
	Similarly, by (\ref{eigen.Ks2}), we also have $	|\mu_i-\gamma_i|=o(\sfr_i)$ with high probability, 
	which implies
	\begin{align*}
	\sup_{z \in \Gamma_i} \| (z{\bf I}-n^{-1}\bK_x)^{-1} \| \prec  \frac{1}{\mathsf{r}_i}.
	\end{align*}
	Now recall the
	resolvent identity 
	\begin{equation*}
	(z{\bf I}-n^{-1}\bK_y)^{-1}-(z{\bf I}-n^{-1}\bK_x)^{-1}=(z{\bf I}-n^{-1}\bK_y)^{-1}\left[n^{-1} \bK_y-n^{-1} \bK_x\right] (z{\bf  I}-n^{-1}\bK_x)^{-1}.
	\end{equation*}
	Applying (\ref{value.bnd}), we have 
	\begin{equation*}
	\sfL_2 \prec { \frac{1}{\mathsf{r}_i^2}\left[\frac{1}{\sum_{i=1}^rn^{\alpha_i-\beta-\eta}}+\frac{1}{(\sum_{i=1}^rn^{{\alpha_i-\beta}})^{1/2}}\right]^{\xi}. }
	\end{equation*}
{Similarly, by (\ref{Ky2}), we also have
\begin{align*}
	\sup_{z \in \Gamma_i} \| (z{\bf I}-n^{-1}\bK_n)^{-1} \| \prec  \frac{1}{\mathsf{r}_i},
\end{align*}
and therefore
\[
\sfL_3 \prec {\frac{1}{\mathsf{r}_i^2}\left(\frac{1}{\sum_{i=1}^rn^{\alpha_i-\beta-\eta}}+\frac{1}{(\sum_{i=1}^rn^{\alpha_i-\beta})^{1/2}}\right).}
\]}
	Insert the above bounds and (\ref{eq_L1control2}) back into (\ref{eq_uvdifference2}), we immediately obtain that 
	\begin{align} \label{vec.ker}
	\left| \langle \ub_i, \vb_i \rangle^2-1 \right| \prec { \frac{1}{\mathsf{r}_i^2}\left[\frac{1}{\sum_{i=1}^rn^{\alpha_i-\beta-1}}+\frac{1}{(\sum_{i=1}^rn^{{\alpha_i-\beta}})^{1/2}}\right]^{\xi},}
	\end{align}
as $\xi\le 1$.
	This completes our proof of (\ref{eq_projectionbound2}).
	
\end{proof}

\subsubsection{Convergence to Population Integral Operator}

\begin{proof}[\bf Eigenvalue convergence.]  The proof of (\ref{v.conv2}) is the same as Theorem \ref{cor_maincor}.
\end{proof}

\begin{proof}[\bf Eigenfunction convergence.]
	Again, in light of Proposition \ref{lem_bandwidthconcentration} and by the same argument as in the beginning of Section \ref{bbb_reduced}, it suffices to prove the result under the setting where the bandwidth is chosen as (\ref{eq_choicestepone}). Recall (\ref{eq_elementraydecomposition}).
%
%
By an argument similar to Section \ref{b.4.1}, we have 
	\begin{align*}
	f\left(\frac{\| \yb-\yb_i \|_2}{\sfh^{1/2}} \right)\le f\left(\frac{\| \bm{y}-\bm{x}_i \|_2}{\sfh^{1/2}} \right)+{ \mathrm{O}_{\prec}\bigg(\left[\frac{1}{\sum_{i=1}^rn^{\alpha_i-\beta-\eta}}+\frac{1}{(\sum_{i=1}^rn^{{\alpha_i-\beta}})^{1/2}}\right]^{\xi}\bigg),}
	\end{align*}
	uniformly in $\yb$ on the support of $\widetilde\sfP$.
	Similar to (\ref{eq_expansion}), we decompose
	\begin{equation}\label{eq_expansion2}
	\left(\widehat{\phi}_i^{(n)}(\yb)-\phi_i^{(n)}(\bm{y}) \right)^2=\left(\mathrm{E}_1+\mathrm{E}_2+\mathrm{E}_3 \right)^2, 
	\end{equation}
	where
	\begin{align*}
	& \mathrm{E}_1:=\left( \frac{1}{\lambda_i}-\frac{1}{\mu_i} \right) \frac{1}{\sqrt{n}} \sum_{j=1}^n f\left(\frac{\| \yb-\yb_j \|_2}{\sfh^{1/2}} \right) v_{ij}, \\
	&  \mathrm{E}_2:= \frac{1}{\sqrt{n} \mu_i} \sum_{j=1}^n \left[ f\left(\frac{\| \yb-\yb_j \|_2}{\sfh^{1/2}} \right)-f\left(\frac{\| \bm{y}-\bm{x}_j \|_2}{\sfh^{1/2}} \right)  \right] v_{ij}, \\
	& \mathrm{E}_3:=\frac{1}{\sqrt{n}\mu_i} \sum_{j=1}^n f\left(\frac{\| \bm{y}-\bm{x}_j \|_2}{\sfh^{1/2}} \right)  (v_{ij}-u_{ij}).
	\end{align*}
	For $\mathrm{E}_1$ and $\mathrm{E}_2$, following the same argument as in proof of Theorem \ref{cor_maincor}, along with (ii) of Assumption \ref{ker.assu}, we have
	\[
	\mathrm{E}_1\prec { \frac{1}{\mu_i}\bigg(\frac{1}{\sum_{i=1}^rn^{\alpha_i-\beta-\eta}}+\frac{1}{(\sum_{i=1}^rn^{{\alpha_i-\beta}})^{1/2}}\bigg)^{\xi},}
	\]
	\[
	\mathrm{E}_2\prec { \frac{1}{\mu_i}\bigg(\frac{1}{\sum_{i=1}^rn^{\alpha_i-\beta-\eta}}+\frac{1}{(\sum_{i=1}^rn^{{\alpha_i-\beta}})^{1/2}}\bigg)^{\xi}.}
	\]
	For $\mathrm{E}_3$, we denote $
	\mathbf{a}=(a_1, \cdots, a_n)^\top$ and $\mathbf{b}=(b_1, \cdots, b_n)^\top$, 
	where $a_j=\frac{1}{\sqrt{n}\mu_i} f\left(\frac{\| \bm{y}-\bm{x}_j \|_2}{\sfh^{1/2}} \right) $ and $b_j=v_{ij}-u_{ij}.$ Similar to (\ref{eq_eqeqe3control}), we have that 
	\begin{equation*}
	\mathrm{E}_3^2 \leq  \| \mathbf{a} \|_2^2 \| \mathbf{b} \|_2^2 \prec { \frac{1}{\mu_i^2}\cdot \frac{1}{\mathsf{r}_i^2}\left[\frac{1}{\sum_{i=1}^rn^{\alpha_i-\beta-\eta}}+\frac{1}{(\sum_{i=1}^rn^{{\alpha_i-\beta}})^{1/2}}\right]^{\xi},}
	\end{equation*} 
	where in the second step we used (\ref{vec.ker}). Combining all the above bounds, using (\ref{eq_expansion2}), when $i \leq \mathsf{K},$ we arrive at 
	\begin{align}\label{eq_boundone2}
	\mu_i\left(\widehat{\phi}_i^{(n)}(\yb)-\phi_i^{(n)}(\bm{y}) \right)^2& \prec { \bigg[\frac{1}{\sum_{i=1}^rn^{\alpha_i-\beta-\eta}}+\frac{1}{(\sum_{i=1}^rn^{{\alpha_i-\beta}})^{1/2}}\bigg]^{2\xi}} \nonumber\\
	&\quad+ { \frac{1}{\mathsf{r}_i^2}\left[\frac{1}{\sum_{i=1}^rn^{\alpha_i-\beta-\eta}}+\frac{1}{(\sum_{i=1}^rn^{{\alpha_i-\beta}})^{1/2}}\right]^{\xi},}
	\end{align}
	for any $\yb$ on the support of $\widetilde\sfP$.
	Finally, by Part 1 of Theorem \ref{thm_mainthm2},  we have
	\begin{align*}
	|\sqrt{\lambda_i} \phi_i^{(n)}(\yb)-\sqrt{\mu_i} \phi_i^{(n)}(\yb)| &\le (\sqrt{\lambda_i}-\sqrt{\mu_i})|\phi_i^{(n)}(\yb)|\prec { \psi_n^{\xi},}
	\end{align*}
	and similarly
	\beq
	|\sqrt{\lambda_i} \widehat\phi_i^{(n)}(\yb)-\sqrt{\mu_i} \widehat\phi_i^{(n)}(\yb)|\le (\sqrt{\lambda_i}-\sqrt{\mu_i})|\widehat\phi_i^{(n)}(\yb)|\prec {\psi_n^{\xi}.}
	\eeq
	Then by a discussion similar to (\ref{tri.ineq}), we have
	\begin{align*}
	|\sqrt{\lambda_i}\widehat{\phi}_i^{(n)}(\yb)-\sqrt{\gamma_i}{\widetilde{\phi}_i}(\yb)| \prec  |\sqrt{\mu_i}\widehat{\phi}_i^{(n)}(\yb)-\sqrt{\mu_i}\phi_i^{(n)}(\yb) |+{ \psi_n^{\xi}}+\frac{1}{\sqrt{n}\sfr_i}.
	\end{align*} 
	Plugging the above bounds into (\ref{eq_boundone2}), we obtain the final result.
\end{proof}
}

\section{Some Further Discussions} \label{dis.sec}

\subsection{Extensions and Related Problems}

We propose a kernel-spectral embedding algorithm for learning low-dimensional nonlinear structure from noisy and high-dimensional datasets under a manifold setup. A key component is to develop an adaptive bandwidth selection procedure that is free from prior knowledge of the manifold. Our proposed method is theoretically justified and our embeddings correspond to an  integral operator from the RKHS associated to the underlying manifold. While our results pave the way towards statistical foundations for nonlinear dimension reduction, manifold learning, among others, there are various problems to be investigated further. We now highlight a few of them.    

First, {although this paper focuses on a general class of kernel functions for the construction of the low-dimensional embeddings, there are still many other interesting kernel functions, such as $f(x)=\frac{\sin x}{x}$ and $f(x)=\frac{1}{e^x+e^{-x}}$, that are excluded from our discussion. Moreover, the current paper only concerns the distance-type kernel matrices as in (\ref{Kmat}), leaving the inner-product kernel matrices with $K(i,j)=f(\yb_i^\top\yb_j/h_n)$ unaddressed. Nevertheless, we believe this is only an initial step towards  interpretable manifold learning under high-dimensional noisy data, and our results  stand as a vantage point for exploring and understanding other kernel-spectral embedding algorithms.}


Second,  in Part 2 of Theorem \ref{cor_maincor}, the convergence of eigenfunctions is obtained in a pointwise manner. In some applications,  a more precise local characterization of the  embeddings is needed. For example, one may need to study the eigenfunction convergence under the $L_{\infty}$ norm, defined by $\|f\|_\infty=\sup_{\xb}|f(\xb)|$. However, to our best knowledge, even in the noiseless setting, the $L_{\infty}$ convergence has only been studied for the normalized integral operators  \cite{luxburg2004convergence,10.2307/25464638} and the linear differential Laplace–Beltrami operator  \cite{calder2020lipschitz,DUNSON2021282,wormell2021spectral}. Less is known for the unnormalized integral operator considered in our paper.  Moreover, to understand the $L_{\infty}$ convergence under noisy datasets, in contrast to (\ref{eq_projectionbound}), we also need to study the $\ell_{\infty}$ convergence of the eigenvectors. The analysis requires a more sophisticated argument from  random matrix theory as in \cite{bao2022statistical,fan2018eigenvector}. We will pursue these directions in future works.

Additionally, the current work mainly concerns the super-critical {non-null regime $n^{\beta+\eta}=\oo(\sum_{i=1}^rn^{\alpha_i})$. When $n^{\beta+\eta}\asymp\sum_{i=1}^rn^{\alpha_i}$,} according to (2) of Theorem 3.1 of \cite{DW2}, since the signal part has a comparable  strength as the noise part, neither of them dominates the other in terms of the behavior of $n^{-1}\bK_n$. Similar results have also been established in \cite{el2010information}. In this critical regime, to better utilize the kernel matrix $n^{-1}\bK_n,$ one needs to develop  a high-dimensional nonlinear signal and noise separation procedure. 

Finally, in the current paper, we use an integral operator to learn the geometric structure of the underlying manifold. To enhance our understanding of the nonlinear structures,  it is of interest to explore the relationship between the eigenfunctions of our integral operator and those of the Laplace-Beltrami operator. 
The main challenge is that, in the high-dimensional noisy setting, the normalized graph Laplacian would no longer converge to the Laplace-Beltrami operator  \citep{DW2}, so that novel methodologies are needed to facilitate the comparison. Moreover, in the current paper, we focus on an RKHS-based method. It will be interesting to study the graph-based methods for high-dimensional noisy datasets under the current setting.

\subsection{Comparisons with Some Existing Works} \label{compare.sec}

{In Section \ref{comp.sec} of the our manuscript, we have briefly summarized the differences between the current work and those of \cite{hofmeyr2019improving,loffler2021optimality, abbe2020ell_p}. In what follows, we provide some more details.
	
First, in \cite{hofmeyr2019improving}, the author proposed a graph-cut based spectral clustering algorithm ("SPUDS"), in which a major step was to obtain the spectral embedding based on some graph Laplacian matrix.  Specifically, the method starts by constructing a Gaussian kernel matrix $\bK_n=(K(i,j))_{1\le i,j\le n}$ using the following bandwidth 
\begin{equation} \label{h1}
	h_n=sn^{-1/(r+4)},
\end{equation}
 where $s$ is some user-specified value which is bounded both above and away from zero, and then obtains the eigendecomposition of the graph Laplacian matrix ${\bf L}_n={\bf I}-\bD_n^{-1/2} \widehat{\bK}_n\bD_n^{-1/2}$, where $\bD=\text{diag}(\sum_{j\ne 1}K(1,j),...,\sum_{j\ne n}K(n,j))$, and $\widehat{\bK}_n=({K}(i,j)\cdot 1_{\{i\ne j\}})_{1\le i,j\le n}$ is the zero-diagonal kernel matrix. For some given initial value for the number of clusters, K-means is then used along with the spectral embedding (i.e., the eigenvectors of ${\bf L}_n$) to determine the cluster labels and update the number of clusters iteratively using the normalized cut.  In general, both SPUDS and our proposed method rely on some kernel spectral embeddings based on some kernel matrices, and require careful selection of the bandwidths. However, there are many significant differences between the two methods. First, \cite{hofmeyr2019improving} focuses on the noiseless data sets when clusters are closely related to the modes of the density functions whereas we consider high-dimensional noisy data sets.  Second, the two methods focus on the spectrum of two very different matrices: our method relies on the eigendecomposition of the kernel matrix itself, whereas SPUDS relies on the graph Laplacian matrix of the zero-diagonal kernel matrix. Third, the bandwidth selection schemes are also very different. On the one hand, {in \cite{hofmeyr2019improving},  $h_n$ in (\ref{h1}) is chosen to ensure the consistency of the normalized cut (see Theorem 2 therein) rather than the kernel matrices or graph Laplacian matrices. In fact, it is not clear from \cite{hofmeyr2019improving} that whether such a choice of bandwidth ensures consistency and  meaningful convergence of the kernel matrices or graph Laplacian matrices. In contrast, our selected bandwidth is capable of capturing the underlying signals and therefore guarantees the convergence of the kernel matrices to some population integral operator.} On the other hand, {our bandwidth selection procedure is data-driven and does not require any prior knowledge.  In contrast, (\ref{h1}) relies on the values of both the (ambient) dimension $r$ which is in general unknown and some constant $s.$ } To choose $s,$
for practical purpose, in Section 4 of \cite{hofmeyr2019improving}, the author suggest
\begin{equation} \label{h2}
h_n=1.2\sqrt{\bar{\lambda}_{r}} \bigg( \frac{4}{(2+r)n} \bigg)^{\frac{1}{r+4}},
\end{equation}
 where $\bar{\lambda}_{r}$ is the average of the largest $r$ eigenvalues of the covariance matrix of the data. However, in our setting even when $r$ is fixed, $\bar{\lambda}_r$ usually diverges which violates the boundedness requirement of $s$ in (\ref{h1}).  Numerically, in Section \ref{cluster.sec} and Figure \ref{cls.fig1} of the main manuscript, compared to the method in \cite{hofmeyr2019improving}, we show that our proposed method has better performance in terms of spectral clustering using two high-dimensional noisy data sets.

Second, in \cite{loffler2021optimality}, the authors studied the spectral embedding and clustering of high-dimensional data under the Gaussian mixture model. The spectral embedding $\hat{\bf Y}\in \mathbb{R}^{n\times k}$ was obtained based on the spectral decomposition 
\[
{\bf Y} = \sum_{i=1}^{\min\{n,p\}}\hat\lambda_i\hat u_i\hat v_i^\top\in \mathbb{R}^{n\times p},\qquad \hat\lambda_1\ge ...\ge \hat\lambda_{\min\{p,n\}},
\]
of the observed data matrix ${\bf Y}$, rather than some kernel matrices. Consequently, they do not need to select a bandwidth. For clustering, they used $\hat{\bf Y}=\hat{\bf U}\hat{\bf \Lambda}$, where $\hat{\bf \Lambda}=\text{diag}(\hat\lambda_1,...\hat\lambda_k)$ and $\hat{\bf U}=(\hat u_1,...,\hat u_k)$, with $k$ being the target embedding dimension. K-means was then used on $\hat{\bf Y}$ to determine the clusters. Compared with Step 3(ii) of our Algorithm \ref{al0}, the embedding $\hat{\bf Y}$ shared a similar form as our proposal. However, the linear embedding  $\hat{\bf Y}$ may not be as effective as our method in dealing with data sets with nonlinear structures. For illustrations,  in Section \ref{cluster.sec} of the main manuscript, we compared the numerical performance of these methods under both a Gaussian mixture model and a nonlinear nested sphere model. Figure \ref{cls.fig1} of the main manuscript shows that, while the method of \cite{loffler2021optimality} performed equally well as ours under the Gaussian mixture model, its performance was significantly worse than our method under the nonlinear nested sphere model. 

Finally, in \cite{abbe2020ell_p}, the authors developed a perturbation theory for a hollowed version of PCA in general Hilbert spaces. Specifically, the model
\begin{equation} \label{model2}
	\yb'_i = {\bf x}'_i+{\bf z}'_i\in \mathcal{H},\qquad 1\le i\le n,
\end{equation}
was considered for some reproducing kernel Hilbert space (RKHS) $\mathcal{H}$ associated with the map $\phi:\mathcal{X}\to \mathcal{H}$ via the kernel $K(\cdot,\cdot): \mathcal{X}\times \mathcal{X}\to \mathbb{R}$ with $K(\xb,\yb)=\langle \phi(\xb),\phi(\yb)\rangle$. In (\ref{model2}),  $\{\xb'_i\}_{1\le i\le n}$ are the noiseless signals, and  $\{{\bf z}'_i\}_{1\le i\le n}$ are sub-Gaussian noise. The general idea, like our paper, was to learn the geometric information of some underlying structure contained in the eigenfunctions of some integral operator. The focus of \cite{abbe2020ell_p} was on the perturbation analysis of the eigenspaces associated to the noiseless infinite-dimensional self-adjoint operator, or Gram matrix ${\bf H}_n^*={\bf X'X'}^\top$, where
$$
{\bf X}' = \begin{bmatrix}
	{\bf x}'_1\\
	{\bf x}'_2\\
	\vdots\\
	{\bf x}'_n
\end{bmatrix},
$$
such that ${\bf X}'({\bf g})= (\langle {\bf x}'_1,{\bf g}\rangle, ...\langle {\bf x}'_n,{\bf g}\rangle).$ Furthermore, for sub-Gaussian and Gaussian mixture models, the spectral embedding $\hat{\bf Y}'={\bf U}'{\bf \Lambda'}^{1/2}$ was proposed, where ${\bf \Lambda'}=\text{diag}(\lambda'_1,...,\lambda'_k)$ and ${\bf U}'=({\bf u}_1, ..., {\bf u}_k)$ contain the $k$ leading eigenvalues and eigenvectors of the zero-diagonal Gram matrix ${\bf H}_n=\mathsf{H}({\bf Y}'{\bf Y}'^\top)$, with the operator $\mathsf{H}(\cdot)$ zeroing out all diagonal entries of a square matrix. 

 In general, we agree that \cite{abbe2020ell_p} is similar to our paper in the sense that both works consider spectral embedding of noisy data and study the convergence of some eigenvectors based on the observed samples to their counterparts associated with the noiseless samples. In what follows, we compare and contrast the two papers in the following aspects. First,  the noiseless data of \cite{abbe2020ell_p} is assumed to be generated from some RKHS as in (\ref{model2}). In contrast, our noiseless data is sampled from a nonlinear manifold model as in (\ref{eq_basicmodel}) and Assumption \ref{assum_generalmodelassumption}. Second, \cite{abbe2020ell_p} focused on the spectrum of the Gram matrices ${\bf H}_n^*$ and ${\bf H}_n$ defined above, whereas our paper focused on the distance kernel matrices $\bK_n^*$ and $\bK_n$ as in (\ref{Kmat}) and (\ref{Kn*}). The main difference between the two types of matrices can be seen as follows. On the one hand, construction of Gram matrices relies on specifying the map $\phi$, whereas our kernel matrices rely on the choice of kernel functions and a careful selection of the bandwidth. On the other hand, a key idea of \cite{abbe2020ell_p} was to zero out the diagonal entries of the Gram matrix ${\bf Y}'{\bf Y}'^\top$ as in \cite{kg2000,ElKaroui_Wu:2016b,ms2018}, which led to better spectral convergence rates. However, for the kernel matrices considered in our manuscript, since the diagonal part is isotropic (i.e., $f(0)\mathbf{I}$ for some given kernel function $f(\cdot)$), zeroing out the diagonals does not change the eigenvectors, nor the convergence results.  Third, in terms of theoretical assumptions, on the one hand, the assumptions of \cite{abbe2020ell_p} on the signal-to-noise ratio are comparable to ours when  $\mathcal{H}=\mathbb{R}^p$. For example, under our model (Assumption \ref{assum_mainassumption}) with finite $r$, { $\eta=1$ and $\beta=0$}, the assumptions of \cite{abbe2020ell_p} (cf. Assumptions 2.4 and 2.6 therein) imply that $\alpha>1$, which is the same as our paper. On the other hand, for general nonlinear space $\mathcal{H}$, since the assumptions of \cite{abbe2020ell_p}
are made on the Gram matrix ${\bf H}_n^*$, they can be relatively less interpretable compared to ours, which are made directly on the underlying manifold (cf. Assumptions \ref{assum_generalmodelassumption} and \ref{assum_mainassumption}) and the kernel functions (cf. Assumption \ref{ker.assu}). Finally, in terms of theoretical results, both papers showed spectral convergence of some matrices to their noiseless counterparts. However, the eigenvector perturbation bounds were obtained under different discrepancy measures (cf. Theorem 2.1 of \cite{abbe2020ell_p} and our Theorems \ref{thm_mainthm} and \ref{thm_mainthm2}). In particular, our results highlighted the impact of signal-to-noise ratio,  the underlying manifold structures and the choice of kernel functions on the final rate of convergence, and, by connecting to the population integral operator, yielded the proper interpretation of the eigenvectors with respect to the underlying manifold structures. Such results have not been established in \cite{abbe2020ell_p}.  

In addition, in terms of the proof technique,  \cite{abbe2020ell_p} relies on the leave-one-out analysis, and uses the Davis-Kahan theorem for proving the eigenvector perturbation bounds, whereas our proof of the eigenvector convergence relies on the integral representation of the eigenvectors along with the resovlent expansion. Since bandwidth plays an important role in our algorithm, we also need to establish a novel concentration inequality (cf. Proposition \ref{lem_bandwidthconcentration}) for our bandwidth selection scheme.  Numerically, since for given kernel matrices, the proposal of \cite{abbe2020ell_p} is equivalent to Step 3 of our algorithm, and the main advantage of our method lies in the construction of kernel matrices and bandwidth selection. For illustrations, in Section \ref{cluster.sec}, we compared our method with \cite{abbe2020ell_p} using two mixture models. 
The result is summarized in Figure \ref{cls.fig1} of the manuscript which demonstrates the advantage of our proposal especially when data sets are highly nonlinear.
}

{ \subsection{Some additional remarks}\label{suppl_additionalremark} In this subsection, we provide a few more remarks.

First, we make a comparison with Davis-Kahan theorem in terms of  the results in (\ref{eq_projectionbound}). In the literature, there exist different variants of Davis-Kahan theorem. For definiteness and easy comparison, we focus on the one from the popular monograph \cite{vershynin2018high} (see Theorem 4.5.5 therein);  that is
\begin{thm}[Davis Kahan Theorem] Let $\mathbf{S}$ and $\mathbf{T}$ be two symmetric matrices with the same dimension. For each fixed $i,$ denote
\begin{equation}\label{eq_deltadefinition}
\delta:=\min_{j: j \neq i} |\lambda_i(\mathbf{S})-\lambda_j(\mathbf{S})|.
\end{equation} 
Assuming that $\delta \geq \tau>0$ for some constant $\tau>0,$ we have that
\begin{equation}\label{eq_daviskahanresult}
|\langle \mathbf{u}_i, \mathbf{v}_i \rangle^2-1| \leq  \frac{4 \| \mathbf{S}-\mathbf{T} \|^2}{\delta^2},
\end{equation}
where $\mathbf{u}_i, \mathbf{v}_i$ are respectively the eigenvectors of $\mathbf{S}$ and $\mathbf{T}.$
\end{thm}
We see that in order to correctly apply Davis-Kahan and obtain (\ref{eq_daviskahanresult}), we need that $\delta$ is bounded from below by some constant. In contrast, in our equation (\ref{eq_projectionbound}), our bound reads $\psi_n/\mathsf{r}_i^2.$ 

We now make a comparison for both the denominators and numerators of the bounds. For the $\mathsf{r}_i^2$ part, on the one hand, our $\mathsf{r}_i$ only depends on three population eigenvalues, i.e., $\mathsf{r}_i=\min\{\gamma_i-\gamma_{i+1}, \gamma_{i-1}-\gamma_i\}$. In fact, $\delta$ used in Davis-Kahan theorem, i.e., (\ref{eq_deltadefinition}), requires much more information than our $\mathsf{r}_i$ (even though this condition may be relaxed using the arguments in \cite{yu2015useful}). On the other hand and more importantly, {our $\mathsf{r}_i$ is allowed to decay to zero} once equation (\ref{eq_assumption}) is satisfied. For the $\psi_n$ part, it is weaker than the results $\psi_n^2$ from Davis-Kahan which relies on stronger assumption on $\mathsf{r}_i.$ Nevertheless, if $\delta$ is bounded from below by a constant, we can improve our results using Davis-Kahan. To summarize, Davis-Kahan may result in faster convergence rates but at the cost of much stronger assumptions on the eigenvalues separation whereas our results are obtained under general and weaker separation conditions. In the current paper, for the purpose of generality and practical usability, we keep the current results in (29).

Second, it is true that the left-hand side of (\ref{eq_assumption}) is large than $\psi_n.$ Therefore, (\ref{eq_assumption}) means that $\psi_n$ is much smaller than $\mathsf{r}_i^2$  so that (\ref{eq_deltadefinition}) implies consistent estimation. Sorry for the confusion.  Finally, the reason that the left-hand side of (\ref{eq_assumption}) is larger than the ideal term $\psi_n$ is technical. Our proof relies on an integral representation via the resolvent, such an assumption is needed to make sure that we can find a contour which only contains $\lambda_i$ so that the residual theorem can be applied; see the discussions between equations (\ref{eq_uvdifference}) and (\ref{eq_largesmallbound}).

}

\section{Some Results on Gaussian Integral Operators}\label{sec_someresultintegral}
In this section, we collect some results on integral operators. Consider a special case when $\{\bm{x}_i\}$ are Gaussian. In such a setting, the eigenvalues and eigenfunctions of the integral operator can be calculated explicitly. We first state the results when $r=1.$ Let $\sfp(x)$ be the density function of $\mathcal{N}(0, \sigma^2)$ random variable. For the ease of statement, we write the Gaussian kernel function that 
$
k(x,y)=e^{-\frac{\|x-y \|_2^2}{2h}},
$
for some fixed bandwidth $h.$ Consequently, we can construct the following integral operator 
\begin{equation*}
\mathcal{K}^{\sigma}_1 f(x)=\int k(x,y) f(y) \sfp(y) \dd y. 
\end{equation*}
Denote $\{\gamma_i\}$ and $\{\phi_i(x)\}$ as the sequence of the eigenvalues and eigenfunctions of $\mathcal{K}^{\sigma}_1$ with
\begin{equation*}
\mathcal{K}_1^{\sigma} \phi_i(x)=\gamma_i \phi_i(x), \qquad \int \phi_i^2(x) \sfp(x) \dd x=1.  
\end{equation*}
Note that the eigenfunctions are in fact defined in $\mathbb{R}$ via  the relation that
\begin{equation*}
\phi_i(x)=\frac{1}{\gamma_i} \int k(x,y) \phi_i(y) \sfp(y) \dd y, \ \text{if} \ \lambda_i>0.
\end{equation*}
The following lemma provides the exact form for the eigenvalues and eigenfunctions.
\begin{lem} Denote $\beta=2 \sigma^2/h.$ The eigenvalues and eigenfunctions of $\mathcal{K}_1^{\sigma}$ are 
	\begin{equation}\label{eq_eigenvalueform}
	\gamma_i=\frac{\sqrt{2}}{\sqrt{1+\beta+\sqrt{1+2\beta}}} \left( \frac{\beta}{1+\beta+\sqrt{1+2 \beta}}\right)^i,
	\end{equation} 
	\begin{equation}\label{eq_eigenfunction}
	\phi_i(x)=\frac{(1+2 \beta)^{1/8}}{\sqrt{2^i i!}} \exp \left(-\frac{x^2}{2 \sigma^2} \frac{\sqrt{1+2 \beta}-1}{2} \right)H_i\left( \left(\frac{1}{4}+\frac{\beta}{2} \right)^{1/4} \frac{x}{\sigma} \right),
	\end{equation}
	where $H_i(\cdot)$ is the $i$th order Hermite polynomial. 
\end{lem}
\begin{proof}
	See Proposition 1 of \cite{AOSko}. 
\end{proof}
The above results can also be generalized to the $r$-dimensional setting when $\sfp(\bm{x})$ is the density function of $\mathcal{N}(\mathbf{0}, \Sigma)$ where $\Sigma$ is a diagonal matrix so that 
\begin{equation*}
\Sigma=\operatorname{diag}\{\theta_1, \cdots, \theta_{r}\},\qquad\text{where $\theta_i=\sigma_i^2$}
\end{equation*}
Due to invariance of the kernel, the associated integral operator, denoted as $\mathcal{K}_r,$ has a direct sum decomposition  
$
\mathcal{K}_{r}=\bigoplus_{i=1}^r \mathcal{K}_1^{\sigma_i}. 
$
Analogously, the eigenvalues and eigenfunctions of $\mathcal{K}_r$ can also be calculated explicitly  as follows. Recall that for given two operators $\mathcal{F}_1$ and $\mathcal{F}_2,$ the spectrum of their direct sum $\mathcal{F}_1 \bigoplus \mathcal{F}_2$ consists of pairwise products $\gamma_i(\mathcal{F}_1) \gamma_j(\mathcal{F}_2).$ Moreover, the eigenfunction of the product can be written into a multindex form so that corresponding to $\gamma_i(\mathcal{F}_1) \gamma_j(\mathcal{F}_2)$, we have
\begin{equation*}
\phi_{[i,j]}(x_1, x_2)=\phi_i(x_1) \psi_j(x_2),
\end{equation*} 
where $\{\phi_i(x)\}$ and $\{\psi_j(x)\}$ are the eigenfunctions of the operators $\mathcal{F}_1$ and $\mathcal{F}_2,$ respectively. Consequently, we can obtain the following results.
\begin{lem}\label{lem_multiplecase}
	Denote $\beta_i=2 \sigma^2_i/h.$ Then the eigenvalues and eigenfunctions of $\mathcal{K}_{r}$ can be represented using the multindex over all components, i.e., 
	$$
	\gamma_{[i_1, i_2, \cdots, i_r]}=\prod_{j=1}^{r} \gamma_{i_j}(\mathcal{K}_1^{\sigma_j}),
	$$
	and for $\bm{x}=(x_1, x_2, \cdots, x_r)$, we have
	$$
	\phi_{[i_1, i_2, \cdots, i_r]}(\bm{x})=\prod_{j=1}^r \phi_{i_j}(\mathcal{K}_{1}^{\sigma_j})(x_j). 
$$
\end{lem}
\begin{proof}
	See Section 3 of \cite{ICMLko}. 
\end{proof}

\section{Tuning Parameter Selection and Additional Numerical Results}\label{sec_choiceofomega}

\subsection{A Resampling Method for Selecting Percentile $\omega$ in (\ref{eq_bandwidthselection})}\label{choiceomegereampls}
{As shown by our theory in Section \ref{theory.sec},  under our assumptions, $\omega$ can be chosen as any constant between 0 and 1 to have the final embeddings that achieve the same asymptotic behavior. Especially, our proof of Proposition \ref{lem_bandwidthconcentration} depends on the aspect ratio (\ref{eq_aspect}) fundamentally. In practice, to optimize the empirical performance and improve automation of the method, we recommend using a resampling approach as follows. }

A resampling-based algorithm for determining the percentile parameter was proposed in Algorithm 1 of \cite{DW2}, which has also been used in \cite{ding2021spiked1,ding2021spiked}.  The method provides a choice of $\mathsf{s}$ using resampling method to distinguish the larger outlier eigenvalues and bulk eigenvalues of the kernel matrix $\bK_n$, where the outlier eigenvalues stand for the signal parts and bulk eigenvalues are for the noise part. A key observation underlying this approach is that the bulk eigenvalues are close to each other so that the ratios of two consecutive eigenvalues are close to one (see Remark 2.9 of \cite{DW2} for more details). Given the choice of $\mathsf{s},$ the algorithm is summarized below for self-completeness. \\

\begin{enumerate}
	\item For a pre-selected sequence of percentiles$\{\omega\}_{i=1}^T,$ calculated the associated bandwidths according to (\ref{eq_bandwidthselection}), denoted as $\{h_i\}_{i=1}^T.$ 
	
	\item For each $1 \leq i \leq T,$ calculate the eigenvalues of  $\bK_{n,i}$, constructed using the bandwidth $h_i.$ Denote the eigenvalues of $\bK_{n,i}$ in the decreasing order as $\{\lambda_{k}^{(i)}\}_{k=1}^n.$
	
	\item For some $\mathsf{s}>0,$ denote
	\begin{equation*}
	\mathsf{k}(\omega_i):=\max_{1 \leq k \leq  n-1} \left\{ k \bigg \vert \frac{\lambda_k^{(i)}}{\lambda_{k+1}^{(i)}} \geq 1+\mathsf{s}  \right\}.
	\end{equation*}
	Choose the percentile $\omega$ such that 
	\begin{equation}\label{omega_choice}
	\omega=\max_{i}[\argmax_{\omega_i} \mathsf{k}(\omega_i)]. 
	\end{equation}
\end{enumerate}
In particular, the last step can be replaced by other criteria according to users' purpose. Here we choose the largest one for the purpose of robustness. 
To demonstrate the empirical performance of the proposed percentile selection  algorithm, we consider the four examples  in Section \ref{simu.sec}, and evaluate the empirical percentile chosen by (\ref{omega_choice}), across a variety of  sample sizes. In each setting, we used $\textsf{s}=1$. 
Figure \ref{CV.fig2} shows the empirical percentile $\omega$ selected by the resampling method in each example. Along with Figures \ref{rate.fig} and \ref{ratio.fig} below, these results indicate the usefulness and robustness of the percentile selection algorithm.

\begin{figure}
	\centering
	\includegraphics[angle=0,width=10cm]{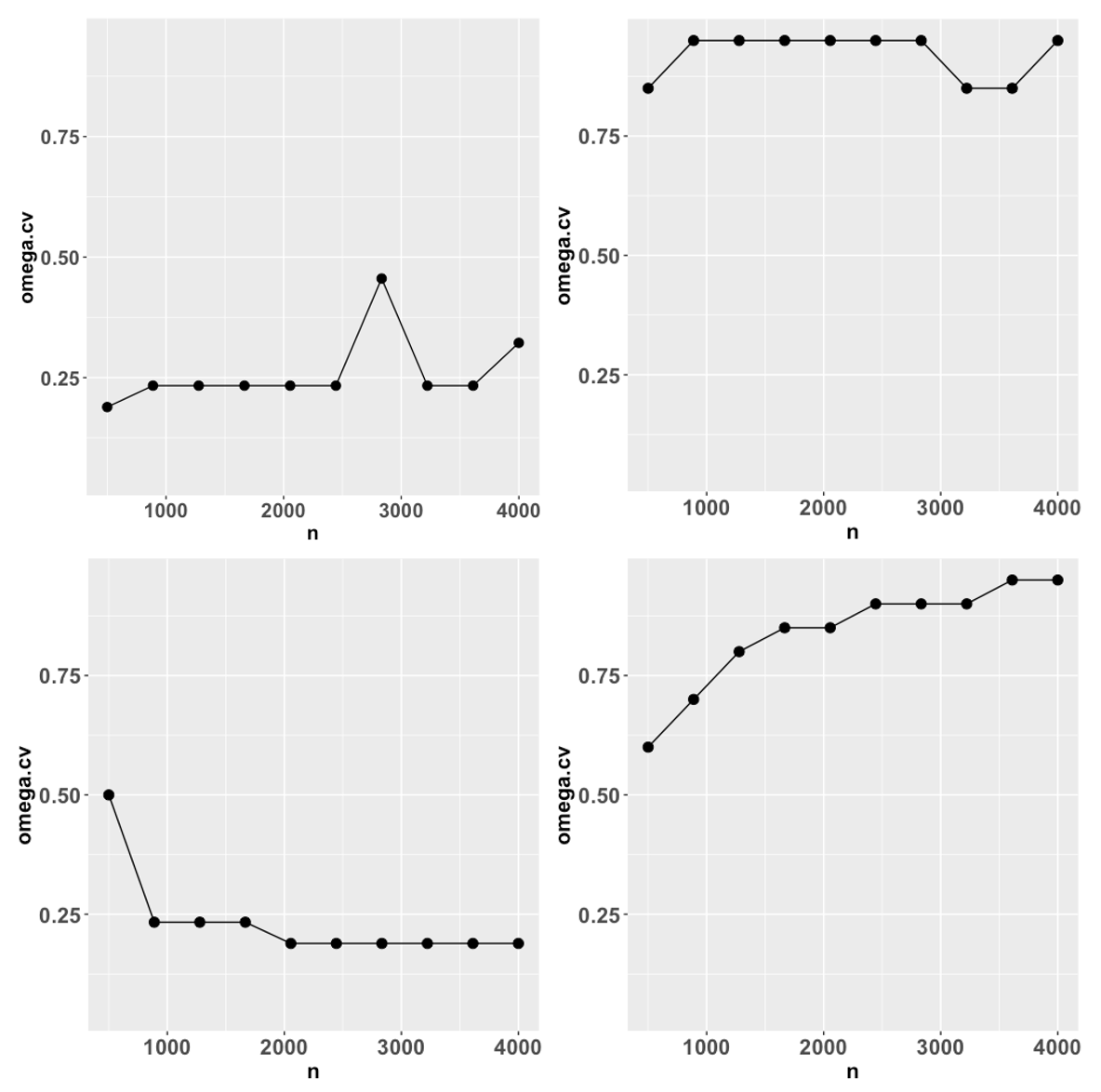}
	\caption{Plots of $\omega$ selected by the resampling method (\ref{omega_choice}). Top row from left to right:  "smiley face" and "mammoth"; Bottom row from left to right: "Cassini oval" and "torus."} 
	\label{CV.fig2}
\end{figure} 

%

{
\subsection{Additional Results from Simulation Studies} \label{simu.supp.sec}

In Figure \ref{man.fig}, we illustrate the manifold structures considered in the simulation studies. In Figure \ref{rate.fig}, we show the convergence of kernel matrices. 
In Figure \ref{bw.fig1}, we show the usefulness of our proposed population bandwidth. In Figure \ref{ratio.fig}, we compare the rates of convergence under different kernel functions. In particular, the top row of Figure \ref{ratio.fig} shows that the convergence rate under the Gaussian kernel is faster than that under the Laplacian kernel, whereas the bottom row of Figure \ref{ratio.fig} shows that convergence rate under the Gaussian kernel is similar to that under the quadratic polynomial kernel. These empirical results agree with our theory in Section \ref{ker.alt.sec}.

\begin{figure}[bt]
	\centering
	\includegraphics[angle=0,width=14cm]{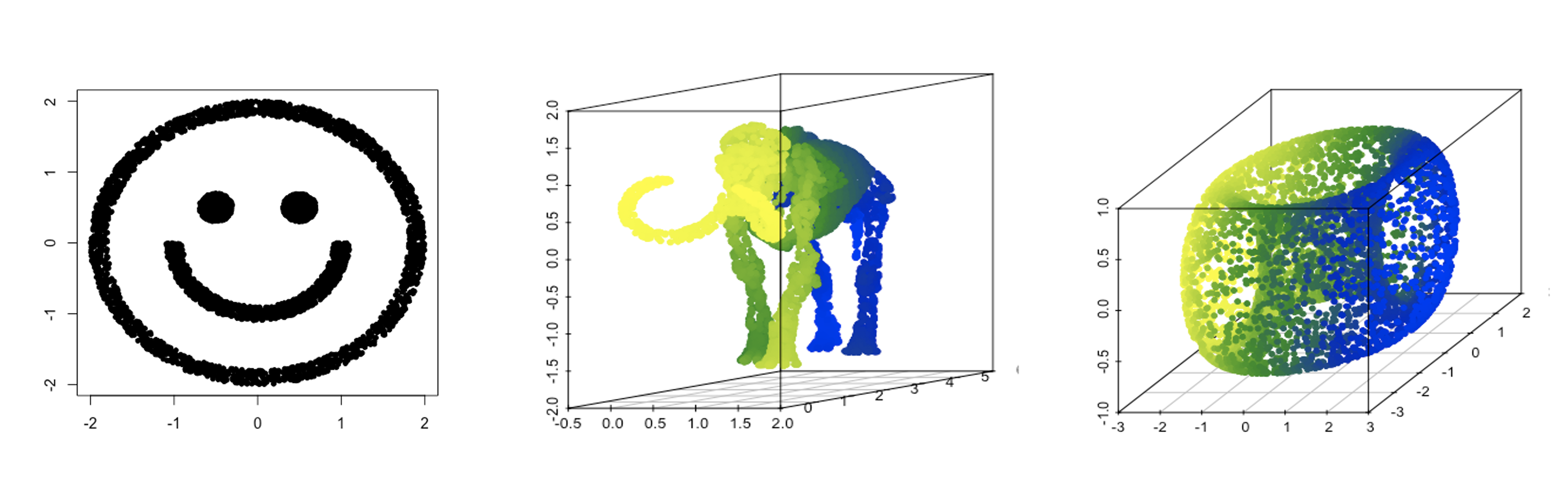}
		\includegraphics[angle=0,width=14cm]{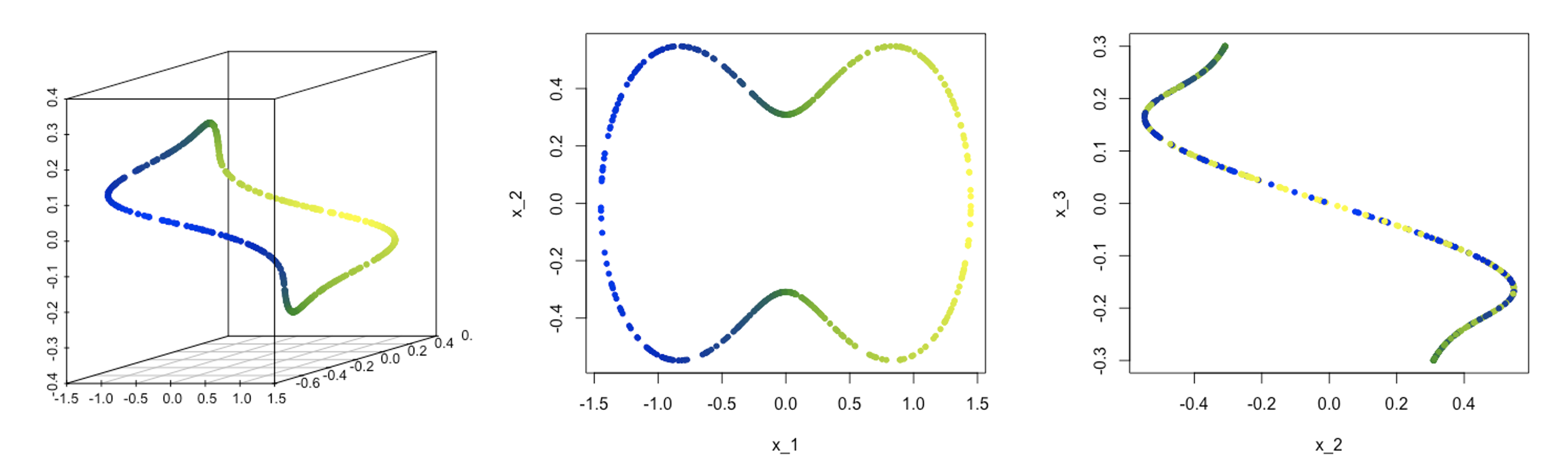}
	\caption{Top: "smiley face," "mammoth," and "torus" (from left to right); Bottom: "Cassini oval" and its two-dimensional projections.} 
	\label{man.fig}
\end{figure} 
}

	\begin{figure}[bt]
	\centering
	\includegraphics[angle=0,width=14cm]{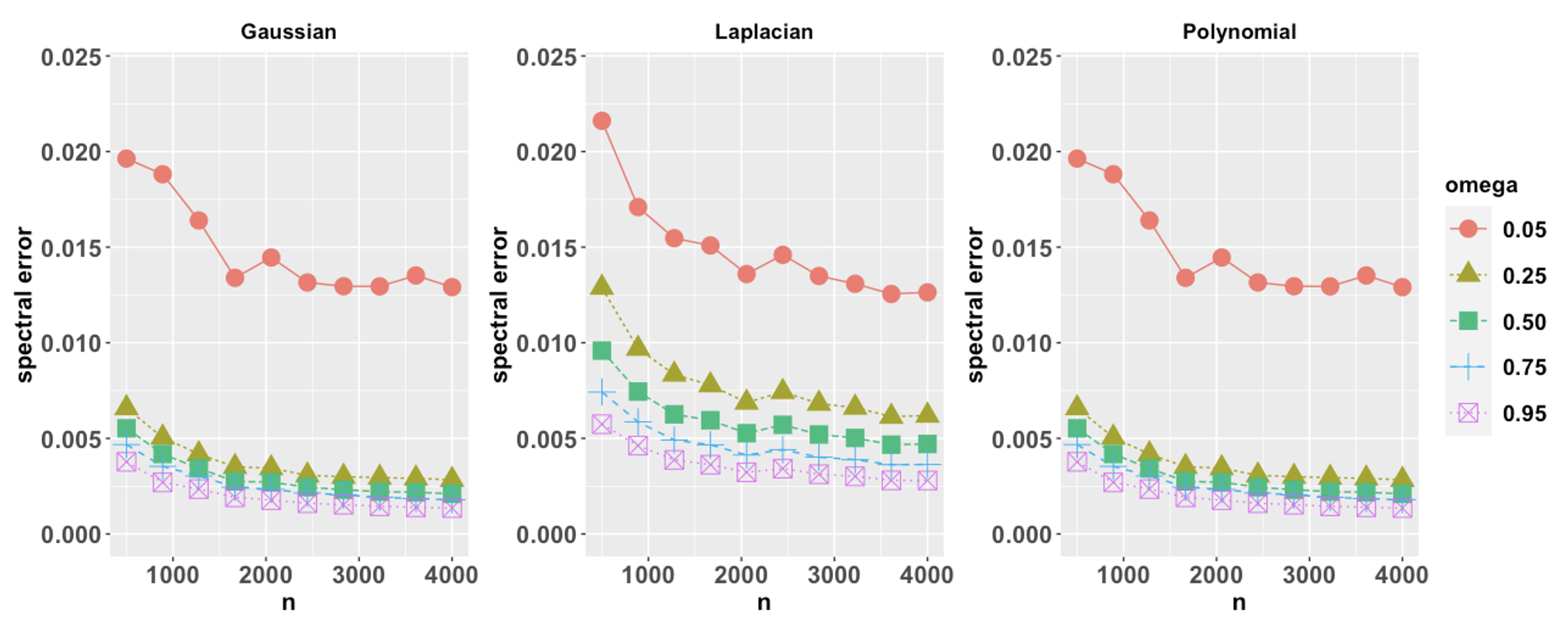}
	\includegraphics[angle=0,width=14cm]{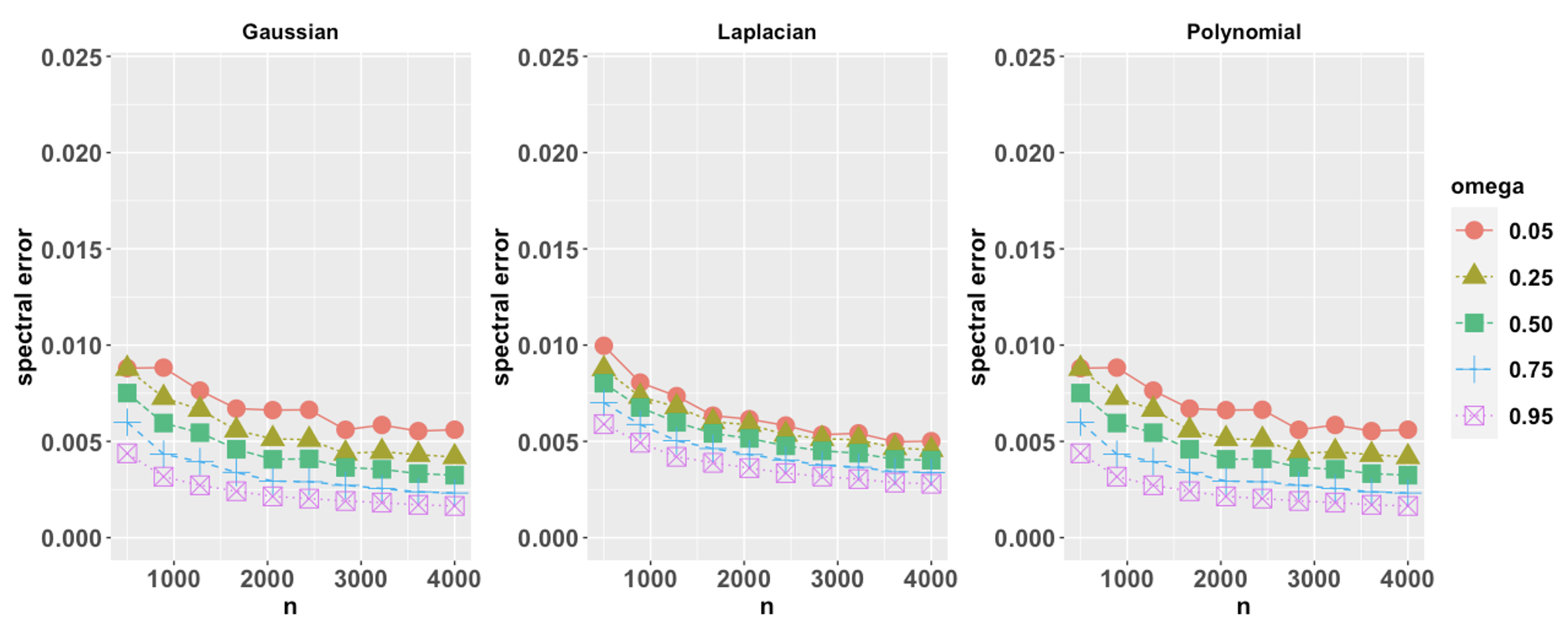}
	\includegraphics[angle=0,width=14cm]{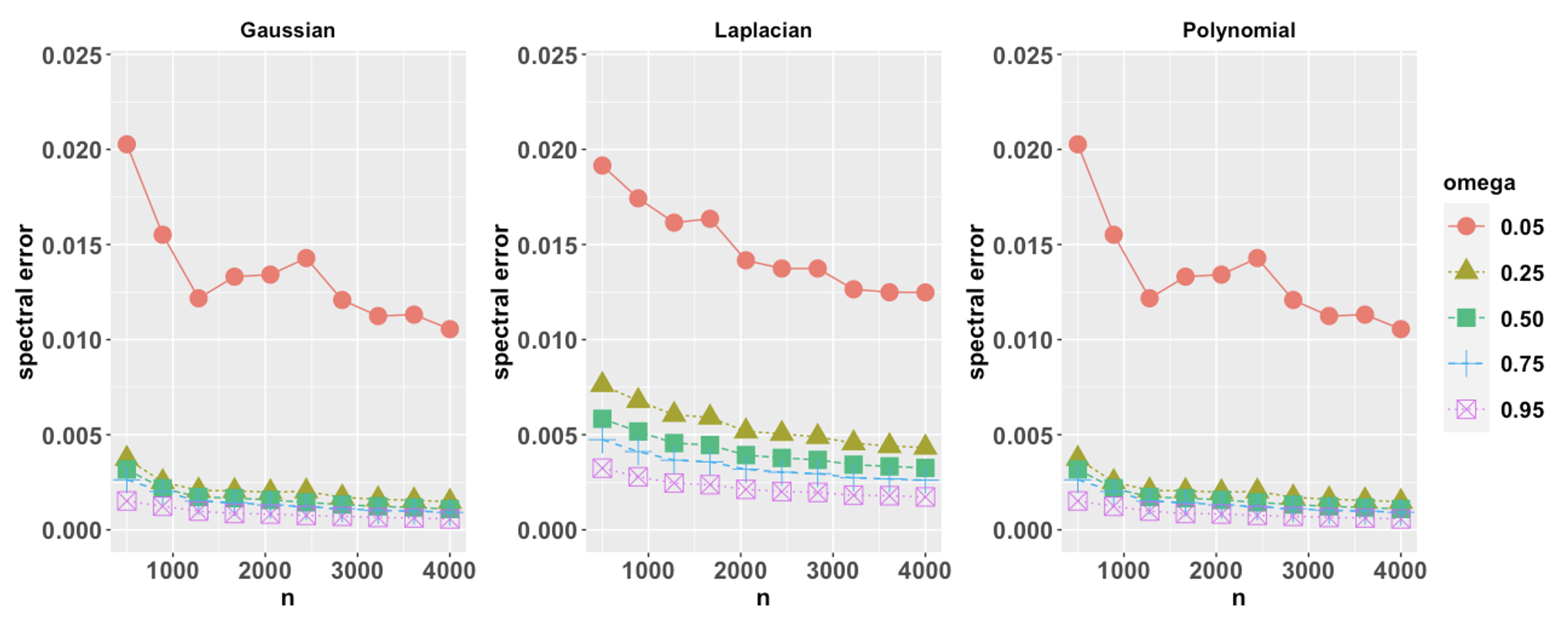}
	\includegraphics[angle=0,width=14cm]{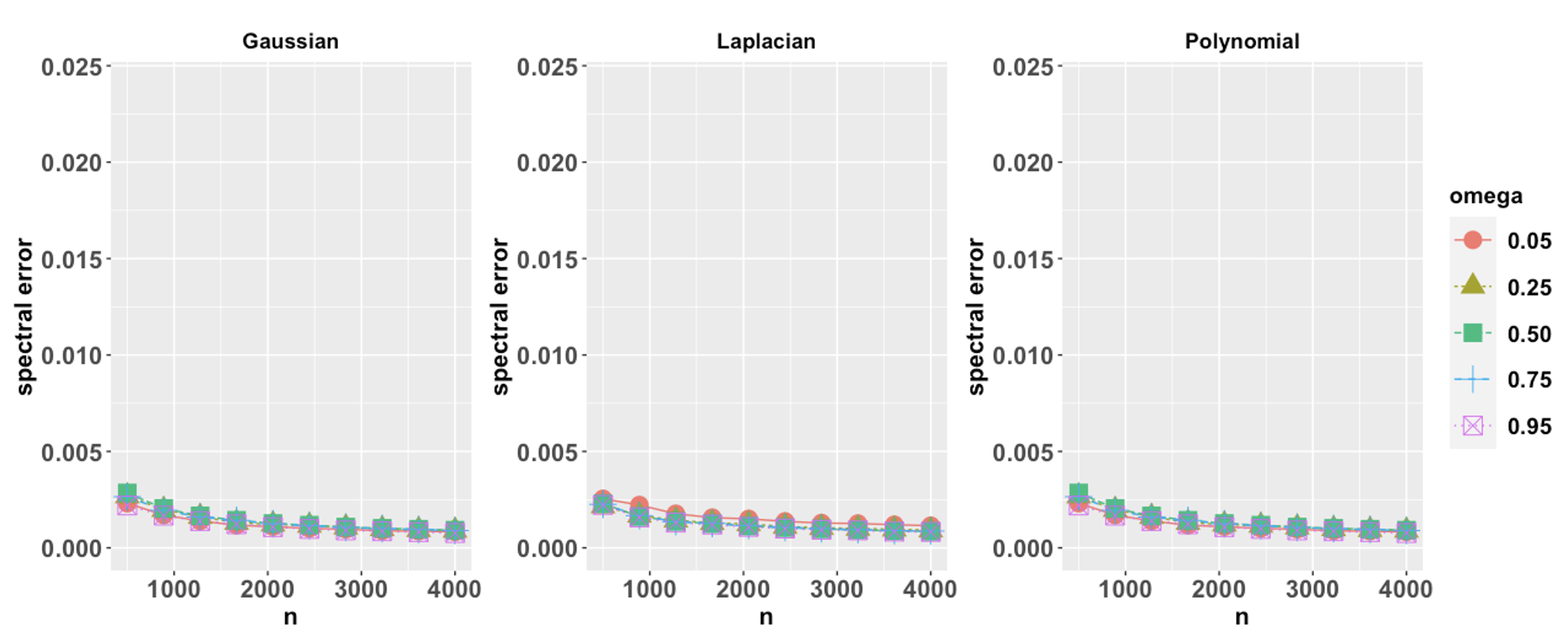}
	\caption{Spectral convergence of kernel matrices under various kernel functions and $\omega\in\{0.05,0.25,0.5,0.75,0.95\}$. Left: Gaussian kernel; Middle: Laplacian kernel; Right: polynomial kernel. From top to bottom: "smiley face,"  "mammoth," "Cassini oval," and "torus."} 
	\label{rate.fig}
\end{figure} 

\begin{figure}[bt]
	\centering
	\includegraphics[angle=0,width=15cm,height=11cm]{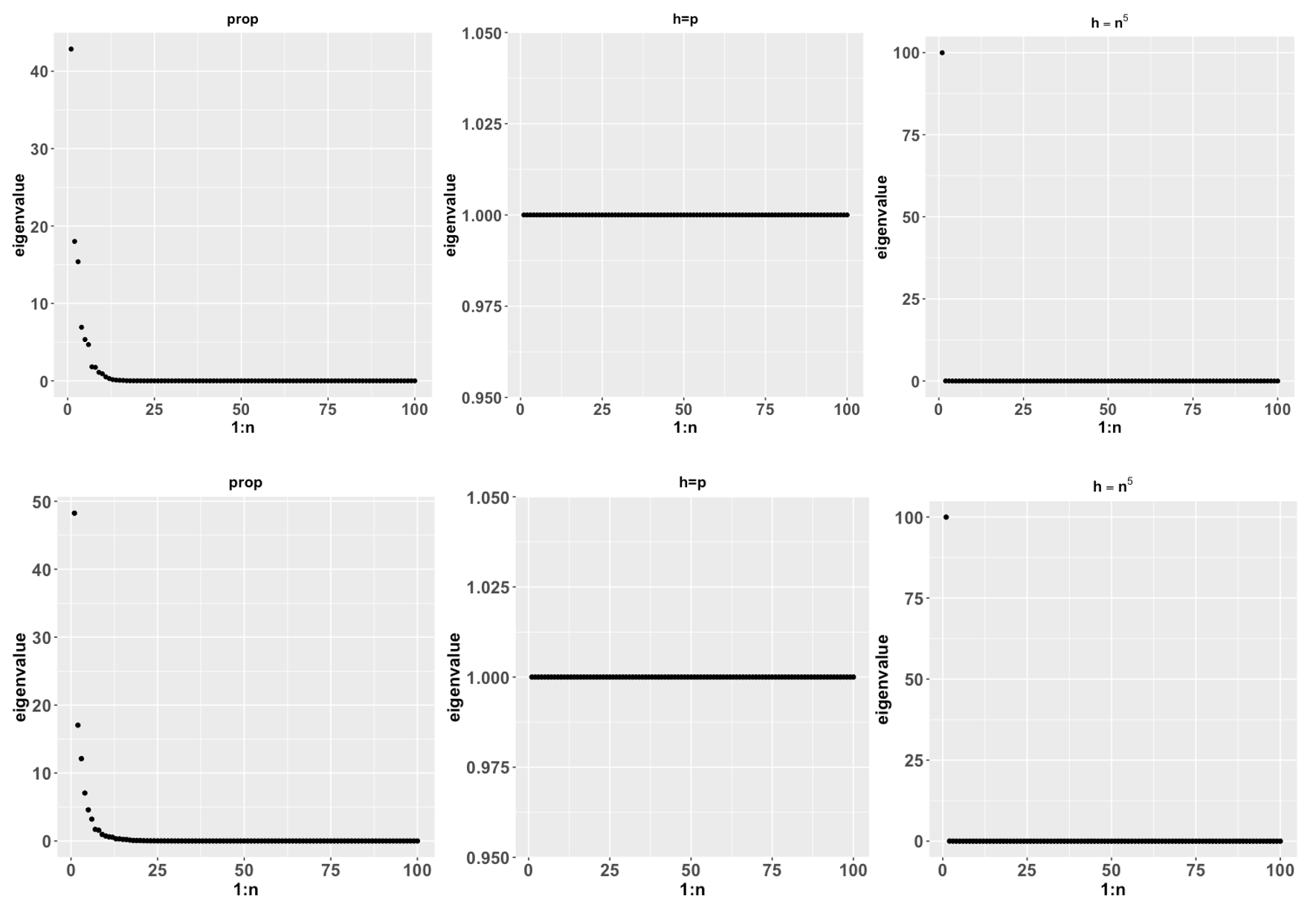}
	\includegraphics[angle=0,width=15cm,height=11cm]{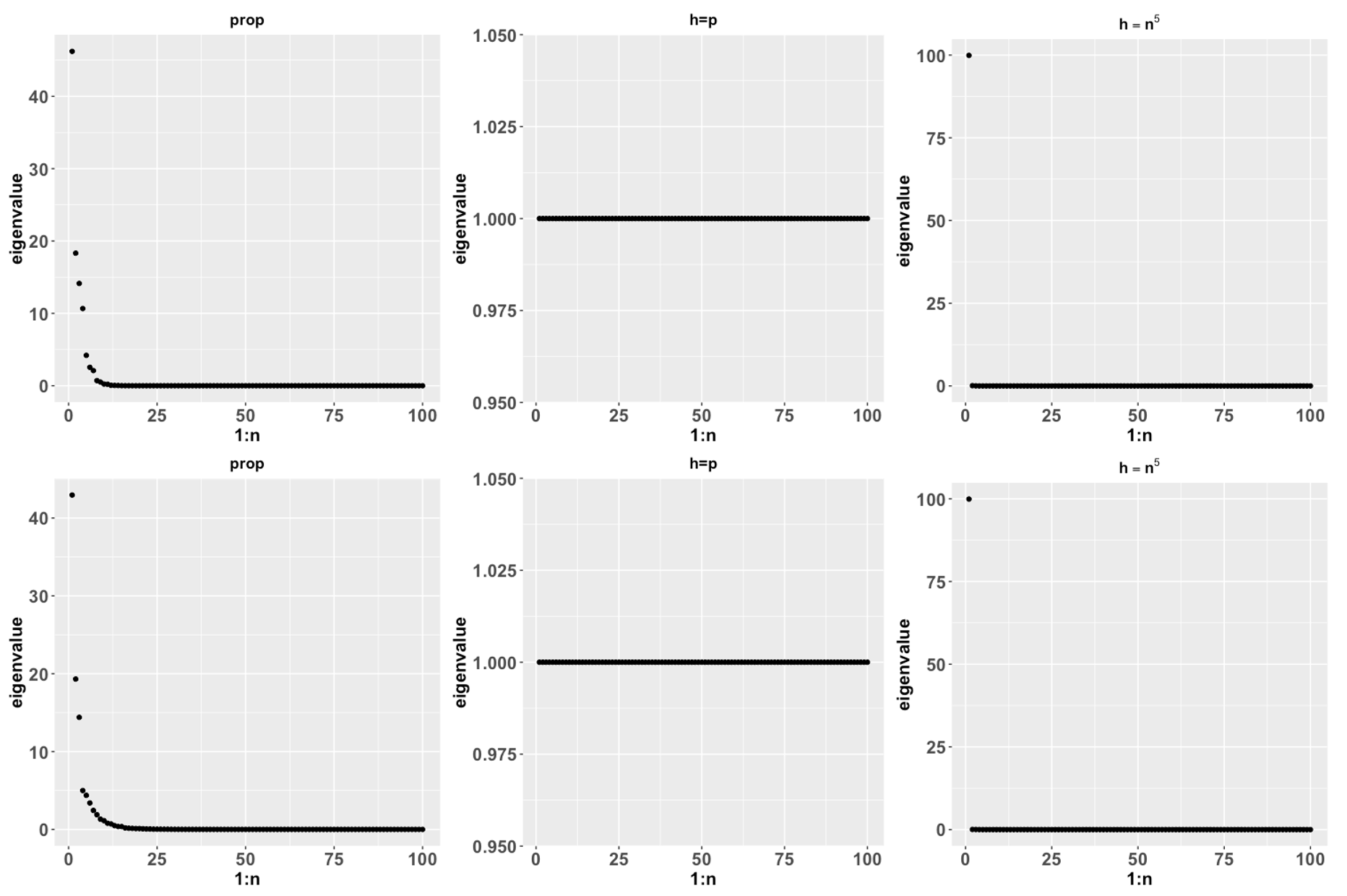}
	\caption{Comparison of eigenvalues of kernel matrices under various bandwidths. Left: proposed bandwidth $\sfh=\sfh_n$ with $\omega=0.5$; Middle: large bandwidth $\sfh=n^5$; Right: small bandwidth $\sfh=p$. From top to bottom: "smiley face,"  "mammoth," "Cassini oval," and "torus."} 
	\label{bw.fig1}
\end{figure} 
%
%

\begin{figure}[bt]
	\centering
	\includegraphics[angle=0,width=14cm,height=6cm]{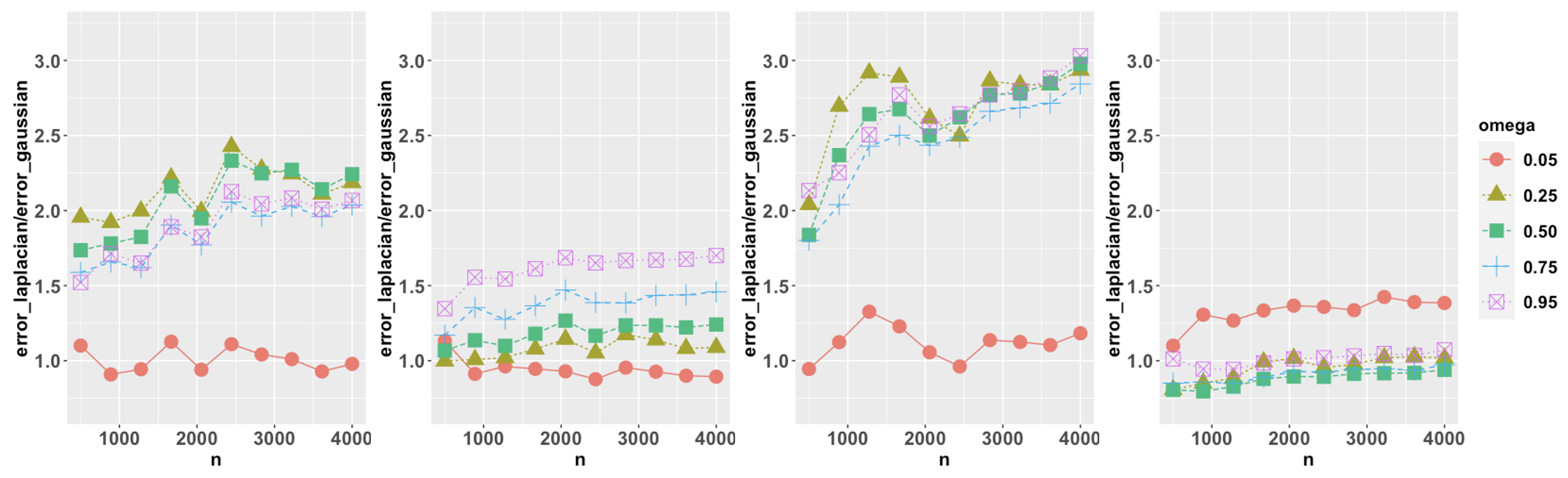}
	\includegraphics[angle=0,width=14cm,height=6cm]{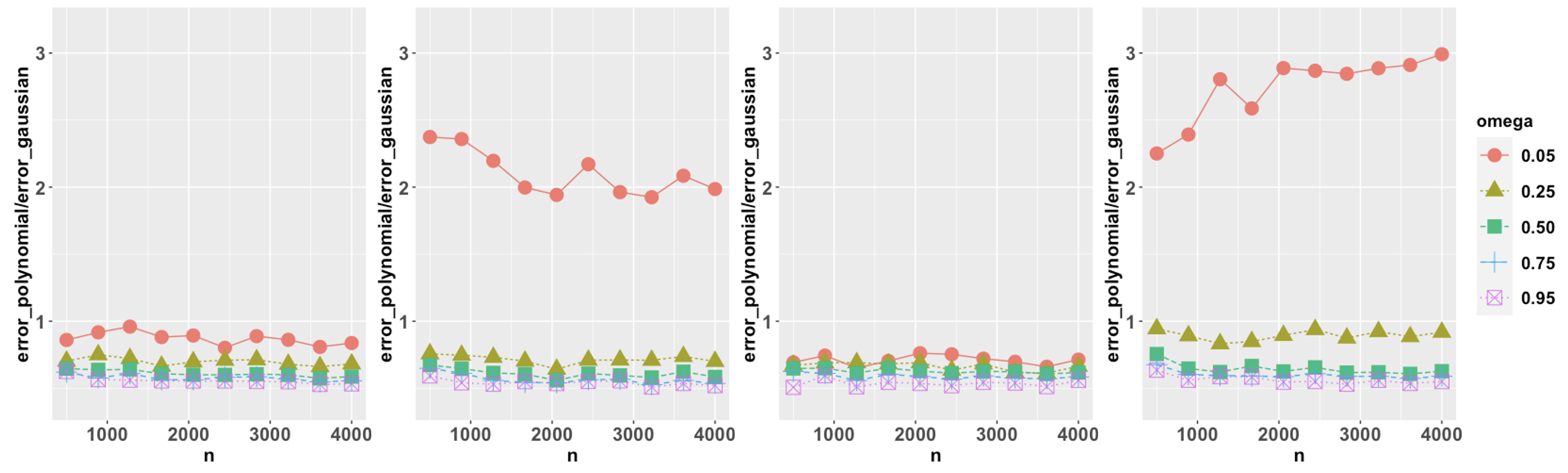}
	\caption{Top: ratios between the convergence rates under  the Laplacian kernel and under the Gaussian kernel. Bottom: ratios between the convergence rates under  the polynomial kernel and under the Gaussian kernel. From left to right:  "smiley face," "mammoth," "Cassini oval" and "torus."} 
	\label{ratio.fig}
\end{figure} 

\subsection{Two Additional Real Data Applications}\label{suppl_additionalrealexamples}

In this subsection, we provide two additional real data analysis examples. 

\subsubsection{Cell Cycle Reconstruction} \label{cycle.sec}

Our first example concerns the reconstruction of cell cycles from single-cell RNA-Seq data. The cell cycle, or cell-division cycle, is the series of events that take place in a cell that cause it to divide into two daughter cells\footnote{https://en.wikipedia.org/wiki/Cell\_cycle}. Determining the cell cycle stages of individual cells analyzed during development is important for understanding its wide-ranging effects on cellular physiology and gene expression profiles. Our  dataset contains 288 mouse embryonic stem cells, whose cell cycle stages were determined using flow cytometry sorting. As a result, one-third (96) of the cells are in the G1 stage, one-third in the S stage, and the rest in the G2M stage. The raw count data were preprocessed  and normalized using the same method as in Section \ref{cell.order.sec}, leading to a dataset consisting of standardized expression levels of $p\in\{2500, 3000,3500, 4000\}$ most variable genes for the 288 cells. We apply our proposed method with  a variety of $\omega\in\{0.25,0.5,0.75\}$. Observing that the leading eigenvector is approximately a constant vector and therefore non-informative, we consider a two-dimensional embedding with $\Omega=\{2,3\}$, expected to capture the underlying circle manifold. As a comparison, we also obtain two-dimensional embeddings  based on PCA, MDS, Laplacian eigenmap, LLE and DM, using functions implemented in the R package \texttt{dimRed} under their default settings.  To recover the cell cycle stages, or the underlying position of the cells on the circle manifold, we project each embedding to the two-dimensional unit circle, and then identify the cell stages with their respective angles on the unit circle. We compare the reconstruction performance of different methods based on their Kendall's tau distance to the true stages, up to a possible circular shift of the reconstructed cycles.

The right panel of Figure \ref{embed.fig22} shows that our proposed method has outstanding performance compared to the other methods. In addition, Figure \ref{cycle} shows the unit-circle projections of the two-dimensional embeddings obtained from each method under $p=3000$. It can be seen that the projected cells reconstructed by the proposed method were more separated and well-ordered according to their true cycle stages.


\begin{figure}[bt]
	\centering
	\includegraphics[angle=0,width=14cm]{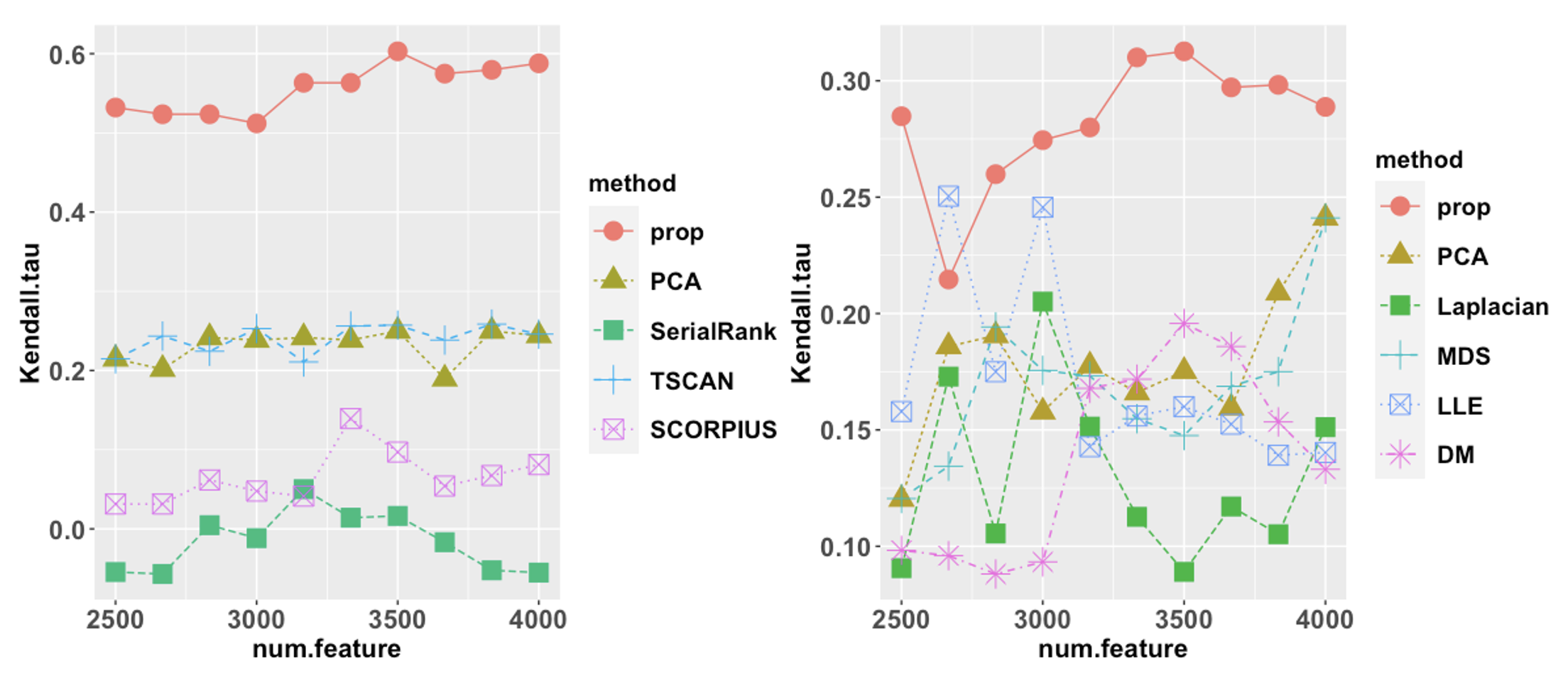}
	\caption{Left: comparison of five methods for cell ordering in Section \ref{cell.order.sec}. Right: comparison of six methods for cell cycle reconstruction in Section \ref{cycle.sec}. In both plots, the proposed method Algorithm \ref{al0} uses $\omega=0.5$.} 
	\label{embed.fig22}
\end{figure} 

\begin{figure}[bt]
	\centering
	\includegraphics[angle=0,width=6cm]{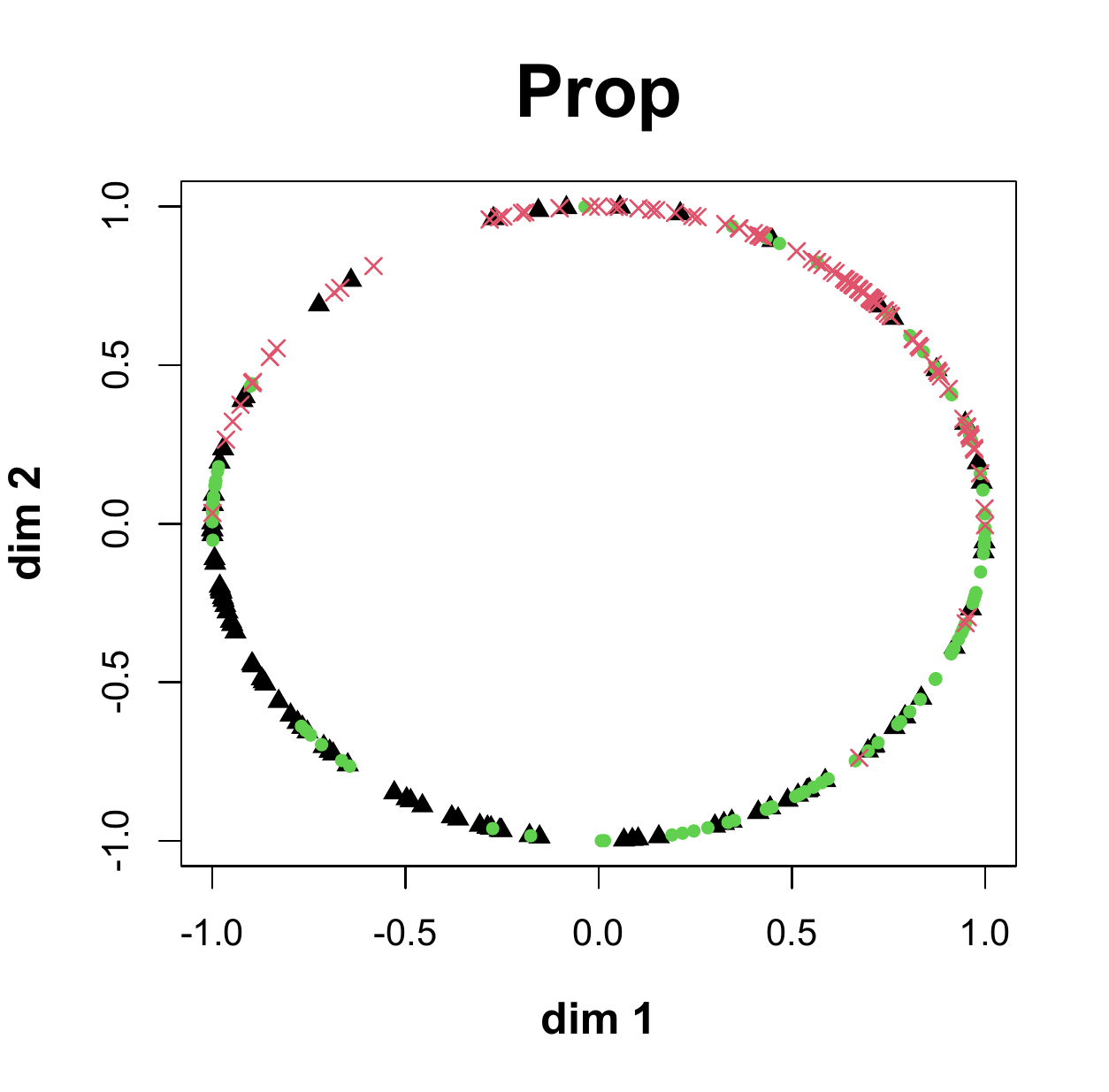}
	\includegraphics[angle=0,width=6cm]{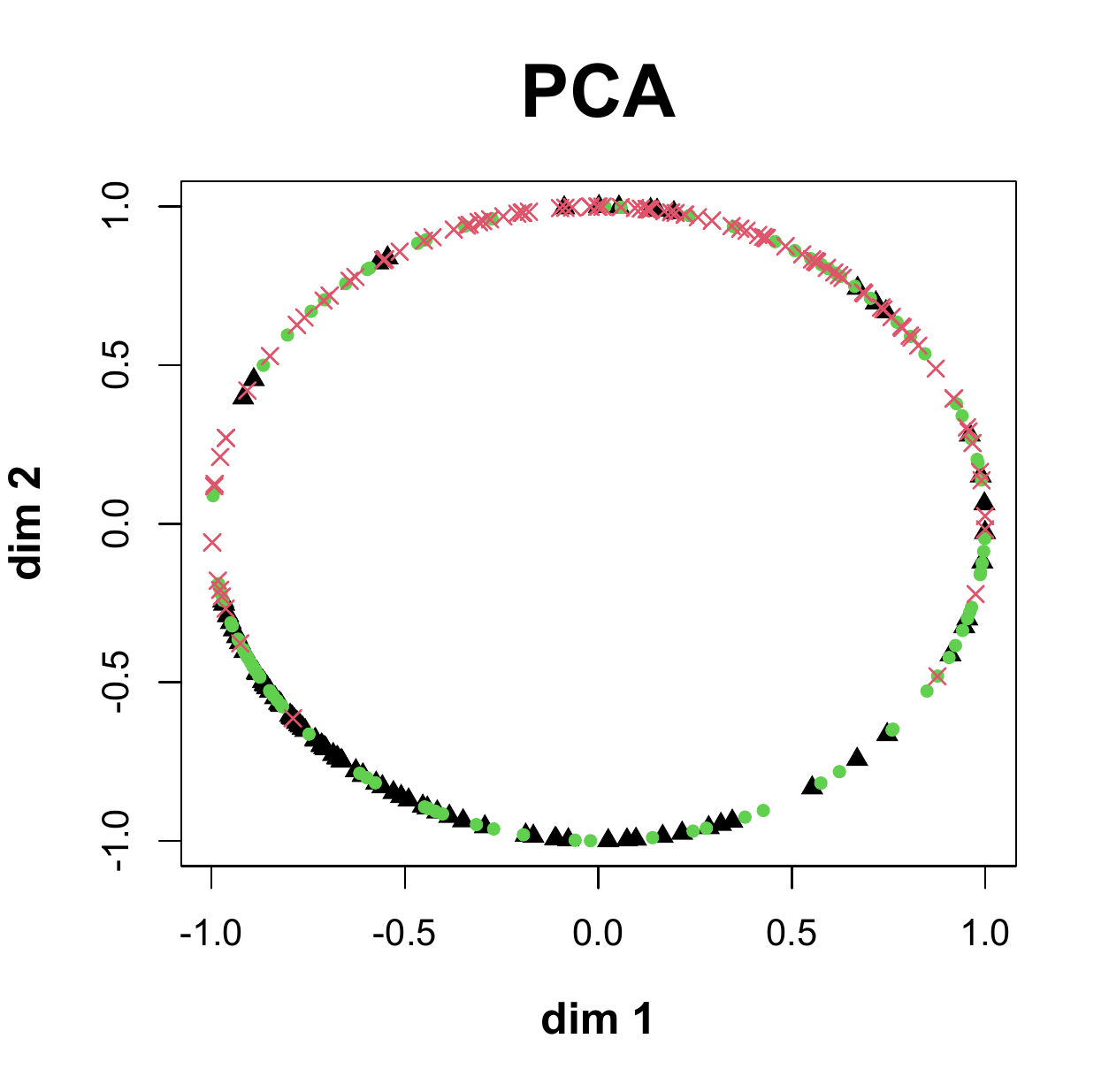}\\	\includegraphics[angle=0,width=6cm]{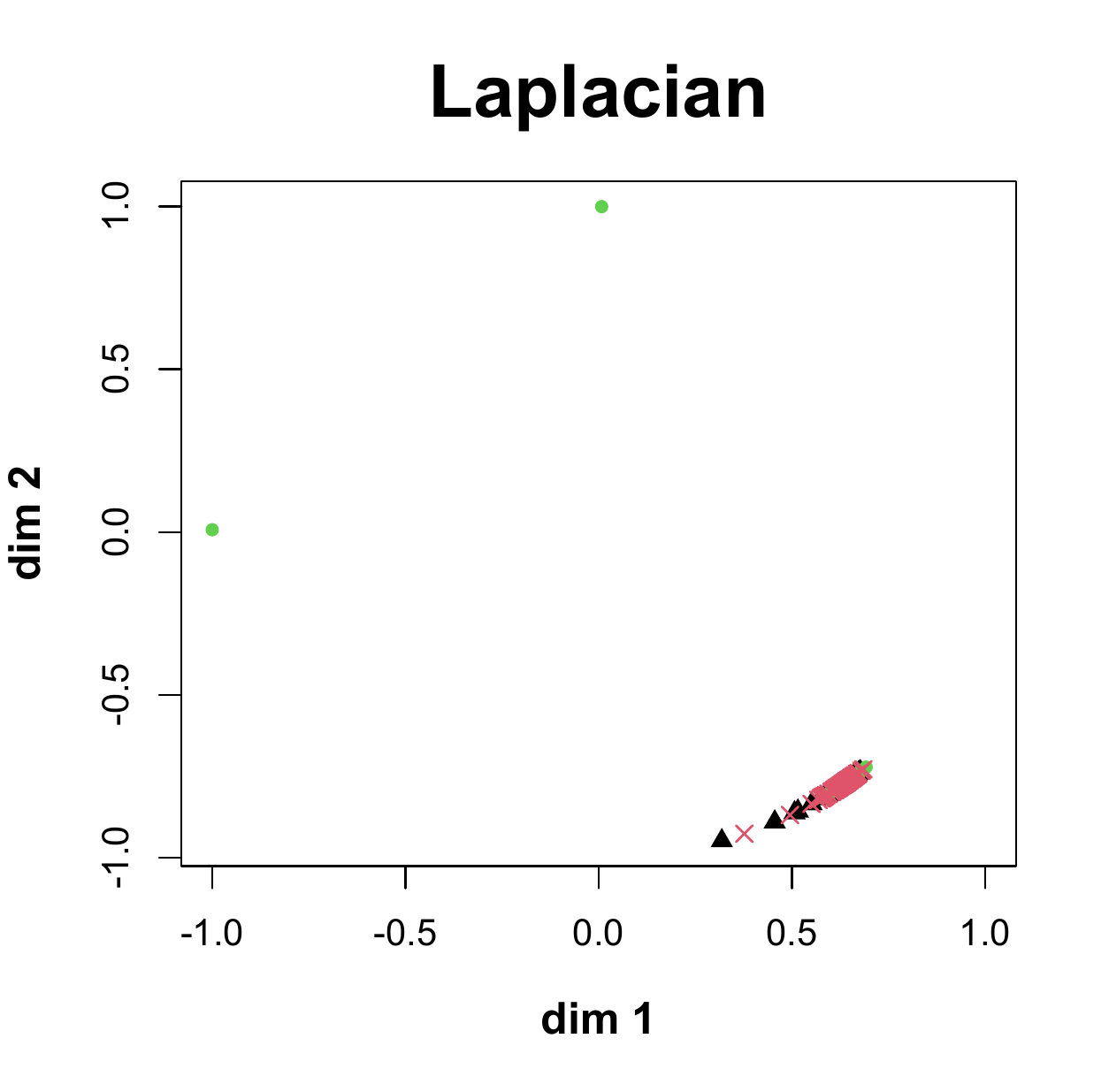}
	\includegraphics[angle=0,width=6cm]{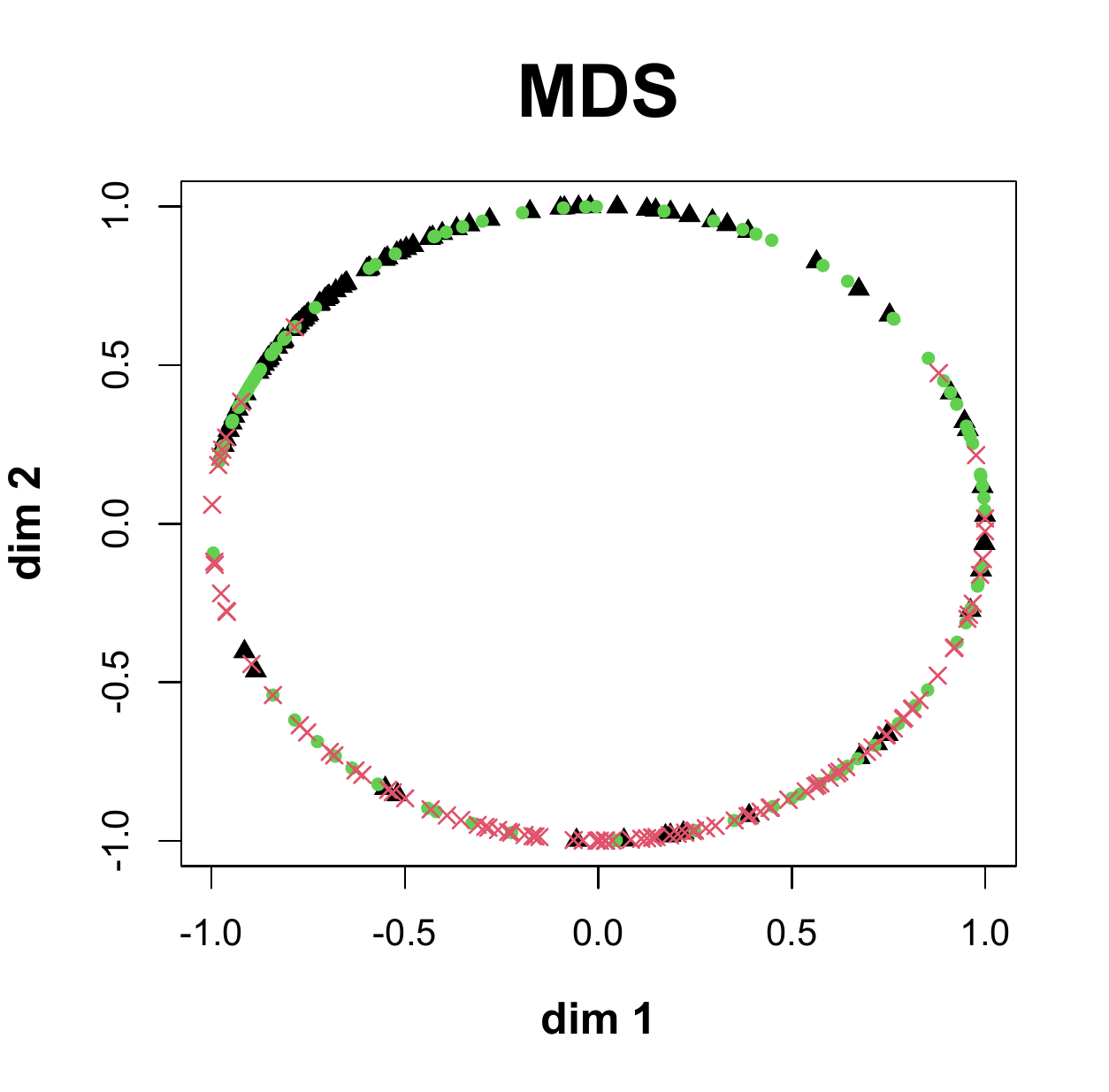}\\
	\includegraphics[angle=0,width=6cm]{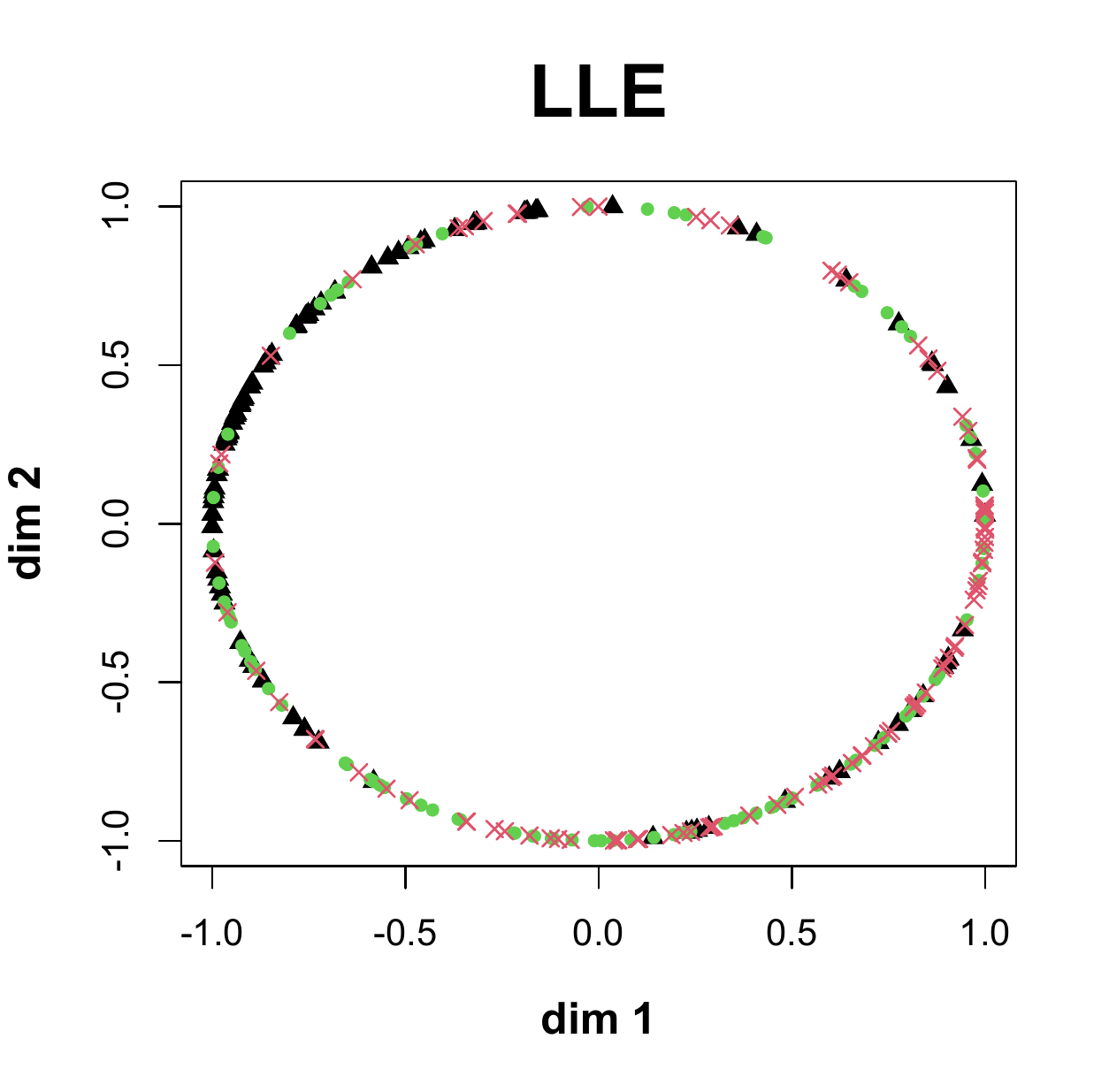}
	\includegraphics[angle=0,width=6cm]{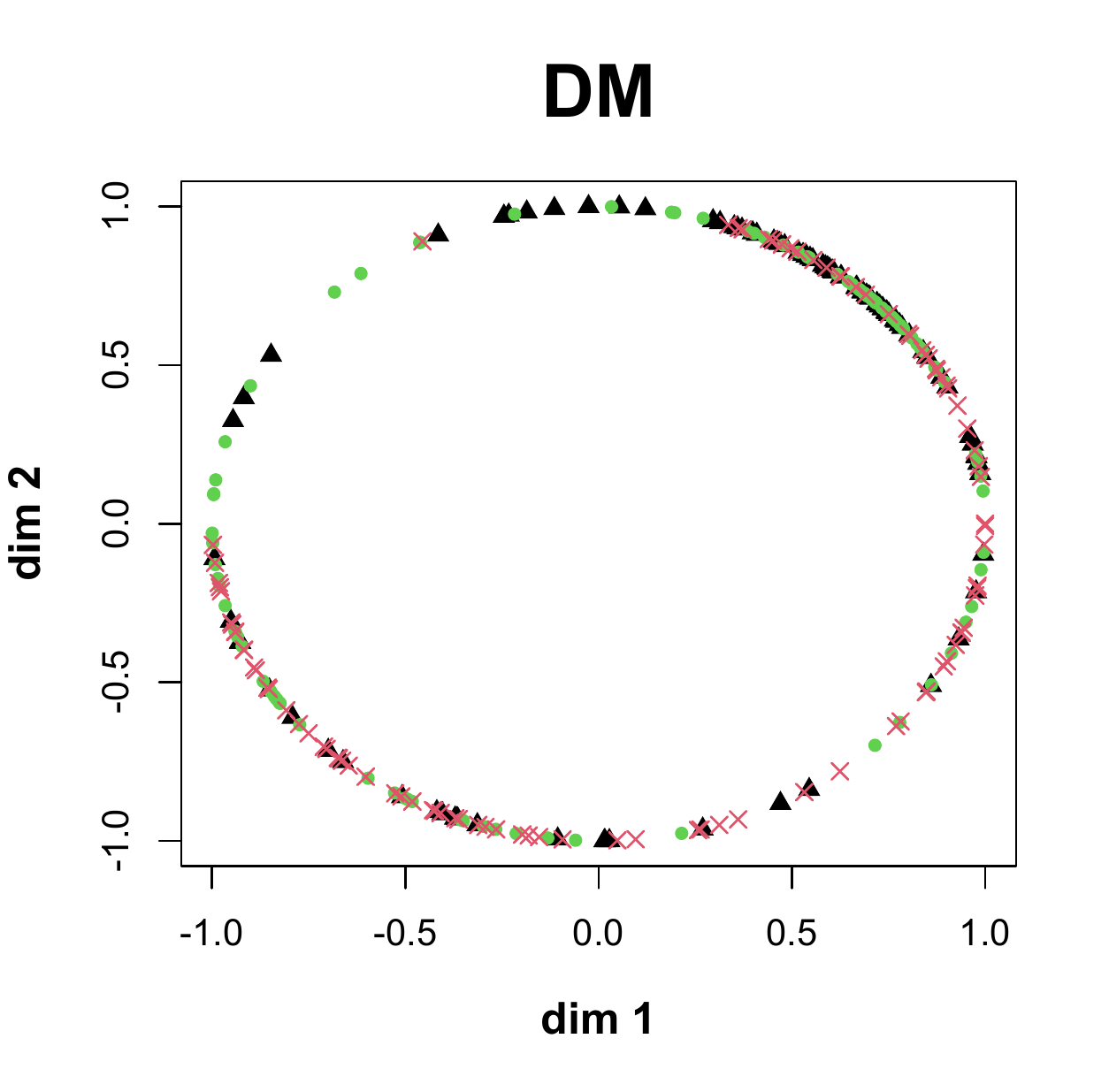}\\
	\vspace{-5mm}
	\caption{Comparison of reconstructed cell cycles based on six embedding methods when $p=3000$. Red crosses: cells in G1 stage. Green dots: cells in S stage. Black triangles: cells in G2M stage. } 
	\label{cycle}
\end{figure}  

\subsubsection{Clustering of Hand-Written Digits} \label{mnist.sec}

Our last example concerns the MNIST\footnote{http://yann.lecun.com/exdb/mnist/} dataset containing images of hand-written digits. Specifically, we focus on $n=3946$ images of hand-written digits "2," "4," "6" and "8,"  among which there are about 1000 images for each digit. As each image contains $28\times28$ pixels, they can be treated as $784$-dimensional vectors.  We apply six different low-dimensional embedding methods, namely DM, kPCA as defined in Section \ref{comp.sec}, Laplacian eigenmap, MDS, PCA and the proposed method, for a variety of embedding dimensions \texttt{n.dim} $\in\{5, 10, 15, 20, 25, 30, 35\}$. In particular, to ensure fairness in comparison, we have not included two-stage algorithms such as tSNE and UMAP, which essentially take one of the above kernel spectral embeddings as an initialization, and then refine the low-dimensional embedding to make the cluster pattern more salient based on some local-metric adjustment \citep{mcinnes2018umap,arora2018analysis,linderman2019clustering,cai2021theoretical}. 

\begin{figure}[bt]
	\centering
	\includegraphics[angle=0,width=8cm]{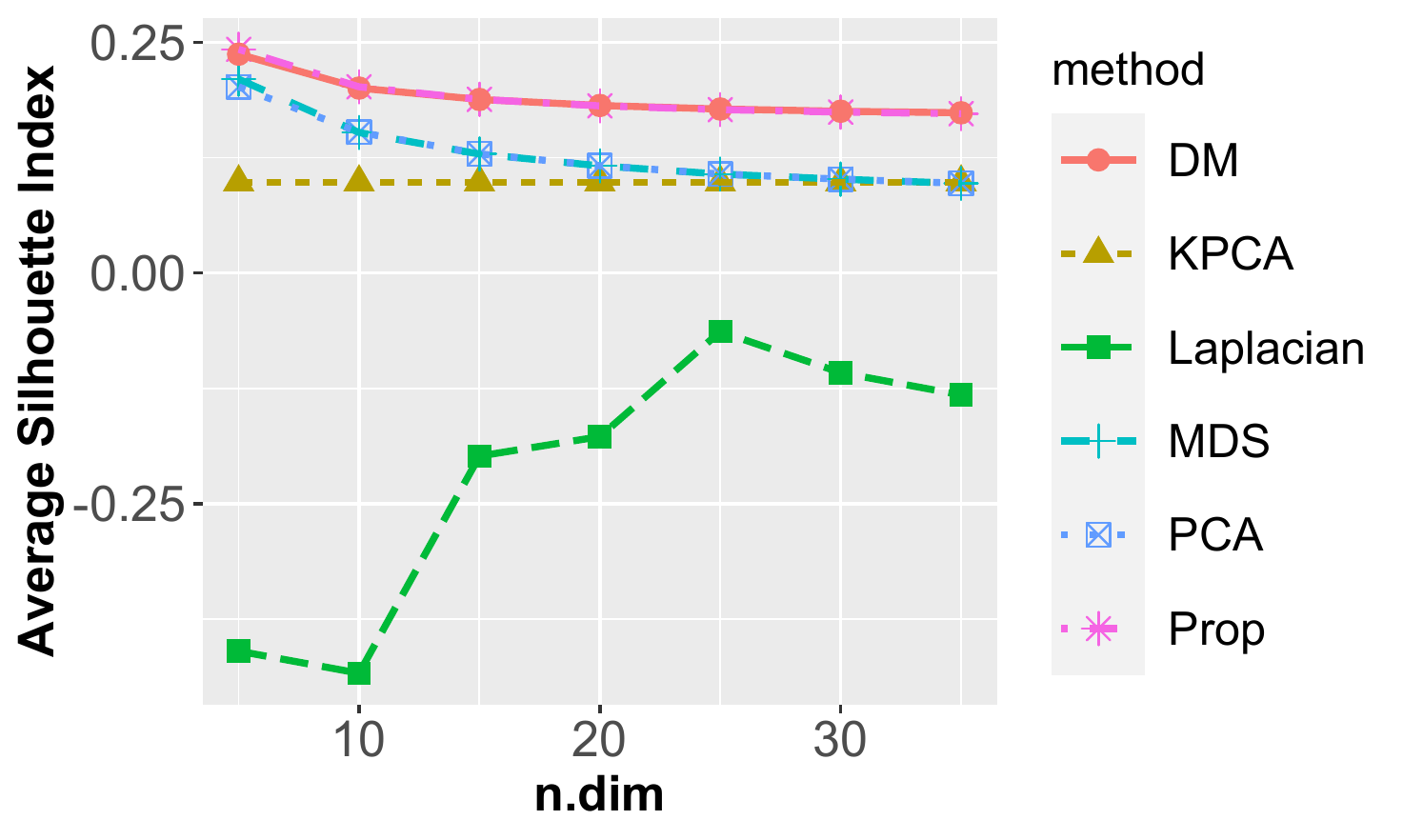}
	\caption{Comparison of six embedding methods for $n=3946$ samples from the MNIST data. The Silhouette index indicates how much of the underlying cluster pattern is preserved in the low-dimensional embeddings.} 
	\label{mnist}
\end{figure}

We evaluate the embedding quality of a method by calculating the average Silhouette index  \citep{rousseeuw1987} of the final  low-dimensional embedding with respect to the underlying true cluster membership. In general, a higher average Silhouette index of a method indicates the underlying clusters are more separate in the final embedding.  Again, for the proposed method we use $\omega=0.5$. Similar results are obtained for $\omega\in\{0.25,0.75\}$ as in Section \ref{sec_additionsimulrealreal}. In Figure \ref{mnist}, we found that our proposed method along with DM has overall the best performance among the six methods. In particular, although the standard kPCA differs from our proposed method only by an additional mean shift (\ref{kpca}), in all cases its embedding quality is clearly worse than the latter, demonstrating the important distinction between kPCA and the proposed method, and the potential advantage of the latter in applications.

\subsection{Supplementary Figures to Section \ref{cell.order.sec}}\label{sec_additionsimulrealreal}
%
%

Figure \ref{embed.fig3} shows the scatter plots of the 2-dimensional embeddings based on PCA and the proposed method (KEF stands for "kernel eigenfunctions") when $p=3000$. The advantage of the proposed embedding over its linear counterpart, such as informativeness and robustness to outliers, is visible and significant.

\begin{figure}[bt]
	\centering
	\includegraphics[angle=0,width=7cm]{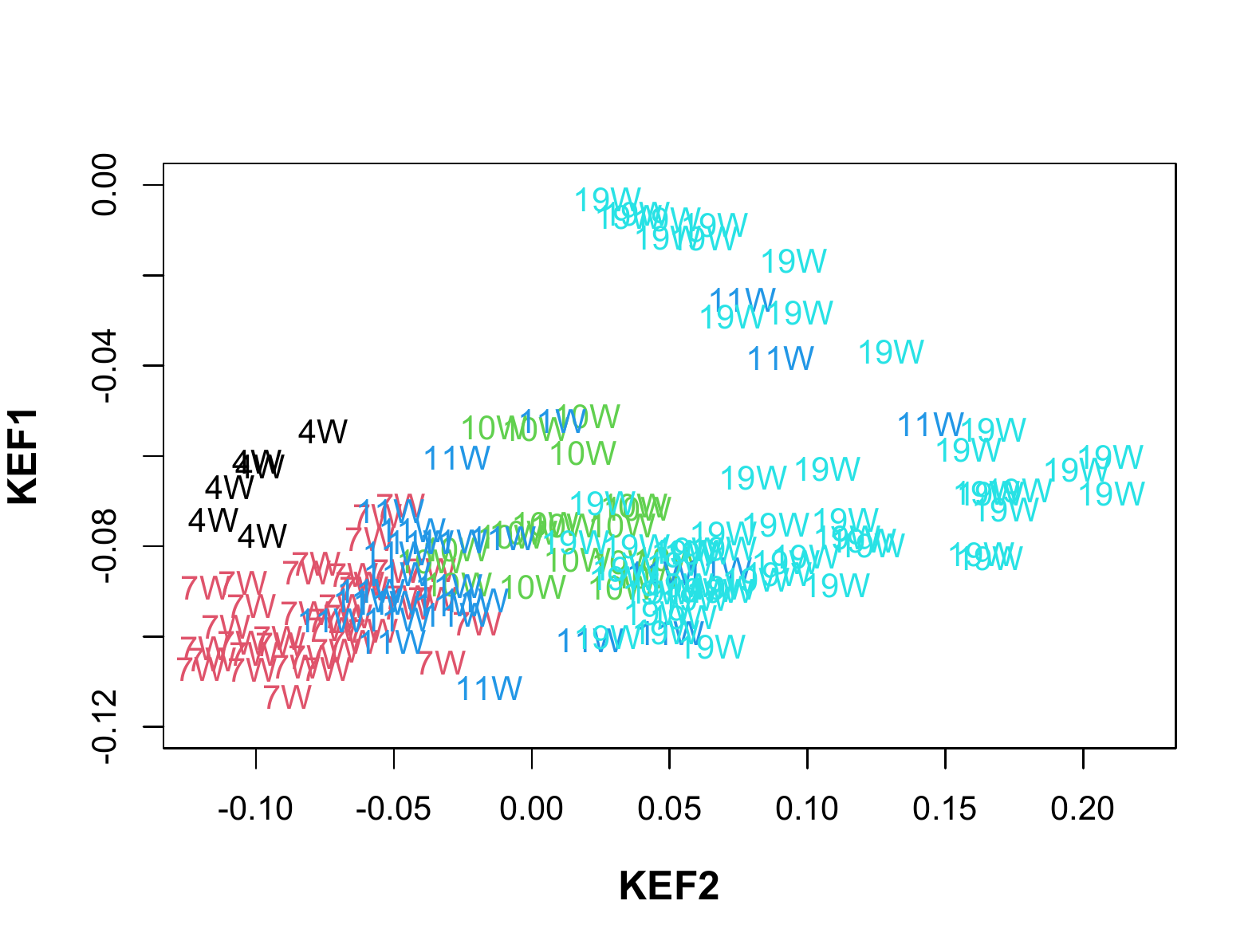}
	\includegraphics[angle=0,width=7cm]{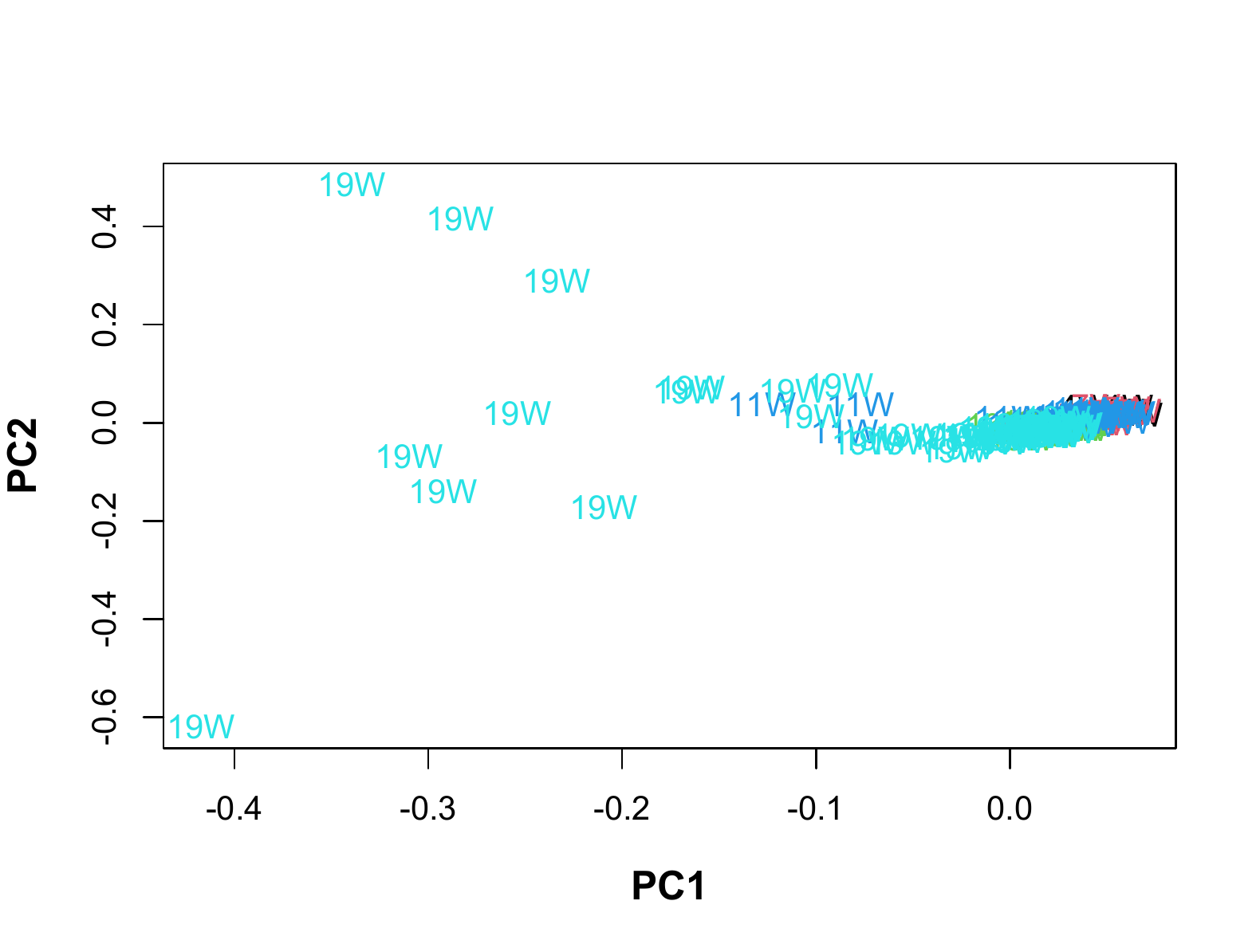}
	\caption{Comparison between principal component embeddings (PCs) and the proposed kernel embeddings (KEFs) when $p=3000$, where the cells were labelled and colored according to their actual time courses.} 
	\label{embed.fig3}
\end{figure}

\bibliographystyle{imsart-number} 
\bibliography{kernel}

\end{document}